\documentclass[nohyperref]{article}
\usepackage{graphicx}

\usepackage[accepted]{icml2023}
 
\usepackage{hyperref}
\definecolor{citeColor}{RGB}{0,20,115}
\hypersetup{colorlinks,linkcolor={citeColor},citecolor={citeColor},urlcolor={citeColor}}  

\usepackage[textsize=tiny]{todonotes}
\usepackage{xspace}

\usepackage{subfigure}
\usepackage{booktabs} 

\usepackage{amsmath}
\usepackage{amssymb}
\usepackage{mathtools}
\usepackage{amsthm}
\usepackage{algorithm}
\usepackage{algorithmic}
\usepackage{amsfonts}
\usepackage[utf8]{inputenc} 
\usepackage[T1]{fontenc}    
\usepackage{hyperref}       
\usepackage{url}            
\usepackage{booktabs}       
\usepackage{amsfonts}       
\usepackage{nicefrac}       
\usepackage{microtype}      
\usepackage{xcolor}         
\usepackage{amsmath}
\usepackage{array}
\usepackage{algorithm}
\usepackage{algorithmic}
\usepackage{amsmath}
\usepackage{amsthm}
\usepackage{amsfonts}
\usepackage{enumerate}
\usepackage{subfigure}
\usepackage{bm}
\usepackage{tabularx}
\usepackage{enumitem}
\usepackage{multirow}
\usepackage{xcolor}
\usepackage{nicefrac}
\usepackage{booktabs}
\usepackage{bbding}
\usepackage{caption}
\usepackage{graphicx}
\usepackage{bbding}
\usepackage{wrapfig}
\usepackage{lipsum}
\usepackage{amssymb}
\usepackage{multibib}
\usepackage{tcolorbox}

\theoremstyle{plain}
\newtheorem{theorem}{Theorem}[section]
\newtheorem{proposition}[theorem]{Proposition}
\newtheorem{lemma}[theorem]{Lemma}
\newtheorem{corollary}[theorem]{Corollary}
\theoremstyle{definition}
\newtheorem{definition}[theorem]{Definition}

\theoremstyle{remark}
\newtheorem{remark}[theorem]{Remark}

\usepackage[textsize=tiny]{todonotes}

\newcommand{\up}[1]{\textcolor[rgb]{0.8 0 0}{#1}}
\newcommand{\down}[1]{\textcolor[rgb]{0 0.6 0}{#1}}

\usepackage{etoc}
\etocdepthtag.toc{mtchapter}
\etocsettagdepth{mtchapter}{subsection}
\etocsettagdepth{mtappendix}{none}

\icmltitlerunning{On Strengthening and Defending Graph Reconstruction Attack with Markov Chain Approximation}

\begin{document}
	
	
	\twocolumn[
	\icmltitle{On Strengthening and Defending Graph Reconstruction Attack \\
		with Markov Chain Approximation}

	
	
	\icmlsetsymbol{equal}{*}
	
	\begin{icmlauthorlist}
		\icmlauthor{Zhanke Zhou}{hkbu}
		\icmlauthor{Chenyu Zhou}{hkbu}
		\icmlauthor{Xuan Li}{hkbu}
		\icmlauthor{Jiangchao Yao}{sjtu,lab}
		\icmlauthor{Quanming Yao}{thu}
		\icmlauthor{Bo Han}{hkbu}
	\end{icmlauthorlist}
	
	\icmlaffiliation{hkbu}{Department of Computer Science, Hong Kong Baptist University}
	\icmlaffiliation{lab}{Shanghai AI Laboratory}
	\icmlaffiliation{sjtu}{Cooperative Medianet Innovation Center, Shanghai Jiao Tong University}
	\icmlaffiliation{thu}{Department of Electronic Engineering, Tsinghua Unversity}
	
	\icmlcorrespondingauthor{Bo Han}{bhanml@comp.hkbu.edu.hk}
	\icmlcorrespondingauthor{Jiangchao Yao}{Sunarker@sjtu.edu.cn}
	
	\icmlkeywords{Machine Learning, ICML}
	
	\vskip 0.3in
	]
	
	
	
	\printAffiliationsAndNotice{}  
	
	\begin{abstract}
		
		Although powerful graph neural networks (GNNs) have boosted numerous real-world applications, the potential privacy risk is still under-explored.
		%
            To close this gap,
		we perform the first comprehensive study of
		\textit{graph reconstruction attack} that aims to reconstruct the \textit{adjacency} of nodes.
		We show that a range of factors in GNNs can lead to the surprising leakage of private links.
		Especially by taking GNNs as a Markov chain and attacking GNNs via a flexible chain approximation,
		we systematically explore the underneath principles of graph reconstruction attack,
		and propose two information theory-guided mechanisms:
		(1) the chain-based attack method
		with adaptive designs for extracting more private information;
		(2) the chain-based defense method that sharply reduces the attack fidelity with moderate accuracy loss. Such two objectives disclose a critical belief that to recover better in attack, you must extract more multi-aspect knowledge from the trained GNN; while to learn safer for defense, you must forget more link-sensitive information in training GNNs. 
		Empirically, we achieve state-of-the-art results on six datasets and three common GNNs.
		The code is publicly available at:
		\url{https://github.com/tmlr-group/MC-GRA}.
	\end{abstract}

	\section{Introduction}
	\label{sec: introduction}
	
	Deep learning has promoted tremendously broad research 
	from Euclidean data like images
	to non-euclidean data like graphs.
	Specifically, graph neural networks (GNNs)~\citep{kipf2016semi, gilmer2017neural, zhang2018link} 
        proposed in the recent years
	have drawn much attention and boosted a wide range of real-world applications, \textit{e.g.},
	social network~\citep{fan2019graph},
	recommender systems~\citep{wu2020graphRecsys}
	and drug discovery~\citep{ioannidis2020few}.

	Nevertheless, the privacy concerns behind these applications raise with the development of the \textit{Model Inversion Attack} (MIA) technique, which only requires a trained model and non-sensitive features to recover the sensitive information.
	In particular, recent progress on MIA~\citep{fredrikson2015model, zhang2020secret, struppek2022ppa} has shown the feasible recovery of private images in high fidelity and diversity.
	%
	As for the scenarios of GNNs,
	the similar inversion of the adjacency of the training graph is also a severe privacy threat, since
	links can reflect the sensitive relationship information
	or intellectual properties of the model's owner.
	%
	We term this kind of MIA on graphs as Graph Reconstruction Attack (GRA) for simplicity and 
	illustrate exemplars in Fig.~\ref{fig: motivation}.
	To date, only limited research has been conducted on GRA~\citep{he2021stealing, zhang2021graphmi}
	that is designed for ad-hoc scenarios.
	The general principles
	for strengthening and defending GRA are still unknown,
	which presents hidden dangers in extensive real-world applications.
	%
	Thus, it is urgent to understand the vulnerability of GNNs
	under such attacks and explore the proper defense methods to protect GNNs 
	and avoid privacy risks in advance.
	
	\begin{figure}[t!]
		\centering
		\includegraphics[width=8.0cm]{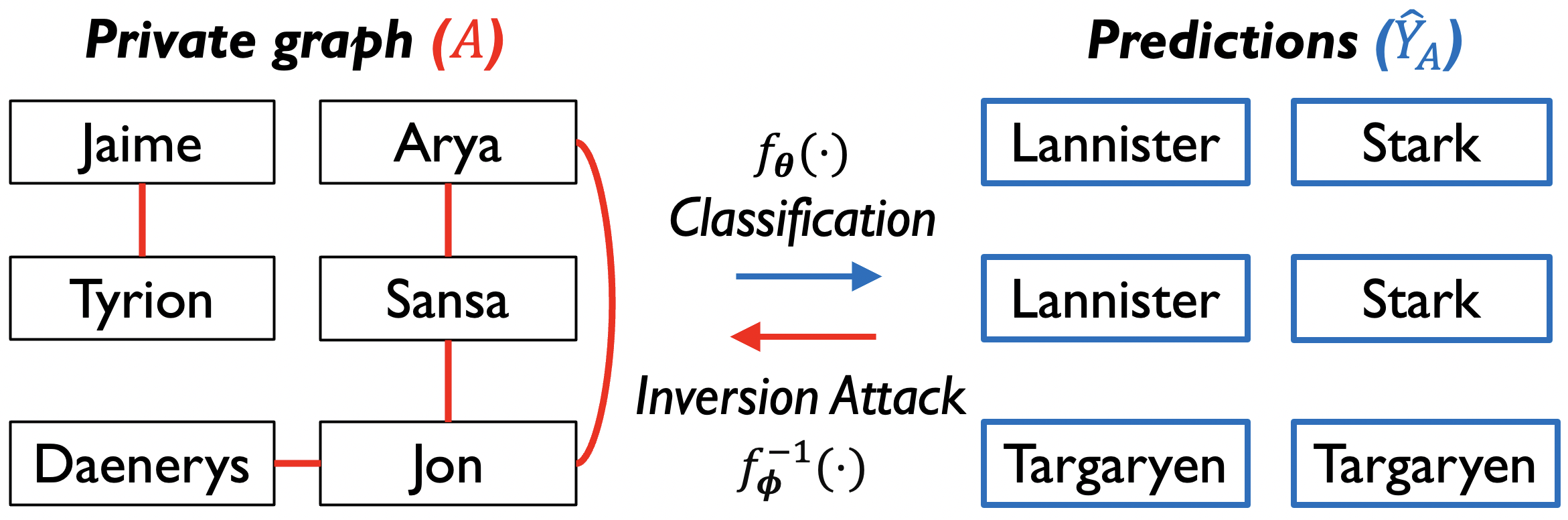}
		\vspace{-2px}
		\caption{
			An illustration of Graph Reconstruction Attack.
			The \textit{forward} inference of a trained model 
			is to predict the node category $\hat{\bm{Y}}_A$,
			\textit{i.e.}, 
                the family name of each character;
			while the \textit{backward} inversion attack is to recover the original adjacency $A$,
			\textit{i.e.}, 
                the kinship among characters (red edges).
		}
		\label{fig: motivation}
	\vspace{-12px}
	\end{figure}
	
	
	
	In this work,
	we systematically investigate this crucial yet under-explored problem
	from both sides of attack and defense.
	Roughly,
	the GNNs' inference procedure can be viewed as a Markov Chain
	$f_{\bm{\theta}}: (A, X) \! \to \! \bm{H}_{A} \! \to \! \bm{\hat{Y}}_{A} \! \leftrightarrow \! Y$,
	where the adjacent matrix $A$ and node features $X$ are taken as the inputs to generate
	node embeddings $\bm{H}_{A}$, and
	a linear layer with activation transforms $\bm{H}_{A}$ 
	into the classification outputs $\bm{\hat{Y}}_{A}$ 
	to predict node labels $Y$.
	%
	More importantly,
	we reveal that every variable 
        in $\{X,  Y, \bm{H}_{A}, \bm{\hat{Y}}_{A} \}$
	can recover adjacency to a certain extent through a simple transformation.
	However, different from the single variable in MIA for images, 
	it is mysterious to understand the intriguing mechanism behind the multiple interplaying factors in GNNs, thus challenging to apply for strengthening and defending GRA.

	\begin{figure}[t!]
		\centering
		\vspace{+4px}
		\includegraphics[width=8.0cm]{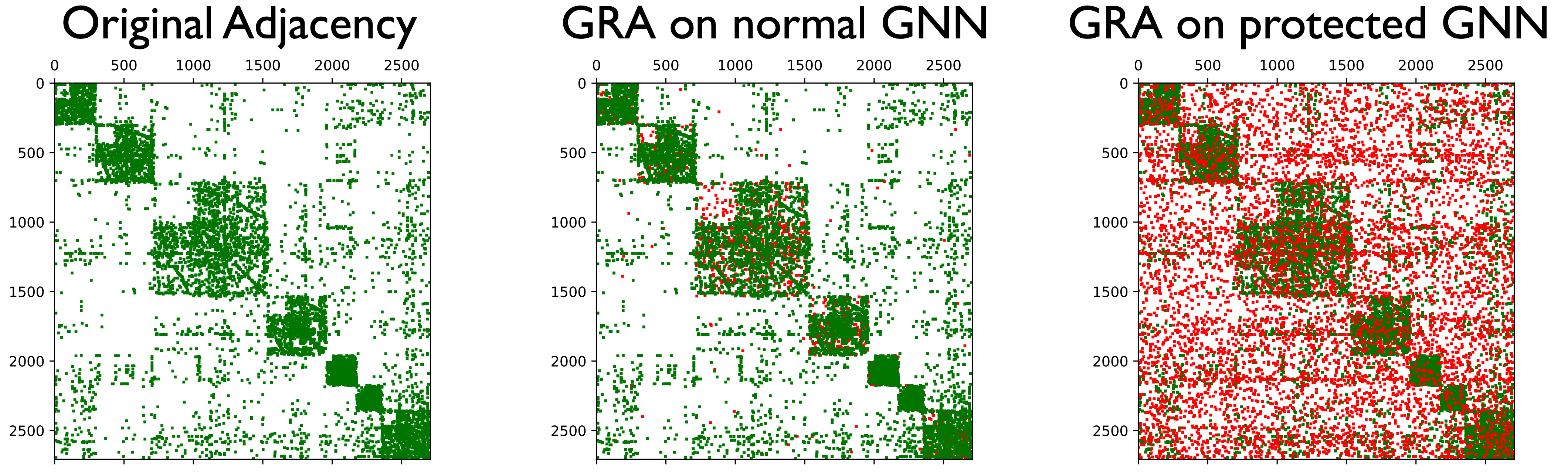}
		\vspace{-6px}
		\caption{
			Recovered adjacency on Cora dataset.
			Green dots are correctly predicted edges while red dots are wrong ones.
			GRA on normal GNN leads to privacy leakage,
			while GRA on protected GNN cannot recover the private adjacency.
		}
		\label{fig: adjacency-demo}
		\vspace{-10px}
	\end{figure}
	
	
	To close the gap, 
	we formulate the GRA problem from a novel perspective, \textit{i.e.},
	approximating the original Markov chain by the attack chain (Fig.~\ref{fig: GRA-two-chains}).
	Note that such a modeling manner brings three-fold advantages:
	(1) adaptively supports the white-box attack that
	utilizes any set of prior knowledge;
	(2) help derive the chain-based attack and defense objectives in optimization;
	(3) enables analysis from the information-theoretical view.
	%
	%
	On the basis of the chain modeling,
	we investigate the underneath principles of the GRA problem, which are two folds.
	To strengthen the attack,
	we derive the Markov Chain-based Graph Reconstruction Attack (MC-GRA)
	that simulates the hidden transformation procedure of the target GNN
	by approximating all the known informative variables in a combinatorial manner.
	As for defense,
	we propose the Markov Chain-based Graph Privacy Bottleneck (MC-GPB), which regularizes the mutual dependency among graph representations, adjacency, and labels to alleviate privacy leakage,
	as shown in Fig.~\ref{fig: adjacency-demo}.
	
	In short, our main contributions are summarized as follows.
	\\\textbf{(1)}
	To our best knowledge,
	we are the first to conduct a systematic study of GRA
	and reveal several essential and useful phenomenons (Sec.~\ref{sec: overview}).
	\textbf{(2)} On the basis of the chain modeling,
	we propose a new method for the attack that boosts the attack fidelity with parameterization techniques
	and injected stochasticity 
	(Sec.~\ref{sec: GRA attack}), 
        and propose an information theory-guided principle for the defense that
	significantly degenerates all the attacks with only a slight accuracy loss
	(Sec.~\ref{sec: GRA defense}).
	\textbf{(3)}
	We provide a rigorous analysis from information-theoretical perspectives to disclose several valuable insights on what and how to strengthen and defend GRA.
	%
	\textbf{(4)}
	Both the two proposed methods achieve state-of-the-art results
	on six datasets and three GNNs (Sec.~\ref{sec: experiment}).

	\begin{figure}[t!]
		\centering
		\includegraphics[width=8.0cm]{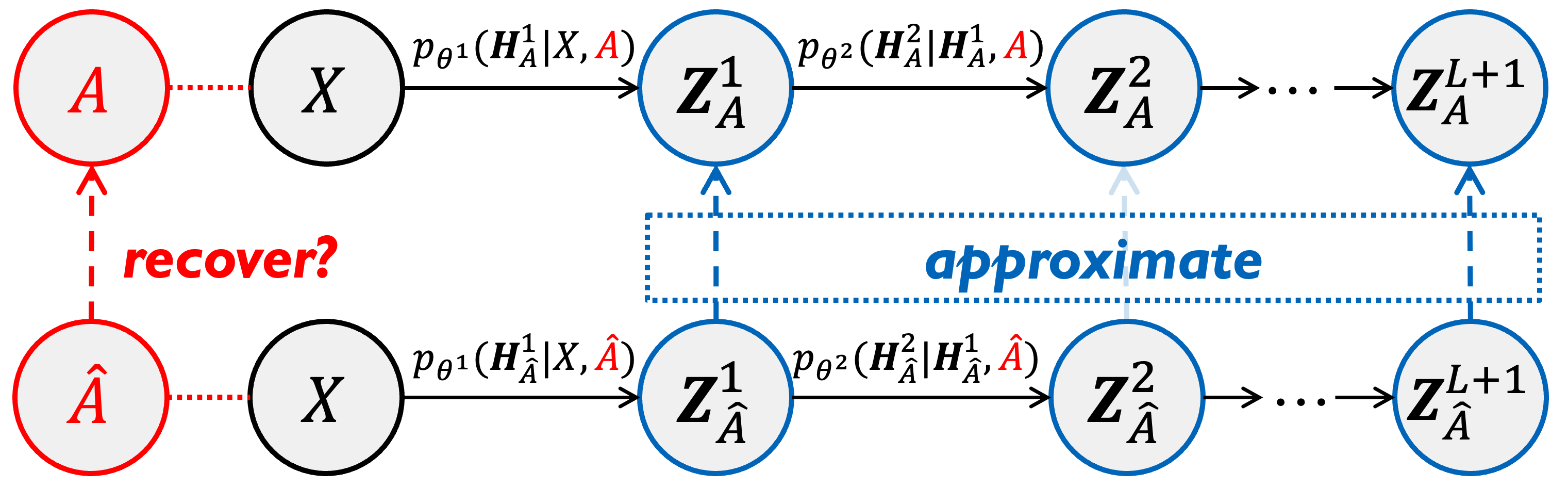}
		\vspace{-4px}
		\caption{
			Modeling the GRA problem as
			approximating the original Markov chain (upper) by the attack chain (lower).
			Note that
			the original chain is with the original adjacency $A$,
			while the attack chain is with the recovered adjacency $\bm{\hat{A}}$.
		}
		\label{fig: GRA-two-chains}
		\vspace{-10px}
	\end{figure}

	\section{Related Work}
	\label{sec: related work}
	
	
	\textbf{Inversion attacks on images.} 
	Pioneer works~\citep{szegedy2013intriguing, fredrikson2014privacy, fredrikson2015model, hidano2017model}
	introduced the Model Inversion Attack (MIA) with shallow models and 
	justified the feasibility of MIA
	in recovering the monochrome images.
	However, 
	they fail in attacking deep models for image classification tasks,
	where the reconstructed images are of low fidelity.
	Generative model inversion~\citep{zhang2020secret}
	is the first to conduct MIA on 
	convolution neural networks.
	Instead of directly reconstructing the private data from scratch,
	its inversion process is guided by a distributional prior 
	through the generative adversarial networks (GAN)
	that can reveal private training data of the target model with high fidelity.
	Later,
	variational model inversion~\citep{wang2021variational}
	further formulates the MIA as the variational inference.
	It generally can bring a higher attack accuracy and diversity
	for its equipped powerful generator StyleGAN to optimize its designed variational objective.
	Recent advance~\citep{struppek2022ppa}
	significantly decreases the cost of conducting MIA
	through relaxing the dependency between the target model and the image prior.
	It enables the use of a single GAN to attack a wide range of targets,
	requiring only minor adjustments to the attack.
	It shows that MIA is possible 
        with publicly available pre-trained GANs 
	under strong distributional shifts.

	\textbf{Inversion attacks on graphs.}
	%
	Early works~\citep{duddu2020quantifying, chanpuriya2021deepwalking}
	attempt to reconstruct the target graph from released graph embeddings of each node that are generated by Deepwalk or GNNs.
	The link stealing attack~\citep{he2021stealing} 
	is the first work to steal links 
        from a GNN as the target model.
	It aims to conduct the MIA
	with three kinds of prior knowledge, 
	including node features, 
        partial target graph, 
	and a shadow dataset.
        It considered all permutations of the three elements
        and proposed eight kinds of attack methods in total
	that are adaptive to the eight scenarios
        with chemical networks and social networks.
	%
	Another recent work~\citep{zhang2021graphmi}
	is a learnable attack that also aims to 
	recover the links of the original graph.
	With the white-box access to the target GNN model,
	the optimal adjacency is obtained 
	by maximizing the classification accuracy 
	regarding the known node labels.
	Please refer to Appendix.~\ref{sec: related_work}
	for a detailed introduction to related work.
	
	\section{Preliminaries}
	\label{sec: preliminaries}
	
	
	\textbf{Notations.}	
	With adjacent matrix $A$ and node features $X$,
	an undirected graph is denoted as $\mathcal{G} \! = \! (A, X)$,
	where $A_{ij} \! = \! 1$ means there is an edge $e_{ij}$ between $v_i$ and $v_j$.
	For each node $v_i$,
	its $D$-dimension node feature is denoted as $X_{[i,:]} \! \in \! \mathbb{R}^{D}$,
	and its label $y_i \! \in \! Y$ indicates the node category.
	The node classification task is to predict the label 
        $Y$ of each node
	via a parameterized model $f_{\bm \theta}(\cdot)$.
        $I(X; Y)$ indicates the mutual information of variables $X$ and $Y$.
	We summarize the frequently used notations in Table~\ref{tab:notations}
	of Appendix.~\ref{ssec: notations}
	
	\textbf{Graph neural networks.}	
	Predicting node labels 
	requires a parameterized hypothesis 
	$f_{\bm{\theta}}$
	with GNN architecture and the message propagation framework~\citep{gilmer2017neural}.
	Specifically,
	the forward inference of a $L$-layer GNN
	generates the node representations $\bm{H}_{A} \! \in \! \mathbb{R}^{N \times D}$ 
	by a $L$-layer message propagation.
	The follow-up linear layer
	transforms 
	the representations $\bm{H}_{A}$  to 
	the classification probabilities
	$\bm{\hat{Y}}_{A} \! \in \! \mathbb{R}^{N \times C}$,
	with $C$ categories of nodes in total.


	\textbf{Model inversion attack on graphs.}
	%
	%
	In this study, to catch more attention to the privacy risk of GNNs, we study the reconstruction of the graph adjacency by MIA
	and term it Graph Reconstruction Attack (GRA),
	as elaborated below.
	\begin{definition}[Graph Reconstruction Attack]
		\label{def: graph reconstruction attack}
		Given a set of prior knowledge $\mathcal{K}$
		and a trained GNN $f_{\bm{\theta}^*}(\cdot)$,
		the graph reconstruction attack aims to recover the original linking relations $\bm{\hat{A}}^*$
		of the training graph $\mathcal{G}_{\text{train}} \! = \! (A, X)$,
		namely,
		\begin{align}
			\text{GRA: \; \;}&  
			\bm{\hat{A}}^* = \arg \max_{\bm{\hat{A}}} \mathbb{P}(\bm{\hat{A}} | f_{\bm{\theta}^*}, \mathcal{K}).
			\label{eqn: graph_MI_attack}
		\end{align}
		\vspace{-20pt}
	\end{definition}
	Here, $\mathbb{P}(\cdot)$ is the attack method to generate $\bm{\hat{A}}$,
	and $\mathcal{K}$ can be any subset of 
	$\{ X, Y, \bm{H}_{A}, \bm{\hat{Y}}_{A} \}$.
	Note that GRA is conducted in a post-hoc manner, \textit{i.e.}, after the training of GNN $f_{\bm{\theta}}(\cdot)$.
	
	
	
	\section{A Comprehensive Study of GRA}
	\label{sec: overview}
	
	In this section,
	we formulate the Graph Reconstruction Attack as a Markov chain approximation problem (Sec.~\ref{ssec: problem statement}), quantify the privacy risk of releasing the non-sensitive features
	(Sec.~\ref{ssec: understanding by direct MI calculation}),
	and investigate the training dynamics of graph representations \textit{w.r.t.} the privacy leakage
	(Sec.~\ref{ssec: tracking by graph information plane}).
	
	\subsection{Modeling GRA as Markov chain approximation.}
	\label{ssec: problem statement}
	%
	To adaptively support the white-box GRA 
	that leverages the target model and any prior knowledge;
	and to properly locate, present, and utilize the interplaying variables of GNN forward 
	in a generic manner; 
	we cast the GRA problem as approximating the 
	original Markov chain \texttt{ORI-chain} by the attack chain \texttt{GRA-chain}, 
	as shown in Fig.~\ref{fig: GRA-two-chains}, namely,
	%
	\vspace{-4pt}	
	\begin{equation}	
		\begin{split}
			\texttt{ORI-chain:}& \bm{Z}^{0} 
			\! \xrightarrow[\bm{\theta}^1]{A} \! \bm{Z}_{A}^1 
			\! \xrightarrow[\bm{\theta}^2]{A} \! \bm{Z}_{A}^2  
			\! \to \! \cdots 
			\! \xrightarrow[\bm{\theta}^{L \! + \! 1}]{A} \! \bm{Z}_{A}^{L \! + \! 1}, \\
			\texttt{GRA-chain:}& \bm{Z}^{0} 
			\! \xrightarrow[\bm{\theta}^1]{\bm{\hat{A}}} \! \bm{Z}_{\bm{\hat{A}}}^1 
			\! \xrightarrow[\bm{\theta}^2]{\bm{\hat{A}}} \! \bm{Z}_{\bm{\hat{A}}}^2  
			\! \to \! \cdots 
			\! \xrightarrow[\bm{\theta}^{L \! + \! 1}]{\bm{\hat{A}}} \! \bm{Z}_{\bm{\hat{A}}}^{L \! + \! 1},
		\end{split}
		\label{eqn: two-Markov-chains}
	\end{equation}
	where $\bm{Z}^0 \! = \! X$,
	$\bm{Z}_{A}^i \! = \! \bm{H}_{A}^i$ 
	for $i \!= \! 1, \! \cdots \!, L$
	and $\bm{Z}_{A}^{L+1} \! = \! \bm{\hat{Y}}_{A}$.
	%
	Note that GNNs' forward can be seen as a Markov chain
	that is discrete-time finite, non-reversible, and pairwise-independent.
	The probability of current state $\bm{H}_{A}^{i}$ 
	only depends on the previous state $\bm{H}_{A}^{i-1}$,
	where the transition kernel 
	is determined by $A$ and $\bm{\theta}^i$.
	%
	Importantly,
	the principle of GRA to recover the adjacency $A$ by $\bm{\hat{A}}$
	is to approximate latent variables 
	$\mathcal{S}_{A} \! = \! \{\bm{Z}_{A}^{i}: \bm{Z}_{A}^{i} \! \in \! \mathcal{K} \}$
	in \texttt{ORI-chain}
	by the corresponding
	$\mathcal{S}_{\bm{\hat{A}}} \! = \! \{\bm{Z}_{\bm{\hat{A}}}^{i}: \bm{Z}_{A}^{i} \! \in \! \mathcal{S}_{A}\}$
	in \texttt{GRA-chain}.
	
	\subsection{What leaks privacy in \texttt{ORI-chain}?}
	\label{ssec: understanding by direct MI calculation}

	\begin{table}[t!]
		\centering
		\caption{
			Quantitative analysis of $I(A; \bm{Z})$ with AUC metric under range $[0,1]$.
			A higher AUC value means a severer privacy leakage.
			"---" indicates that nodes in this dataset do not have features.
			Besides, the \textbf{boldface} numbers mean the best results, 
			while the \underline{underlines} indicate the second-bests.
                The target model $f_{\bm{\theta}}$ is a two-layer GCN by default.
		}
		\vspace{-8px}
		\label{tab:understanding-MI-term-comparison}
		\fontsize{8}{10}\selectfont
		\setlength\tabcolsep{5.2pt}
		\begin{tabular}{c|cccccc}
			\toprule
			MI & Cora & Citeseer &  Polblogs   & USA   &   Brazil &  AIDS  \\
			\midrule
			$I(A; X)$ &  \underline{.781} & \textbf{.881} & --- & --- & --- &  .521 \\ 
			$I(A; \bm{H}_A)$ & .766 & .760 & \textbf{.763} & \textbf{.850} & \textbf{.758} & \textbf{.584} \\
			$I(A; \bm{\hat{Y}}_A)$  &  .712 & .743 & \underline{.772} & \underline{.826} & \underline{.732} & \underline{.561}  \\
			$I(A;Y)$  & \textbf{.815} & \underline{.779} & .705 & .728 & .613 & .536 \\
			\bottomrule
		\end{tabular}
	\vspace{-4px}
	\end{table}

	\begin{table}[t!]
		\centering
		\caption{
			An ensemble study on the prior knowledge with AUC metric.
			For a generic evaluation,
			it is assumed that node feature $X$ is accessible (if exists),
			based on which we evaluate all the possible $8$ combinations
			with $2$, $3$, or $4$ components, 
			where "$\checkmark$" means accessible for this variable.
		}
		\vspace{-8px}
		\label{tab:understanding-MI-term-ensemble}
		\fontsize{8}{10}\selectfont
		\setlength\tabcolsep{3.4pt}
		\begin{tabular}{cccc|cccccc}
			\toprule
			$X$ & $\bm{H}_A$ &  $\bm{\hat{Y}}_A$  &  $Y$  & Cora & Citeseer &  Polblogs   & USA   &   Brazil &  AIDS \\
			\midrule
			$\checkmark$ & $\checkmark$ &  &  & .781 &.881 & .763 & .850 & .758 & .521 \\
			$\checkmark$ &  & $\checkmark$ &   &  .781 & .881 & .772 & .826 & .732 & .521 \\
			$\checkmark$ &  &  & $\checkmark$  &  .849 & .907 & .705 & .728 & .613 & .522 \\ 
			\midrule
			$\checkmark$ & $\checkmark$ & $\checkmark$ &   &  .781 & .881 & .763 & .848 & .756 & .521 \\ 
			$\checkmark$ & $\checkmark$ & & $\checkmark$   & .849 & .907 & .779 & .850 & .743 & .522 \\ 
			$\checkmark$ &  & $\checkmark$ & $\checkmark$   &  .842 & .907 & .785 & .842 & .730 & .522 \\ 
			\midrule
			$\checkmark$ & $\checkmark$  & $\checkmark$ & $\checkmark$  & .849 & .907 & .781 & .852 & .717 & .522 \\ 
			\bottomrule
		\end{tabular}
	\vspace{-10px}
	\end{table}

	Intuitively,
	variables in \texttt{ORI-chain} might contain information about ground-truth adjacency $A$
	as the transition kernel is partially determined by $A$.
	To figure out,
	we quantify the direct correlation between $A$
	and a single variable 
	$\bm{Z} \! \in \! \{ X, Y, \bm{H}_{A}, \bm{\hat{Y}}_{A} \}$
	in \texttt{ORI-chain}
	%
	through the informative concept of mutual information (MI) $I(A; \bm{Z})$.
	%
	Following link prediction works~\citep{zhang2018link, zhu2021neural},
	the AUC 
        metric is utilized
	to quantify $I(A; \bm{Z})$ regarding 
	edges in $A$ and $\hat{A}_{\bm{Z}} \! = \! \sigma(\bm{Z} \bm{Z}^\top)$,
	where $\sigma(\cdot)$ is the activation function.
	Here, the inner product transforms 
	the informative variable $\bm{Z} \! \in \! \mathbb{R}^{N \times D}$ 
	to the predictive adjacency 
	$\hat{A}_{\bm{Z}} \! \in \! \mathbb{R}^{N \times N}$, 
	where the $(i,j)$ entry in $\hat{A}_{\bm{Z}}$
	indicate the existence of edge $e_{ij}$. 
	See Appendix.~\ref{ssec: quantifying privacy leakage} for details.
	
	%
	%

	\begin{figure*}[t!]
		\centering
		\hfill
		\includegraphics[width=8.0cm]{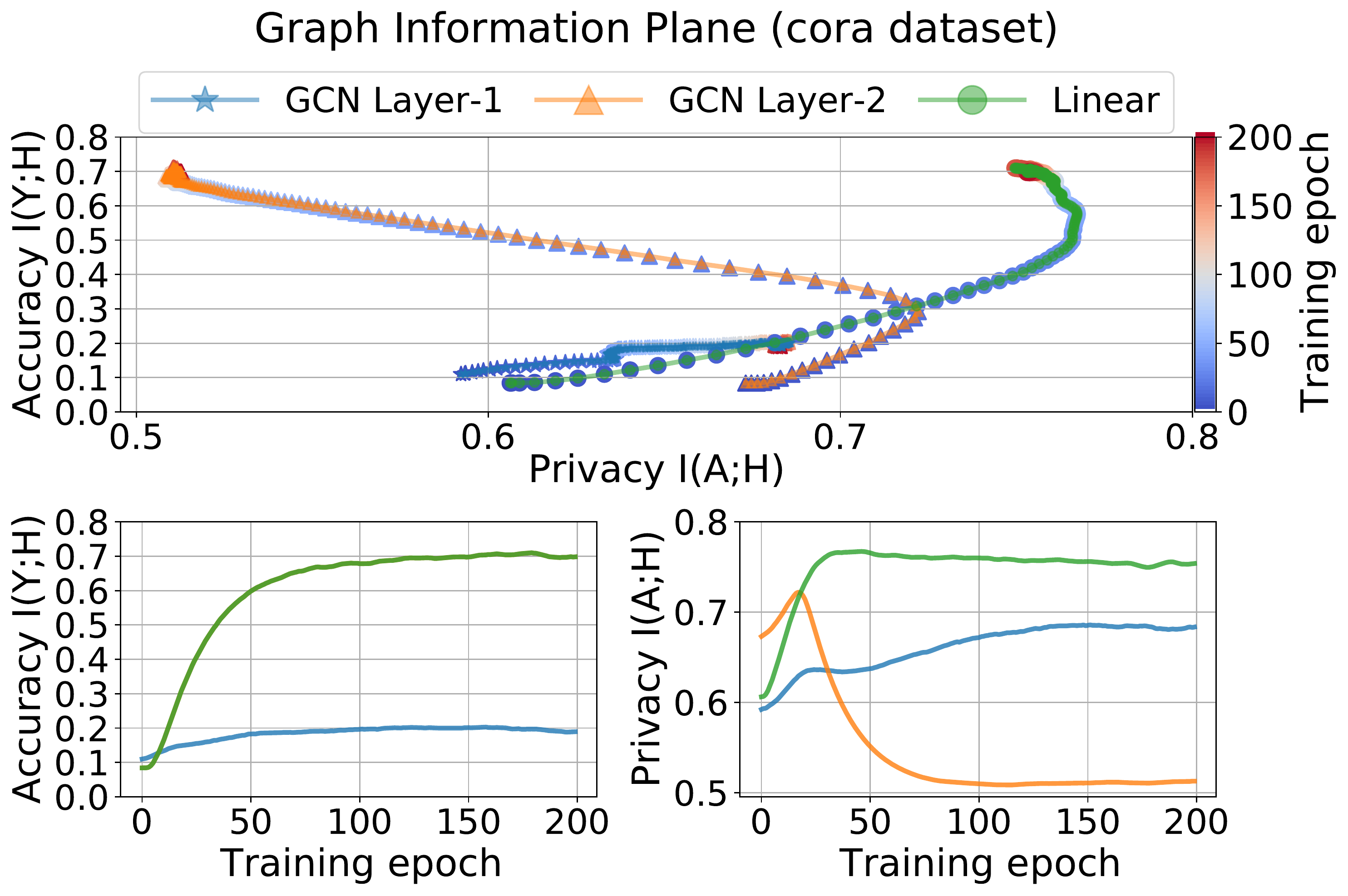}
		%
		\hfill
		\includegraphics[width=8.0cm]{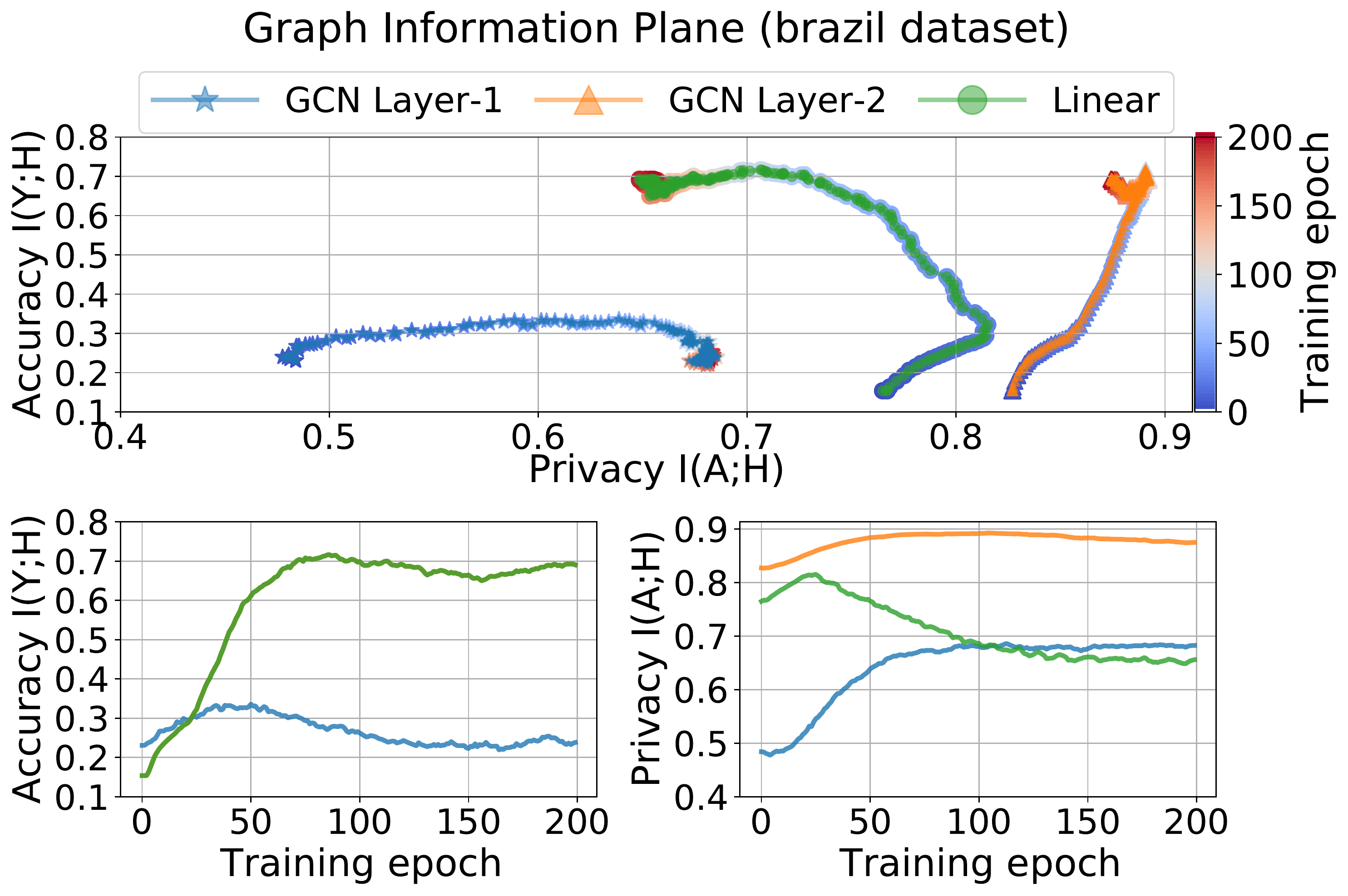}
		\hfill
		\vspace{-6px}
		\caption{
			Graph information plane:
                tracking the
                standard training procedures 
                of a two-layer GCN on Cora (left) and Brazil (right).
			The accuracy of GCN layer-2 and the linear layer is the same 
			as $\hat{\bm{Y}}_A \! = \! \texttt{Linear}(\bm{H}_A^2)$
                (thus with overlapped curves).
		}
		\label{fig: graph-information-plane}
		\vspace{-10px}
	\end{figure*}

	\textbf{Observation 3.1.}
	As shown in Tab.~\ref{tab:understanding-MI-term-comparison},
	a single variable in \texttt{ORI-chain} can recover the original adjacency to     a certain extent through the inner-product transformation. 
	It is applicable to black-box attacks once obtain these variables.
	%
        Besides,
	the model outputs $\{\bm{H}_{A}, \bm{\hat{Y}}_{A}\}$
	generally contain more adjacency information
	than the original data $\{X,Y\}$.
	As single variables $\bm{Z}$ in \texttt{ORI-chain} present diverse approximation power,  
	the stored private information might be complementary to each other in recovering adjacency.
	To answer, we ensemble these variables via a linear combination, namely,
	$\bm{\hat{A}}_{esm} \! = \! \nicefrac{1}{|\mathcal{K}|} \sum_{i=1}^{|\mathcal{K}|} \hat{A}_{\mathcal{K}_i}$,
	where $\hat{A}_{\mathcal{K}_i} \! = \! \sigma({\mathcal{K}_i} {\mathcal{K}_i}^\top)$.
	%
	%
	
	\textbf{Observation 3.2.}
	As shown in Tab.~\ref{tab:understanding-MI-term-ensemble},
	the straightforward and linear combination of informative terms
	only brings marginal improvements in recovering the adjacency.
	Such an observation is consistent with the chain rule of MI, \textit{i.e.},
	$\forall \mathcal{K}_i, \mathcal{K}_j \in \mathcal{K}, \;
        I(A;\mathcal{K}_i, \mathcal{K}_j) \geq \max \big( I(A;\mathcal{K}_i), I(A; \mathcal{K}_j) \big)$.
	%
	
	\vspace{-6pt}
	\subsection{How \texttt{ORI-chain} memorizes the privacy?}
	\label{ssec: tracking by graph information plane}



	\begin{figure*}[t!]
		\centering
		\subfigure[
		The attack framework MC-GRA.
		In forward, a recovered adjacency $\bm{\hat{A}}$
		is sampled from the parameterized distribution $\mathbb{P}_{\bm{\phi}}(\bm{\hat{A}})$ 
		and injected with manual stochasticity.
		As for backward, the learnable parameters $\bm{\phi}$ gain supervision from the MC-GRA objective Eq.~\eqref{eqn: MC-GRA}.
		]
		{\includegraphics[width=8.3cm]{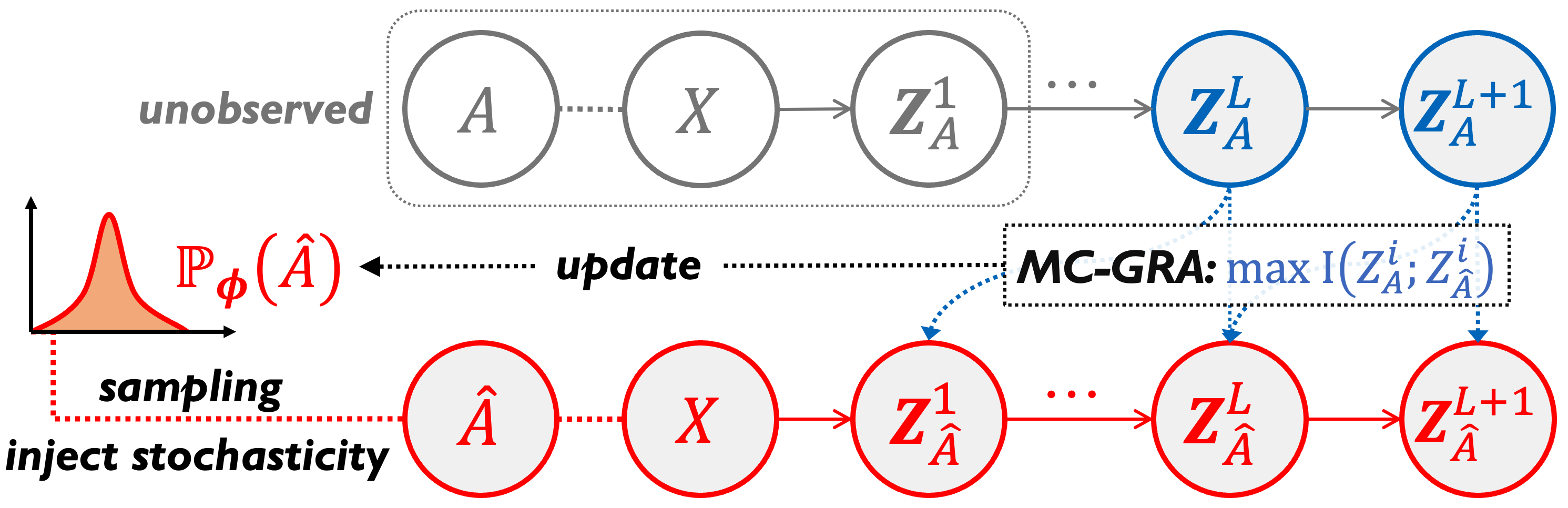}
			\label{fig: markov-attack}}
		\hfill
		\subfigure[
		The defense framework MC-GPB.
		It solves the accuracy-privacy tradeoff by objective Eq.~\eqref{eqn: tighter-defense-MI} 
		through regularizing graph representations to make GNNs forget about private $A$ and injecting stochasticity to promote forgetting that decreases the privacy risk further.
		]
		{\includegraphics[width=8.3cm]{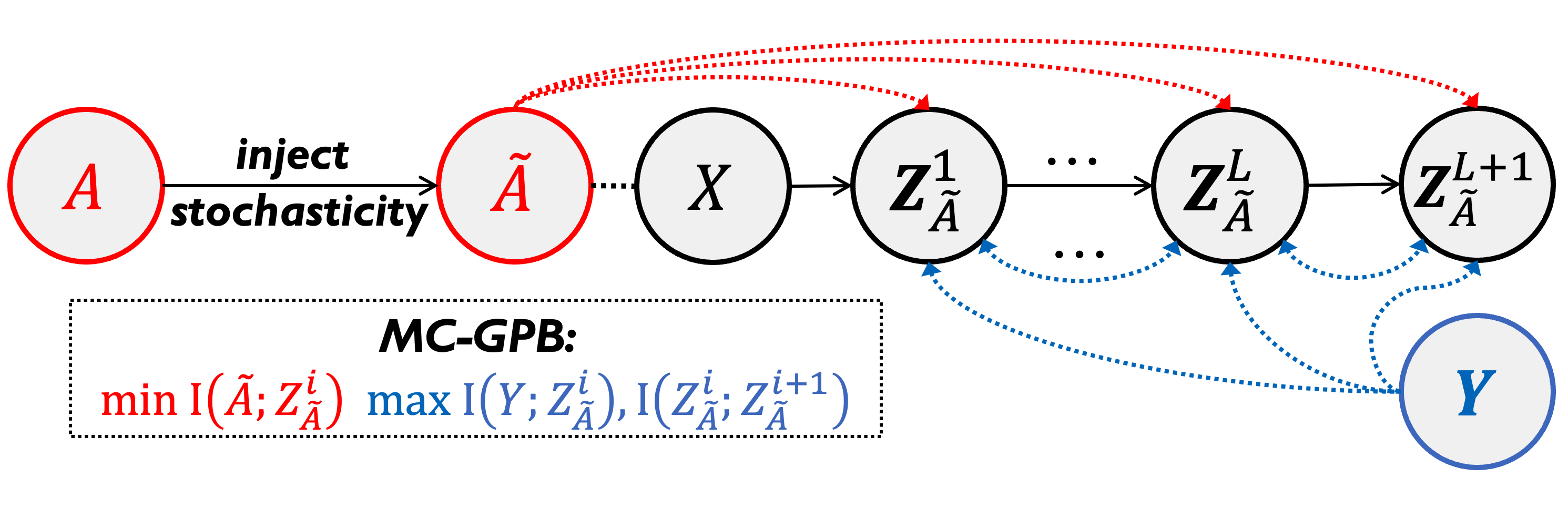}
			\label{fig: markov-defense}}
		\vspace{-10px}
		\caption{
			Illustrations of the two proposed methods for strengthening (a) and defending (b) the GRA, respectively.
		}
	\vspace{-10px}
	\end{figure*}

	For further understanding of the learning and memorization mechanisms of \texttt{ORI-chain}
	and acquiring inspiration for devising the corresponding defense approach,
	we track the training process 
	by privacy $I(A; \bm{Z})$ and accuracy $I(Y; \bm{Z})$,
	where variable $\bm{Z} \! \in \! \{ \bm{H}_A^1, \bm{H}_A^2, \bm{\hat{Y}}_A \}$ are from \texttt{ORI-chain}.
	Conceptually,
	we derive Graph Information Plane 
	\footnote{We leave the formal definition and details in Appendix.~\ref{ssec: graph information plane}.}
	inspired by information theory~\citep{tishby2015deep, shwartz2017opening}.
	%
	The anytime $\bm{Z}$ in training phase
	is projected to the two-dimensional $\big(I(A; \bm{Z}),  I(Y; \bm{Z})\big)$ plane.

	\textbf{Observation 3.3.}	
	As shown in Fig.~\ref{fig: graph-information-plane},
	the training procedure with v-shape curves 
	contains two main phases:
	\textit{fitting} and \textit{compressing}.
	In the first and shorter phase, 
	the layers increase the information about privacy.
	While in the second and longer phase,
	the layers gradually forget about privacy.

	\section{To Recover Better, You Must Extract More}
	\label{sec: GRA attack}
	
	To attack, one must integrate all the available prior knowledge
	to backward recover the adjacency.
	The key challenge here is the lack of an effective way 
	to employ all the prior knowledge and the target model in attacks.
	Besides, it is also hard to represent and update the recovered adjacency 
	in a differentiable way due to the discrete nature of adjacency.

	To solve,
	we propose the Markov Chain-based Graph Reconstruction Attack (MC-GRA) framework,
	as illustrated in Fig.~\ref{fig: markov-attack}. 
	Here, instead of directly maximizing $I(\bm{\hat{A}}; \mathcal{K})$,	
	we choose to promote $I(\bm{Z}_{\bm{\hat{A}}}^i; \bm{Z}_A^i \! = \! \mathcal{K}_i)$
	as it provides supervision signals that can be tractably approximated.
	To be specific,
	we adopt the aforementioned chain-based modeling
	for extracting the knowledge stored in the target model 
	while utilizing all the prior knowledge for optimization simultaneously.
	\footnote{The detailed deriving is elaborated in Appendix.~\ref{ssec: deriving MC-GRA}.}
	The relaxation power hails from
	approaching the known variable $\mathcal{K}_i$ of \texttt{ORI-chain} 
	by the locationally corresponding $\bm{Z}_{\bm{\hat{A}}}^i$ generated by \texttt{GRA-chain},
	namely,
	%
	%
	%
	\begin{equation}	
		\begin{split}
			\text{MC-GRA: }
			\bm{\hat{A}}^* = \arg \max_{\bm{\hat{A}}} 
			\underbrace{\alpha_{p} I(\bm{H}_A ; \bm{H}^{i}_{\bm{\hat{A}}})}_{\text{propagation approximation}} \\
			+ \underbrace{\alpha_{o} I(\bm{Y}_A ; \bm{Y}_{\bm{\hat{A}}})
				+ \alpha_{s} I(Y ; \bm{Y}_{\bm{\hat{A}}})}_{\text{outputs approximation}}
			- \underbrace{\alpha_{c} H(\bm{\hat{A}})}_{\text{complexity}}.
		\end{split}
		\label{eqn: MC-GRA}
	\end{equation}
	Note that MC-GRA is a maximin game:
	it maximizes the approximation of forward processes of the two Markov chains,
	while minimizing the complexity in each transition with $\bm{\hat{A}}$ 
	to avoid trivial solutions by constraining the density.
	
	\begin{remark}
		The adaptive power of MC-GRA comes from its leveraging any prior knowledge set.
		That is, the propagation approximation term in Eq.~\eqref{eqn: MC-GRA}
		for $\bm{H}_A$ works once obtained,
		while the outputs approximation term for $\bm{\hat{Y}}_A$ and $Y$.
		Thus, it can be utilized for all the $7$ settings in Tab.~\ref{tab:understanding-MI-term-ensemble}.
	\end{remark}

	\textbf{Parameterize Eq.~\eqref{eqn: MC-GRA} with different forms.}
	For approximating the original adjacency in a learnable manner,
	the recovered adjacency is parameterized and updated with the relaxed objective.
	Each time forward,
	an adjacency $\bm{\hat{A}}$ 
	is sampled from its parameterized distribution 
	as $\bm{\hat{A}} \! \sim \! \mathbb{P}_{\bm{\phi}}(\bm{\hat{A}})$.
	Technically,
	three implementations of $\mathbb{P}_{\bm{\phi}}(\bm{\hat{A}})$ 
	with learnable weights $\bm{\phi}$ are listed below with increasing complexity.
	\vspace{-6pt}
	\begin{itemize}[leftmargin=*]
		\setlength\itemsep{0.1em}
		\item
		Formulating $\bm{\hat{A}}$ as the only learnable parameter and directly optimizing it, \textit{i.e.},
		$\mathbb{P}_{\bm{\phi}}(\bm{\hat{A}}) = \bm{\hat{A}} \in  [0,1]^{N \times N}$. 
		
		\item
		A Gaussian distribution $\mathbb{P}_{\bm{\phi}} = \mathcal{N}(\bm{\mu}, \bm{\sigma}^2)$
		with two learnable parameters $\bm{\mu}, \bm{\sigma} \in  [0,1]^{N \times N}$.
		That is utilized to generate $\bm{\hat{A}}$ 
		as	$\bm{\hat{A}} = \bm{\mu} + \epsilon \bm{\sigma}$,
		where random noise $\epsilon \sim \mathcal{N}(0,1)$.
		
		\item
		A parameterized generator $f_{\bm{\phi}}(\cdot)$ 
		initialized with the same architecture and weights as $f_{\bm{\theta}^*}(\cdot)$.
		It generates the probabilistic distribution
		by $\mathbb{P}_{\bm{\phi}} \! = \! \sigma (\bm H_{I} \bm H_{I} ^ \top) \! \in \!  [0,1]^{N \times N}$,
		where $I$ is the identity matrix
		and $\bm H_{I} \! = \! f_{\bm \phi}(I, X)$.
	\end{itemize}
	\vspace{-6pt}
	
	
	\begin{figure*}[t!]
		\centering
		\subfigure[Standard training.]
		{\includegraphics[width=5.5cm]{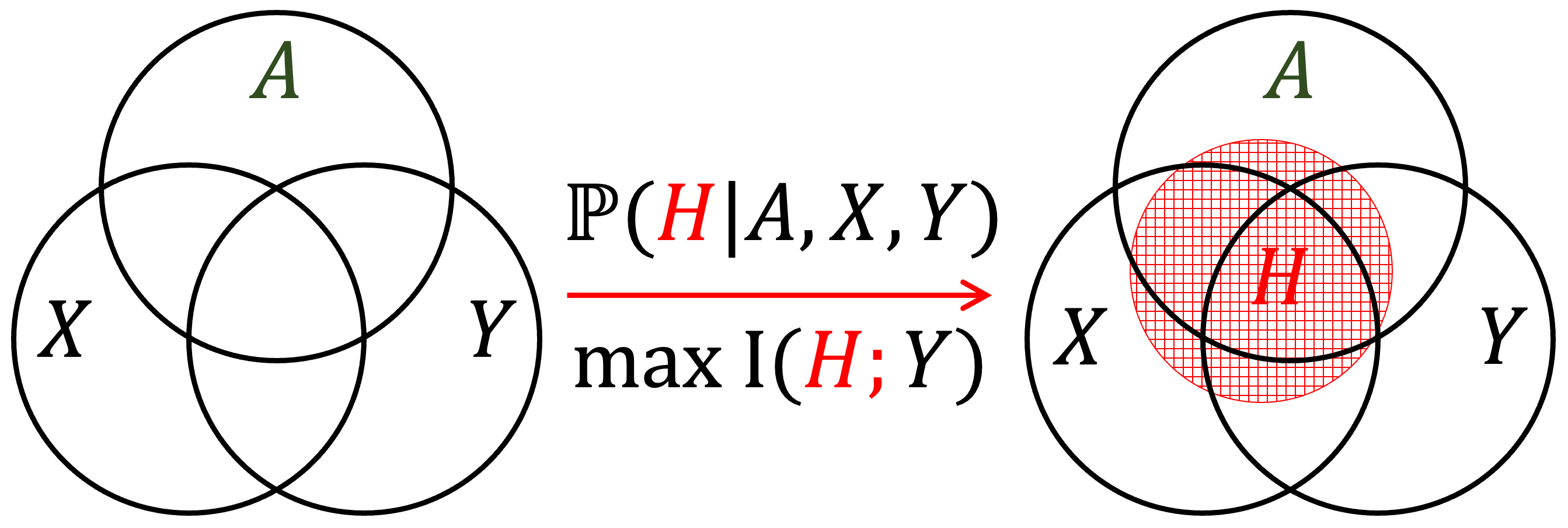}
			\label{fig: three-ball-standard-training}}
		\hfill
		\subfigure[Reconstruction attack by MC-GRA.]
		{\includegraphics[width=5.5cm]{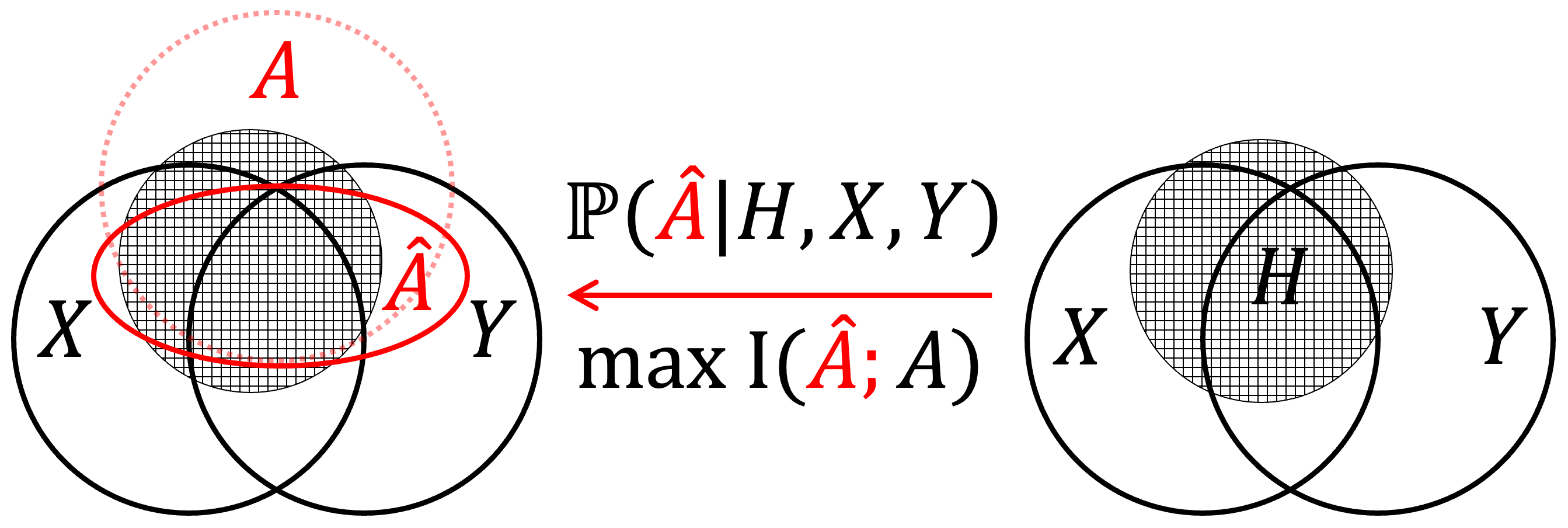}
			\label{fig: three-ball-GRA}}
		\hfill
		\subfigure[Defensive training by MC-GPB.]
		{\includegraphics[width=5.5cm]{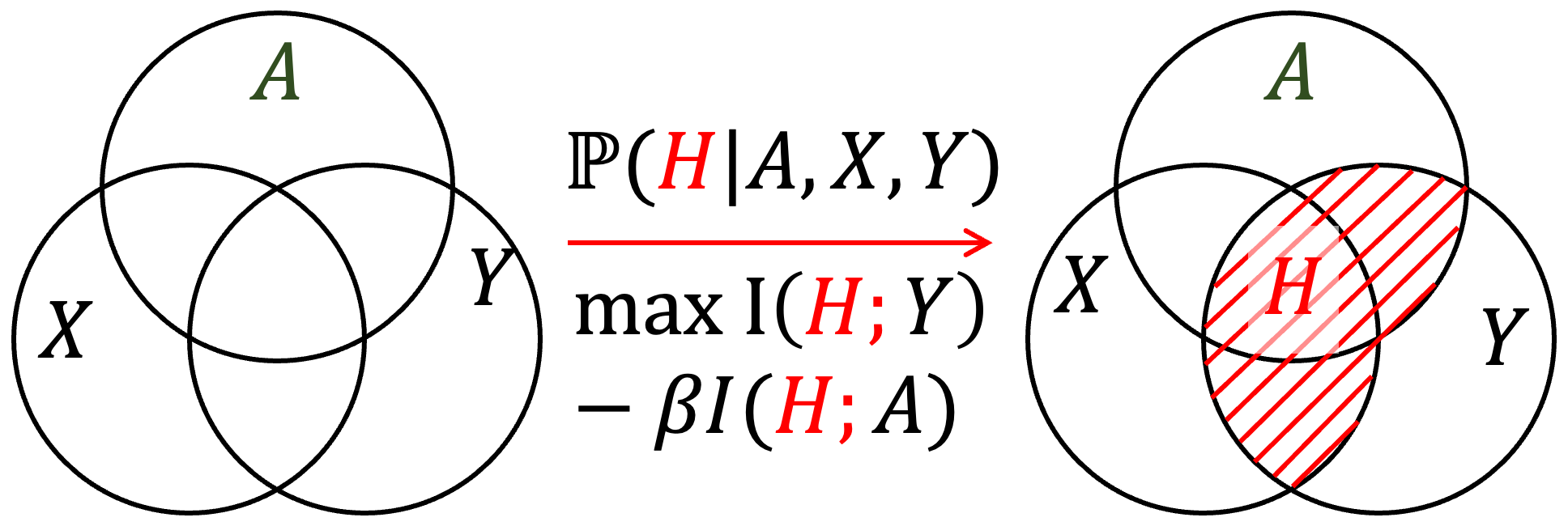}
			\label{fig: three-ball-GPB}}
		\vspace{-6px}
		\caption{
			Illustrations of the information properties
			regarding the training, attacking, and defending processes.
		}
		\label{fig: information properties of training and attack}
	\vspace{-10px}
	\end{figure*}

	\textbf{Optimize Eq.~\eqref{eqn: MC-GRA} with injected stochasticity.}
	Considering that 
	both $\bm{\hat{A}}$ and $X$ contribute to 
	to $\bm{Z} \in \{ \bm{H}_{A}, \bm{\hat{Y}}_{A}, Y\}$,
	the mutual dependence among these three variables is coupled together.
	%
	The spurious correlation $I(X; \bm{\hat{A}} | \bm{Z})$, 
	possibly degenerates the effectiveness of GRA~\citep{yang2022understanding, miao2022interpretable}.
	To solve,
	we inject stochasticity
	to further remove the spurious correlation among $\bm{\hat{A}}, X$ and $\bm{Z}$,
	where the probability of spurious correlation naturally increases with the length of the Markov chain.
	%
	Specifically,
	the debias power comes from
	the lower MI $I(\tilde{X}; \bm{\tilde{A}} | \bm{Z})$, 
	where $\tilde{X}, \bm{\tilde{A}}$ are perturbed
	as $\tilde{X} \! = \! X \oplus X_{\epsilon}, \bm{\tilde{A}} \! = \! \bm{\tilde{A}} \oplus A_{\epsilon}$.
	%
	%
	Technically,
	for each potential edge $e_{ij}$, 
	its existence $a_{ij}$ is sampled from a Bernoulli distribution, \textit{i.e.},
	$a_{ij} \! \sim \! \text{Bern}(\bm{p}_{ij})$,
	$a_{ij} \! \in \! \{0,1\}$, and $\bm{p}_{ij} \! = \! \bm{\hat{A}}_{ij} \! \in \! [0,1]$.
	To cooperate with the stochasticity
	and enable the back-propagation of gradients,
	the Gumbel-softmax reparameterization~\citep{kool2019stochastic, xie2019reparameterizable} is applied.
	%
	That is, the edge probabilities are perturbed as
	$\tilde{\bm{p}}_{ij} \! = \! \bm{p}_{ij} \! - \! \log( \! - \! \log (\epsilon))$,
	where $\epsilon \! \sim \! \text{Uniform}(0,1)$.
	
	

        \vspace{+2pt}
	\begin{remark}
		The incremental contribution of $\bm{\hat{A}}$ regarding $X$ to approximate $\bm{Z}$
		is $I(\bm{Z}; \bm{\hat{A}} | X) \! = \! I(\bm{Z}; A, X) \! - \! I(\bm{Z}; X)$.
		%
		%
		Here, a general solution for obtaining a more informative $\bm{\hat{A}}$
		is to reduce $I(\bm{Z}; X)$ via perturbation and promote $I(\bm{Z}; \bm{\hat{A}})$.
	\end{remark}
	
	\textbf{Theoretical analysis about Eq.~\eqref{eqn: MC-GRA}.}
	A rigorous analysis is conducted
	on the basis of information properties shown in Fig.~\ref{fig: information properties of training and attack}.
	The non-invertible nature of GNN forwarding,
	which hails from the adopted non-linear operations,
	decreases the information entropy by layers
	and forms a bottleneck that extracts informative signals from the input data.
	%
	As a result, the MI of two Markov chains (Eq.~\eqref{eqn: two-Markov-chains})
	is decreasing by layers, which is elaborated in the following Theorem~\ref{theorem: reducing MI with two chains}.
	Based on this,
	we derive a tractable bound in Theorem~\ref{theorem: GRA attack lower bound}
	to estimate the attack fidelity without the ground truth $A$.
	
	\begin{theorem}
		\label{theorem: reducing MI with two chains}
		The layer-wise transformations ${\bm{Z}_{A}^i \! \to \! \bm{Z}_{A}^{i+1}}$ are non-invertible,
		\textit{e.g.}, ${ \bm{Z}_{A}^{i+1} \! = \! \sigma(\psi(A) \cdot \bm{Z}_{A}^i \cdot \bm{\theta}^{i}) }$,
		where $\psi(A)$ is the graph convolution kernel, as in Eq.~\eqref{eqn: two-Markov-chains}.
		It leads to a lower MI between the two Markov chains, \textit{i.e.},
		${ I(\bm{Z}_{A}^{i}; \bm{Z}_{\bm{\hat{A}}}^{i}) \! - \! 
			I(\bm{Z}_{A}^{i+1}; \bm{Z}_{\bm{\hat{A}}}^{i+1}) \! \geq \! 0 }$.
		$\!$Proof. See Appendix.\ref{ssec: proof of reducing MI with two chains}.
	\end{theorem}
	%
	
        \vspace{+2pt}
	\begin{theorem}[Tractable Lower Bound of Fidelity]
		\label{theorem: GRA attack lower bound}
		The attack fidelity satisfies
		$I(A; \bm{\hat{A}}) 
		\! \geq \! H(\bm{H}_{A}) \! - \! H_b(e) \! - \! P(e) \log(|\mathcal{H}|)$,
		where 
		$P(e) \! \triangleq \! P(\bm{H}_{A} \! \neq \! \bm{H}_{\bm{\hat{A}}})$ is the probability of approximation error,
		$\mathcal{H}$ denotes the support of $\bm{H}_{A}$,
		and $H_b(\cdot)$ is the binary entropy.
		\textit{Proof}. See Appendix.~\ref{ssec: proof of GRA attack lower bound}.
	\end{theorem}
	
	%
	The estimated $I(A; \bm{\hat{A}})$ can be a valuable reference
	when conducting GRA that maximizes such a MI term (see Fig.~\ref{fig: three-ball-GRA}).
	Besides, Theorem~\ref{theorem: GRA attack lower bound} also indicates that 
	a higher approximation $I(\bm{H}_{A} ; \bm{H}_{\bm{\hat{A}}})$
	with a lower error $P(e)$ can bring a higher $I(A; \bm{\hat{A}})$ that indicates a higher attack fidelity.
	Then, we indicate the worst privacy leakage
	with the optimal attack fidelity 
        as the upper bound 
	in following Theorem~\ref{theorem: optimal fidelity in GRA}.

	\begin{theorem}[The Optimal Fidelity]
		\label{theorem: optimal fidelity in GRA}
		The recovering fidelity 
		satisfies
		${I(A; X,Y,\bm{H}_{A}) \! - \! I(A; \bm{\hat{A}}) \! \geq \! 0}$.
		%
            Solving MC-GRA sufficiently yields a solution to achieve the optimal case,
            \textit{i.e.},
		$I(A; \bm{\hat{A}}^{*}) \! = \! I(A; X,Y,\bm{H}_{A})$.
		%
		Proof. $\!$See Appendix.~\ref{ssec: proof of optimal fidelity in GRA}.
	\end{theorem}

	Theorem~\ref{theorem: optimal fidelity in GRA} 
        indicates that MC-GRA
	is capable of achieving optimal recovering fidelity.
	Nonetheless, the remaining information of $A$,
	\textit{i.e.}, $H(A | \bm{\hat{A}}^{*}) \! = \! H(A | \mathcal{K})$,
	is unobservable from $\mathcal{K} \! = \! \{X, Y, \bm{H}_A\}$.
	Such information refers to the non-overlapping area of $A$ shown in Fig.~\ref{fig: three-ball-GRA},
	which cannot be recovered unless additional information is provided.

	\vspace{-4pt}
	\section{To Learn Safer, You Must Forget More}
	\label{sec: GRA defense}
	
	Recall in Sec.~\ref{sec: overview},
	graph representations naturally comprise the connectivity information,
	while the graph information plane shows that increasing privacy information 
	is stored in the training phase.
	So, how can GNNs be \textit{GRA-resistant}?
	
	For defense, one must require the GNN to \textit{forget} 
	the privacy information in the training process,
	\textit{i.e.}, make the learned representations contain less information about adjacency.
	Nonetheless,
	it could easily degenerate the accuracy as the adjacency 
	also essentially supports the prediction.
	To solve the trade-off,
	we proposed 
	the \textit{Markov Chain-based Graph Privacy Bottleneck} (MC-GPB) framework
	to defend against GRA (see Fig.~\ref{fig: markov-defense}).
	Intuitively, the expected graph representations
	should come from a refined training process
	that learn the $\bm{\theta}^*$ from the original data $A, X, Y$.
	Inspired by the principle that
	\textit{to learn, you must forget}
	by the information bottleneck~\cite{tishby2000information, shwartz2017opening, wu2020graph}
	that constrains the data compression procedure $X \! \to \! Z \! \to \! Y$,
	we derive the defense objective as
	%
	%
	\vspace{-4pt}
	\begin{equation}	
		\begin{split}
			\text{MC-GPB:}
			\bm{\theta}^* \! = \! \arg \min_{\bm{\theta}} &
			\!
			\sum_{i=1}^{L} 
			\!
			\underbrace{-I(Y \! ; \! \bm{H}^{i}_{A})}_{\text{accuracy}}
			\! + \! \underbrace{\beta_{p}^{i} I(A ; \! \bm{H}^{i}_{A})}_{\text{privacy}} \\
			& \; + \sum_{i=1}^{L-1}
			\underbrace{\beta_{c}^{i} I(\bm{H}^{i}_{A}; \bm{H}^{i+1}_{A})}_{\text{complexity}}.
		\end{split}
		\label{eqn: tighter-defense-MI}
	\end{equation}
	%
	Note that
	MC-GPB is also a maximin game:
	the correlation between hidden representations and labels is maximized, 
	while that with adjacency is minimized instead.
	Analytically,
	it aims to minimize the conditional MI $I(A; \bm{H}^{i}_{A} | Y)$
	through balancing accuracy $I(Y; \bm{H}^{i}_{A})$ and privacy $I(A; \bm{H}^{i}_{A})$.
	And the transformation complexity $I(\bm{H}^{i}_{A}; \bm{H}^{i+1}_{A})$ 
	is constrained to relieve the smoothing effect of message propagation.


	\textbf{Promote forgetting in Eq.\eqref{eqn: tighter-defense-MI} with injected stochasticity.}
	Making GNNs forget more about adjacency leads to lower privacy risk.
	%
	%
	For simplicity,
	the DropEdge~\citep{rong2020dropedge} method is adopted,
	which performs random drop with probability $p$ 
	on each observed edge of $A$.
	The perturbed adjacency 
	$\tilde{A} \! = \! A \oplus A_{\epsilon} \! : A_{\epsilon} \! \perp \!\!\! \perp \!  A, Y, \bm{Z}$,
	which satisfies $I(\tilde{A}; Y) \! \leq \! I(A; Y)$
	and $I(\tilde{A}; \bm{Z}) \! \leq \! I(A; \bm{Z})$
	\citep{you2020graph, you2022bringing}.
	The injected stochasticity
	enforces the GNN model to discriminate the essential topological information $I(A;Y)$,
	rather than fully capturing the association between $A$ and $Y$
	that can be potentially spurious~\citep{zhao2022learning}.
	As such, the redundancy $I(\tilde{A}; \bm{Z} | Y)$ 
	is compressed to preserve privacy and maintain accuracy simultaneously.

	\begin{table*}[t!]
		\centering
		\caption{
			Results of MC-GRA with standard GNNs.
			Relative promotions (in $\%$) are computed \textit{w.r.t.} results in Tab.~\ref{tab:understanding-MI-term-ensemble}.
		}
		\vspace{-8px}
		\label{tab:exp-attack-results}
		\fontsize{8}{10}\selectfont
		\setlength\tabcolsep{7pt}
		\begin{tabular}{cccc|cccccc}
			\toprule
			$X$ & $\bm{H}_A$ &  $\bm{\hat{Y}}_A$  &  $Y$  & Cora & Citeseer &  Polblogs   & USA   &   Brazil &  AIDS   \\
			\midrule
			$\checkmark$ & $\checkmark$ &  &  &  .864 (\up{10.6$\% \! \uparrow$}) & .912 (\up{3.5$\% \! \uparrow$})  & .831 (\up{8.9$\% \! \uparrow$})  & .883 (\up{3.8$\% \! \uparrow$})  & .771 (\up{1.7$\% \! \uparrow$})  &  .574 (\up{10.1$\% \! \uparrow$})   \\ 
			$\checkmark$ &  & $\checkmark$ &   &  .839 (\up{7.4$\% \! \uparrow$}) & .902 (\up{2.3$\% \! \uparrow$})  & .836 (\up{8.2$\% \! \uparrow$})  & .913 (\up{10.5$\% \! \uparrow$})  & .800 (\up{9.2$\% \! \uparrow$})  &  .567 (\up{8.8$\% \! \uparrow$})  \\ 
			$\checkmark$ &  &  & $\checkmark$  &  .896 (\up{5.5$\% \! \uparrow$}) & .918 (\up{1.2$\% \! \uparrow$})  & .837 (\up{18.7$\% \! \uparrow$})  & .825 (\up{13.3$\% \! \uparrow$})  & .753 (\up{22.8$\% \! \uparrow$})  &  .574 (\up{9.9$\% \! \uparrow$})  \\ 
			\midrule
			$\checkmark$ & $\checkmark$ & $\checkmark$ &   &  .866 (\up{10.8$\% \! \uparrow$}) & .921 (\up{4.5$\% \! \uparrow$})  & .839 (\up{9.9$\% \! \uparrow$})  & .878 (\up{3.5$\% \! \uparrow$})  & .776 (\up{2.6$\% \! \uparrow$})  &  .572 (\up{9.7$\% \! \uparrow$})  \\ 
			$\checkmark$ & $\checkmark$ & & $\checkmark$   &  .905 (\up{6.5$\% \! \uparrow$}) & .930 (\up{2.5$\% \! \uparrow$})  & .832 (\up{6.8$\% \! \uparrow$})  & .878 (\up{3.5$\% \! \uparrow$})  & .758 (\up{2.0$\% \! \uparrow$})  &  .603 (\up{15.5$\% \! \uparrow$})  \\ 
			$\checkmark$ &  & $\checkmark$ & $\checkmark$   &  .897 (\up{5.6$\% \! \uparrow$}) & .928 (\up{2.3$\% \! \uparrow$})  & .839 (\up{6.8$\% \! \uparrow$})  & .870 (\up{3.3$\% \! \uparrow$})  & .758 (\up{3.7$\% \! \uparrow$})  &  .567 (\up{8.6$\% \! \uparrow$})  \\ 
			\midrule
			$\checkmark$ & $\checkmark$  & $\checkmark$ & $\checkmark$  &  .904 (\up{6.4$\% \! \uparrow$}) & .931 (\up{2.6$\% \! \uparrow$})  & .853 (\up{9.2$\% \! \uparrow$})  & .870 (\up{1.9$\% \! \uparrow$})  & .760 (\up{5.9$\% \! \uparrow$})  &  .588 (\up{12.6$\% \! \uparrow$})  \\ 
			\bottomrule
		\end{tabular}
	\vspace{-2px}
	\end{table*}

	\begin{table*}[t!]
		\centering
		\caption{
			Results of GRA with MC-GPB protected GNNs.
			Relative reductions are computed \textit{w.r.t.} results in Tab.~\ref{tab:understanding-MI-term-comparison}.
			$I(A; \bm{H}_A), I(A; \bm{\hat{Y}}_A)$ 
			are non-learnable GRA~\citep{he2021stealing}
                while
			$I(A; \bm{H}_{\bm{\hat{A}}}^1)$ 
			is the learnable GRA~\citep{zhang2021graphmi}.
		}
		\vspace{-8px}
		\label{tab:exp-defense-results}
		\fontsize{8}{10}\selectfont
		\setlength\tabcolsep{11pt}
		\begin{tabular}{c|cccccc}
			\toprule
			MI & Cora & Citeseer &  Polblogs   & USA   &   Brazil &  AIDS   \\
			\midrule
			$I(A; \bm{H}_A)$ &  .706 (\down{7.8$\% \! \downarrow$}) & .750 (\down{1.3$\% \! \downarrow$})  & .724 (\down{5.1$\% \! \downarrow$})  & .716 (\down{15.8$\% \! \downarrow$})  & .745 (\down{1.7$\% \! \downarrow$})  &  .564 (\down{3.4$\% \! \downarrow$})  \\ 
			$I(A; \bm{\hat{Y}}_A)$  &   .704 (\down{0.1$\% \! \downarrow$}) & .730 (\down{1.7$\% \! \downarrow$})  & .705 (\down{8.7$\% \! \downarrow$})  & .587 (\down{28.9$\% \! \downarrow$})  & .692 (\down{5.5$\% \! \downarrow$})  &  .559 (\down{0.4$\% \! \downarrow$})  \\ 
			$I(A; \bm{H}_{\bm{\hat{A}}}^1)$  & .625 (\down{9.9$\% \! \downarrow$}) & .691 (\down{9.8$\% \! \downarrow$})  & .506 (\down{26.3$\% \! \downarrow$})  & .300 (\down{64.5$\% \! \downarrow$})  & .609 (\down{25.1$\% \! \downarrow$})  &  .514 (\down{10.6$\% \! \downarrow$})  \\ 
			\midrule
			Acc. &   .734 (\down{3.0$\% \! \downarrow$}) & .602 (\down{4.4$\% \! \downarrow$})  & .830 (\down{1.1$\% \! \downarrow$})  & .391 (\down{16.8$\% \! \downarrow$})  & .808 (\up{5.1$\% \! \uparrow$})  &  .668 (\up{0.0$\% \! \uparrow$})  \\ 
			\bottomrule
		\end{tabular}
        \vspace{-2px}
	\end{table*}

	\begin{table*}[t!]
		\centering
		\caption{
			Results of MC-GRA with MC-GPB protected GNNs.
			Relative reductions are computed \textit{w.r.t.} results in Tab.~\ref{tab:exp-attack-results}.
		}
		\vspace{-8px}
		\label{tab:exp-defense-results-2}
		\fontsize{8}{10}\selectfont
		\setlength\tabcolsep{7pt}
		\begin{tabular}{cccc|cccccc}
			\toprule
			$X$ & $\bm{H}_A$ &  $\bm{\hat{Y}}_A$  &  $Y$  & Cora & Citeseer &  Polblogs   & USA   &   Brazil &  AIDS   \\
			\midrule
			$\checkmark$ & $\checkmark$ &  &  &  .816 (\down{5.5$\% \! \downarrow$}) & .871 (\down{4.4$\% \! \downarrow$})  & .748 (\down{9.9$\% \! \downarrow$})  & .841 (\down{4.7$\% \! \downarrow$})  & .752 (\down{2.4$\% \! \downarrow$})  &  .503 (\down{12.3$\% \! \downarrow$})   \\ 
			$\checkmark$ &  & $\checkmark$ &   &  .817 (\down{9.7$\% \! \downarrow$}) & .843 (\down{6.5$\% \! \downarrow$})  & .707 (\down{15.4$\% \! \downarrow$})  & .844 (\down{7.5$\% \! \downarrow$})  & .747 (\down{6.6$\% \! \downarrow$})  &  .458 (\down{19.2$\% \! \downarrow$})  \\ 
			$\checkmark$ &  &  & $\checkmark$  &  .892 (\down{0.4$\% \! \downarrow$}) & .888 (\down{3.2$\% \! \downarrow$})  & .699 (\down{16.4$\% \! \downarrow$})  & .738 (\down{10.5$\% \! \downarrow$})  & .700 (\down{7.0$\% \! \downarrow$})  &  .490 (\down{14.6$\% \! \downarrow$})  \\ 
			\midrule
			$\checkmark$ & $\checkmark$ & $\checkmark$ &   &  .804 (\down{7.1$\% \! \downarrow$}) & .894 (\down{2.9$\% \! \downarrow$})  & .706 (\down{15.8$\% \! \downarrow$})  & .754 (\down{14.1$\% \! \downarrow$})  & .636 (\down{16.7$\% \! \downarrow$})  &  .546 (\down{3.7$\% \! \downarrow$})  \\ 
			$\checkmark$ & $\checkmark$ & & $\checkmark$   &  .890 (\down{1.6$\% \! \downarrow$}) & .881 (\down{5.2$\% \! \downarrow$})  & .731 (\down{12.1$\% \! \downarrow$})  & .808 (\down{5.6$\% \! \downarrow$})  & .705 (\down{6.9$\% \! \downarrow$})  &  .507 (\down{15.9$\% \! \downarrow$})  \\ 
			$\checkmark$ &  & $\checkmark$ & $\checkmark$   &  .858 (\down{4.3$\% \! \downarrow$}) & .903 (\down{2.6$\% \! \downarrow$})  & .791 (\down{5.7$\% \! \downarrow$})  & .768 (\down{11.7$\% \! \downarrow$})  & .656 (\down{13.4$\% \! \downarrow$})  &  .511 (\down{9.8$\% \! \downarrow$})  \\ 
			\midrule
			$\checkmark$ & $\checkmark$  & $\checkmark$ & $\checkmark$  &  .864 (\down{4.4$\% \! \downarrow$}) & .891 (\down{4.2$\% \! \downarrow$})  & .757 (\down{11.2$\% \! \downarrow$})  & .853 (\down{1.9$\% \! \downarrow$})  & .637 (\down{16.1$\% \! \downarrow$})  &  .547 (\down{6.9$\% \! \downarrow$})  \\ 
			\bottomrule
		\end{tabular}
	
        \vspace{-8px}
	\end{table*}
	
	
	\textbf{Promote feasibility via differentiable measurements.}
	Solving Eq.~\eqref{eqn: MC-GRA} and Eq.~\eqref{eqn: tighter-defense-MI}
	requires tractable objectives.
	%
	Given two variables,
	$X \! \in \! \mathbb{R}^{N \times D_x}$ and $Y \! \in \! \mathbb{R}^{N \times D_y}$,
	we calculate the similarity $s(X, Y)$ to approximate $I(X, Y)$
	considering six differentiable measurements 
        \citep{kornblith2019similarity}.
	Technical details can be found in Appendix.~\ref{ssec: differentiable-MI}.

	%

	\textbf{Theoretical analysis about Eq.~\eqref{eqn: tighter-defense-MI}.}
	Regularizing the graph representations $\bm{H}_{A}$
	with a lower $I(A; \bm{H}_{A})$ indicates a lower $I(A; X,Y,\bm{H}_{A})$,
	and thus, the optimal fidelity $I(A; \bm{\hat{A}}^{*})$ is also decreased
	(refer to Theorem~\ref{theorem: optimal fidelity in GRA}).
	Note that 
	accuracy is prior to privacy in optimization with trade-offs,
	which corresponds to the concept of sufficient statistics.
	\begin{proposition}[Sufficient Statistics]
		\label{prop: sufficient graph representation}
		Denote the sufficient statistics of $X$ as $\bm{Z}$.
		Namely, $\bm{Z}$ is a compression of $X$ as $\bm{Z} \! = \! f(X)$,
		and sufficiency satisfies $I(\bm{Z};Y) \! = \! I(X;Y)$.
	\end{proposition}
	%
	\begin{theorem}[Maximum Adjacency Information]
		\label{theorem: maximum adjacency information}
		The MI between representations $\bm{H}_{A}$
		and adjacency $A$ satisfies that
		$I(A; \bm{H}_{A}) \! \leq \! I(A;A) \! = \! H(A)$.
		\textit{Proof}. See Appendix.~\ref{ssec: proof of maximum adjacency information}.
	\end{theorem}
	\vspace{-6pt}
	As such, Theorem~\ref{theorem: maximum adjacency information} indicates that 
        the graph representations might maintain the maximum information of private $A$,
        as
        $\max I(A; \bm{H}_{A}) \! = \! H(A)$.
        Thus, the only sufficient guarantee is not \textit{safe} enough, and
	the representations $\bm{H}_A$ potentially stores excess adjacency information
	$I(A; \bm{H}_A | Y)$, as illustrated in  Fig.~\ref{fig: three-ball-standard-training}.
	To reduce,
	we refer to the minimal sufficient statistics in Proposition~\ref{prop: minimal graph representation},
	and deduce the lower bound of adjacency information
	in Theorem~\ref{theorem: minimum adjacency information} as follows.
	%
	\begin{proposition}[Minimal Sufficient Statistics]
		\label{prop: minimal graph representation}
		Denote sufficient statistics
		(Proposition~\ref{prop: sufficient graph representation})
		of $X$ as $\bm{Z}$,
		and the minimal sufficient statistics, $\bm{Z}^*$, 
		is the optimal graph representation, namely,
		$\bm{Z}^* = \arg \min_{\bm{Z}: \; I(\bm{Z};Y) = I(X;Y)} I(\bm{Z}; X)$.
	\end{proposition}
	%
	\begin{theorem}[Minimum Adjacency Information]
		\label{theorem: minimum adjacency information}
		For any sufficient graph representations $\bm{H}_{A}$ 
		of adjacency $A$ \textit{w.r.t.} task $Y$,
		its MI with $A$ satisfies that
		$I(A; \bm{H}_{A}) \! \geq \! I(A; Y)$.
		The minimum information $I(A; \bm{H}_{A}) \! = \! I(A; Y)$
		can be achieved iff $I(A; \bm{H}_{A} | Y) \! = \! 0$.
		\textit{Proof}. See Appendix.~\ref{ssec: proof of minimum adjacency information}.
	\end{theorem}
	\vspace{-6pt}
	Then, the Theorem~\ref{theorem: GPB approximate the optimal representation}
	justifies that solving MC-GPB 
	yields an approximation to the optimal representations $\bm{H}_{A}^{*}$,
	as illustrated in Fig.~\ref{fig: three-ball-GPB},
	It satisfies sufficiency (accuracy guarantee) and contains minimal adjacency (privacy guarantee).
	\begin{theorem}
		\label{theorem: GPB approximate the optimal representation}
		When degenerating $\beta_c \! = \! 0$ and $\beta^{i} \! = \! \beta$,
		MC-GPB Eq.~\eqref{eqn: tighter-defense-MI} is equivalent to
		minimizing the Information Bottleneck Lagrangian, \textit{i.e.},
		$\mathcal{L}(p(\bm{Z}|A)) = H(Y|\bm{Z}) + \beta I(\bm{Z}; A)$.
		It yields a sufficient representation $\bm{Z}$ of 
		data $A$ for task $Y$,
		that is an approximation to the optimal representation $\bm{Z}^*$ 
		in Proposition~\ref{prop: minimal graph representation}.
		\textit{Proof}. See Appendix.~\ref{ssec: proof of GPB approximate the optimal representation}.
	\end{theorem}
	

	
	

	
	\section{Empirical Study}
	\label{sec: experiment}
        
	In this section, 
	we empirically verify the two proposed methods
	and provide answers to the three questions.
	\textbf{\textit{Q1:}} how effective are the proposed methods 
	on real-world datasets with common GNNs?
	\textbf{\textit{Q2:}} how helpful are MI constraints and injected stochasticity?
	\textbf{\textit{Q3:}} what insights can empirical results provide
	to GNNs and defending GRA in practice?

	\textbf{Setup.}
	The default target model is a two-layer GCN followed by a linear layer.
	We also investigate other GNN architectures, 
	including GAT~\citep{velickovic2018graph} and GraphSAGE~\citep{hamilton2017inductive}.
	For evaluation,
	we use the AUC metric as in~\citep{zhang2020revisiting, zhu2021neural, zhang2021graphmi}, 
	which considers a set of thresholds.
	Besides,
	the implementation software is Pytorch~\citep{paszke2017automatic}
	while the hardware is an NVIDIA RTX 3090 GPU. 
	%
	Details of the six datasets are referred to Appendix.~\ref{ssec: datasets}.
	
	
	\vspace{-2px}
	\textbf{Baselines.}
	Two recent works are considered as baselines here:
	(1) Stealing link~\citep{he2021stealing}
	that performs non-learnable GRA on the target model's outputs,
	which shares a similar spirit as in Sec.~\ref{ssec: understanding by direct MI calculation}.
	(2) GraphMI~\citep{zhang2021graphmi}
	that conducts learnable GRA with prior knowledge $\mathcal{K} \! = \! \{X,  Y\}$.
	The recovered adjacency is obtained by maximizing the classification probability with regard to labels.

	\subsection{Quantitative Results}
	\vspace{+4px}
        
	\textbf{Attacking.}
	As results shown in Tab.~\ref{tab:exp-attack-results}, 
	the proposed MC-GRA achieves the best results in all six datasets
	with various settings of prior knowledge sets.
	The relative promotions in AUC are gained 
	by comparing with the linear ensemble results in Tab.~\ref{tab:understanding-MI-term-ensemble}.
	As can be seen,
	more prior knowledge with a larger $|\mathcal{K}|$ generally bring a higher attack AUC.
	The learnable MC-GRA brings significantly and consistently 
	better results than the non-learnable methods,
	especially on the more challenging datasets,
	\textit{i.e.}, Brazil and AIDS, where at most \textbf{22.8$\%$} and \textbf{15.5$\%$} promotion can be achieved.
	%
	
	\textbf{Defending.}
	Here, we evaluate the effectiveness of the proposed 
	MC-GPB method in defending against GRA.
	First, in Tab.~\ref{tab:exp-defense-results},
	we show that MC-GPB
	is able to defend all the attack methods of GRA.
	%
	Especially on the Polblogs dataset,
	MC-GPB achieves a \textbf{13.4$\%$} average reduction in privacy leakage
	at the cost of \textbf{1.1$\%$} loss in accuracy.
	Besides, as in Tab.~\ref{tab:exp-defense-results-2}.
	we show that MC-GPB can also defend the MC-GRA,
	where it shows consistent and significant reductions in attack AUC.
	The above results verify the effectiveness of MC-GPB in defending both learnable and non-learnable GRA methods.
        It can potentially protect GNNs 
        applied in real-world applications, \textit{e.g.},
        the recommendation system.
	
	\textbf{Different GNN architectures.}
	%
	As shown in Tab.~\ref{tab:ablation-attack-GNN-arch} as well as Tab.~\ref{tab:ablation-defense-GNN-arch},
	we show that both proposed methods are model-agnostic
	as they can be generalized to different kinds of GNN with various layers.
	Generally, a deeper model (larger $L$) can better protect privacy (lower $I(A; \bm{H}_{A}^{L})$),
	which is consistent with observations in Sec.~\ref{ssec: tracking by graph information plane}.
	However, it might come at the cost of severe accuracy degradation
	due to the well-known over-smoothing effect of GNNs in message propagation.
	Besides,
	it is found that \textit{a more powerful model with a higher accuracy
		is usually more vulnerable to GRA}, which presents a higher risk of privacy leakage in practice.

	\vspace{-2px}
	\subsection{Ablation Study}
	
	\textbf{The MI regularization.}
	As shown in Tab.~\ref{tab:ablation-two-algorithms},
	each MI component contributes to the final results. 
	Specifically,
	the encoding approximation terms in MC-GRA contribute most to the attack.
	A potential reason is that hidden representations $\bm{H}_A$ 
	contain more information about privacy than other variables.
	And thus, extracting this term brings a higher fidelity in outcomes. 
	In addition, all three kinds of constraints contribute greatly to
	MC-GPB, while the contributing patterns are diverse.
	Thus, it is essential to have a careful balance of these three constraints 
	with tuning hyperparameters $\beta^{i}_{p}, \beta^{i}_{c}$.
	
	\textbf{The injected stochasticity.}
	As can be seen from Tab.~\ref{tab:ablation-stochasticity}, 
	learning without injecting stochasticity generally leads to sub-optimal outcomes
	for both methods.
	That is, the manual randomness
	help the removal of spurious correlation for MC-GRA
	and boosts the forgetting about privacy for MC-GPB.
	In addition to MI regularization and injected stochasticity,
	the other ablation study can be found in Appendix.~\ref{sec: full quantitative results}.
	
	

	\begin{table}[t!]
		\centering
		\caption{
			MC-GRA with various architectures on Cora.
		}
		\vspace{-6px}
		\label{tab:ablation-attack-GNN-arch}
		\fontsize{8}{8}\selectfont
		\setlength\tabcolsep{0.5pt}
		\begin{tabular}{c|ccc|ccc|ccc}
			\toprule
			\multirow{2}{*}{$\mathcal{K}$}  & \multicolumn{3}{c|}{GCN}  & \multicolumn{3}{c|}{GAT} & \multicolumn{3}{c}{GraphSAGE} \\
			&  $L \! = \! 2$    &   $L \! = \!  4$   &  $L \! = \!  6$  &  $L \! = \! 2$    &   $L \! = \!  4$   &  $L \! = \!  6$   & $L \! = \! 2$    &   $L \! = \!  4$   &  $L \! = \!  6$  \\   
			\midrule
			$\{X,  Y\}$   &   
			.895 & .892 &  .878 &  .883 &  .878 &  .876 & .889 &  .872 & .840   \\
			$\{X,  Y, \bm{H}_{A}\}$   &   
			.904 & .900 &  .884 &  .897 &  .885 &  .874 & .892 &  .8881 & .873   \\
			$\{X,  Y, \bm{H}_{A}, \bm{\hat{Y}} \}$   &   
			.905 & .895 &  .892 &  .913 &  .887 &  .879 & .909 &  .893 & .865   \\
			\midrule
			Acc.   &   
			.792 & .661 &  .248 &  .637 &  .651 &  .630 & .614 &  .443 & .145   \\
			\bottomrule
		\end{tabular}
		\vspace{-8px}
	\end{table}

	\begin{table}[t!]
		\centering
		\caption{
			MC-GPB with various architectures on Polblogs.
		}
		\vspace{-6px}
		\label{tab:ablation-defense-GNN-arch}
		\fontsize{8}{8}\selectfont
		\setlength\tabcolsep{1.2pt}
		\begin{tabular}{c|ccc|ccc|ccc}
			\toprule
			\multirow{2}{*}{MI}  & \multicolumn{3}{c|}{GCN}  & \multicolumn{3}{c|}{GAT} & \multicolumn{3}{c}{GraphSAGE} \\
			&  $L \! = \! 2$    &   $L \! = \!  4$   &  $L \! = \!  6$  &  $L \! = \! 2$    &   $L \! = \!  4$   &  $L \! = \!  6$   & $L \! = \! 2$    &   $L \! = \!  4$   &  $L \! = \!  6$  \\   
			\midrule
			$I(A; \bm{H}_A)$ &   
			.724 & .790 &  .810  &  .901 & .808  &  .854 & .805 &  .808 &  .813 \\
			$I(A; \bm{\hat{Y}}_A)$  &   
			.705 & .650 &  .650  &  .654 & .623 &  .673 & .803 &  .668 &  .652 \\
			$I(A; \bm{H}_{\bm{\hat{A}}})$   &   
			.506 & .577 &  .532  &  .542 & .656   &  .536 & .599 &  .769 &  .468 \\
			\midrule
			Acc.   &   
			.830 & .822 &  .512  &  .855 & .880 &  .869 & .830 &  .869 &  .801  \\
			\bottomrule
		\end{tabular}
		\vspace{-14px}
	\end{table}
	
	\begin{table}[t!]
		\centering
		\caption{
			Ablation study of two algorithms \textit{w.r.t.} the 
			approximation (\textit{appr.}) and constraint (\textit{cons.}) terms.
		}
		\vspace{-6px}
		\label{tab:ablation-two-algorithms}
		\fontsize{8}{8}\selectfont
		\setlength\tabcolsep{1.8pt}
		\begin{tabular}{c|ccc}
			\toprule
			variant & Cora  & USA & AIDS \\
			\midrule
			MC-GRA (full)                          
			& .905 & .904 & .572  \\
			- w/o encoding appr.
			& .829 (\down{8.3$\% \! \downarrow$})  & .870 (\down{3.7$\% \! \downarrow$})  &  .536 (\down{6.2$\% \! \downarrow$})  \\ 
			- w/o decoding appr.
			& .854 (\down{5.6$\% \! \downarrow$})  & .849 (\down{6.0$\% \! \downarrow$})  &  .490 (\down{14.3$\% \! \downarrow$})  \\ 
			- w/o complexity cons.
			& .889 (\down{1.7$\% \! \downarrow$})  & .858 (\down{5.0$\% \! \downarrow$})  &  .537 (\down{11.3$\% \! \downarrow$})  \\ 
			\midrule
			MC-GPB (full) 
			& .745 & .391 & .668  \\
			- w/o accuracy cons.
			& .681 (\down{8.6$\% \! \downarrow$})  & .369 (\down{5.6$\% \! \downarrow$})  &  .625 (\down{6.4$\% \! \downarrow$})  \\ 
			- w/o privacy cons.
			& .707 (\down{5.1$\% \! \downarrow$})  & .249 (\down{36.3$\% \! \downarrow$})  &  .480 (\down{28.1$\% \! \downarrow$})  \\ 
			- w/o complexity cons.
			& .705 (\down{5.4$\% \! \downarrow$})  & .251 (\down{35.8$\% \! \downarrow$})  &  .448 (\down{32.9$\% \! \downarrow$})  \\ 
			\bottomrule
		\end{tabular}
		\vspace{-6px}
	\end{table}

	\begin{table}[t!]
		\centering
		\caption{
			Results of removing injecting stochasticity. 
		}
		\vspace{-6px}
		\label{tab:ablation-stochasticity}
		\fontsize{8}{8}\selectfont
		\setlength\tabcolsep{0.2pt}
		\begin{tabular}{cc|ccc}
			\toprule
			type & case & USA  & Brazil & AIDS \\
			\midrule
			\multirow{4}{*}{attack} 
			& $\mathcal{K} \! = \! \{X,  Y\}$
			& .802 (\down{2.7$\% \! \downarrow$})  & .713 (\down{5.3$\% \! \downarrow$})  &  .567 (\down{1.2$\% \! \downarrow$})  \\ 
			& $\mathcal{K} \! = \!  \{X,  Y, \bm{H}_{A}\}$ 
			& .856 (\down{1.3$\% \! \downarrow$})  & .740 (\down{2.3$\% \! \downarrow$})  &  .572 (\down{5.1$\% \! \downarrow$})  \\ 
			& $\mathcal{K} \! = \!  \{X,  Y, \bm{H}_{A}, \bm{\hat{Y}} \}$
			& .864 (\down{0.4$\% \! \downarrow$})  & .730 (\down{3.9$\% \! \downarrow$})  &  .567 (\down{3.5$\% \! \downarrow$})  \\ 
			\midrule
			\multirow{4}{*}{defense}
			& $I(A; \bm{H}_A)$
			& .861 (\up{16.2$\% \! \uparrow$})  & .758 (\up{1.7$\% \! \uparrow$})  &  .564 (\up{0.0$\% \! \uparrow$})  \\ 
			& $I(A; \bm{\hat{Y}}_A)$
			& .309 (\down{47.4$\% \! \downarrow$})  & .722 (\up{4.3$\% \! \uparrow$})  &  .548 (\down{2.0$\% \! \downarrow$})  \\ 
			& $I(A; \bm{H}_{\bm{\hat{A}}})$
			& .389 (\up{29.7$\% \! \uparrow$})  & .796 (\up{30.7$\% \! \uparrow$})  &  .539 (\up{4.9$\% \! \uparrow$})  \\ 
			& Acc.
			& .259 (\down{33.8$\% \! \downarrow$})  & .538 (\down{33.4$\% \! \downarrow$})  &  .628 (\down{6.0$\% \! \downarrow$})  \\ 
			\bottomrule
		\end{tabular}
	\vspace{-14px}
	\end{table}

	
	\begin{figure*}[t!]
		\centering
		\hfill
		{\includegraphics[width=8.0cm]{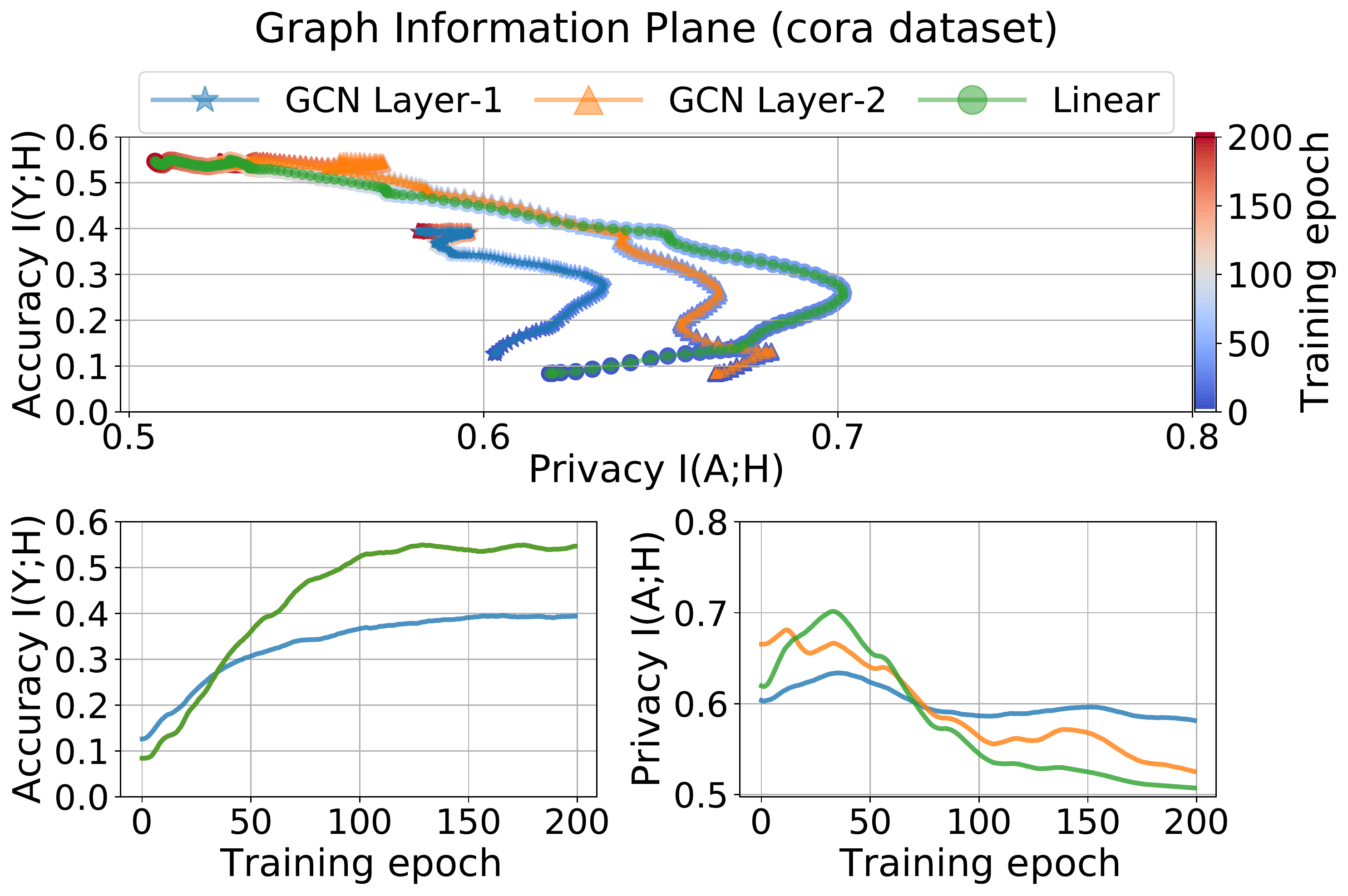}}
		\hfill
		{\includegraphics[width=8.0cm]{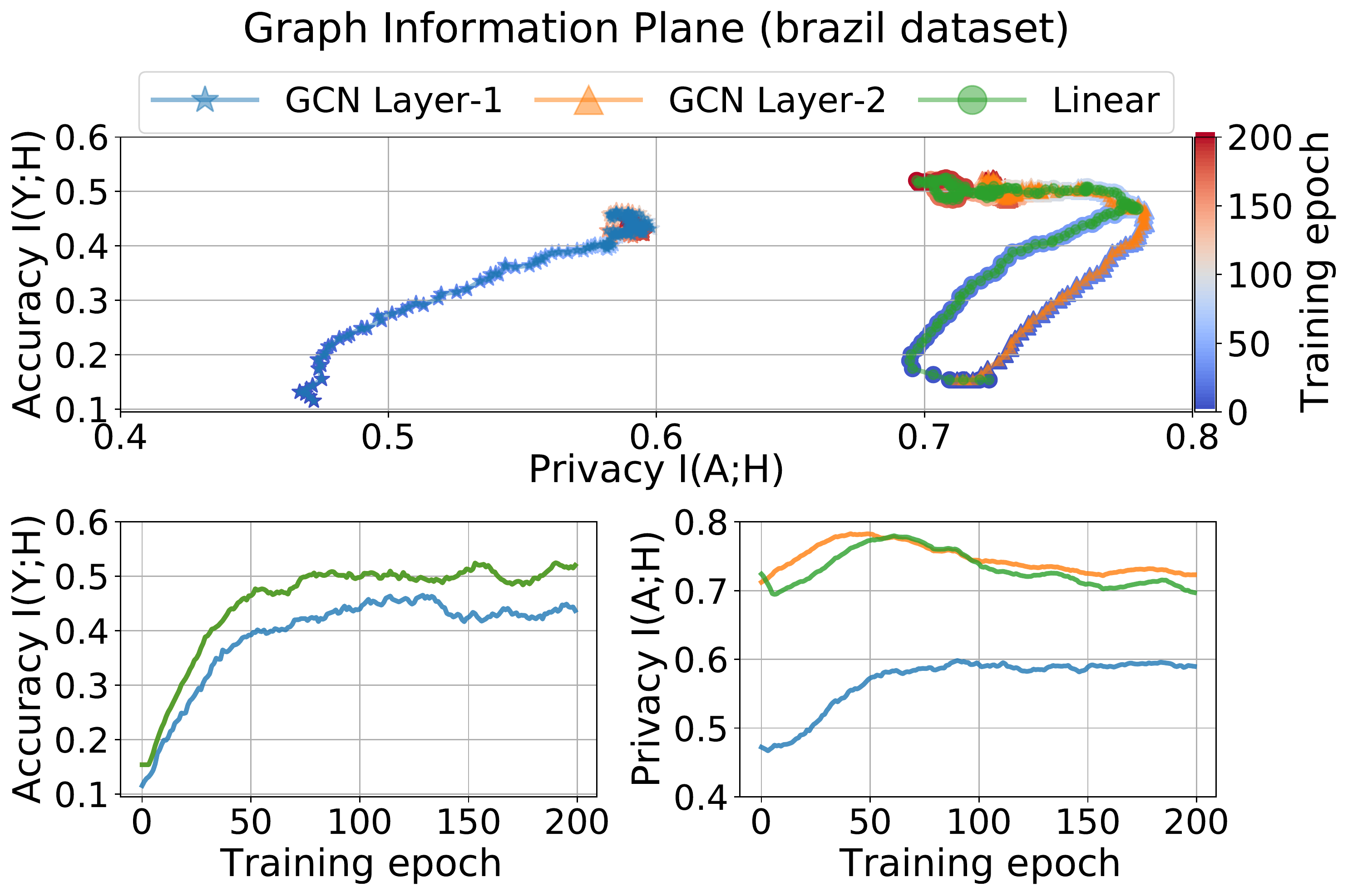}}
		\hfill
		\vspace{-4px}
		\caption{
			Graph information plane: defensive training with MC-GPB.
			Compared with the standard training (Fig.~\ref{fig: graph-information-plane}) without any constraints,
			MC-GPB effectively decreases the amount of privacy information contained in the graph representations.
		}
		\label{fig: demo-GPB-graph-information-plane}
		\vspace{-10px}
	\end{figure*}
	
	\begin{figure}[t!]
		\centering
		\subfigure[GRA on normally trained GNNs.]
		{{\includegraphics[width=8.0cm]{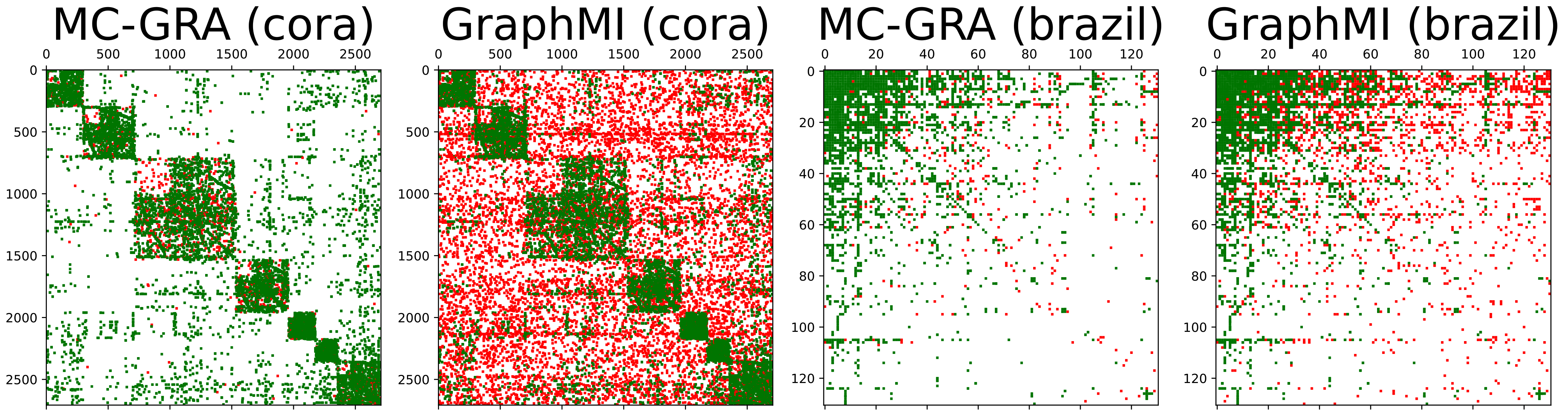}}}
		
		\vspace{-0.3cm}
		
		\subfigure[GRA on protected GNNs, \textit{i.e.}, trained with MC-GPB.]
		{{\includegraphics[width=8.0cm]{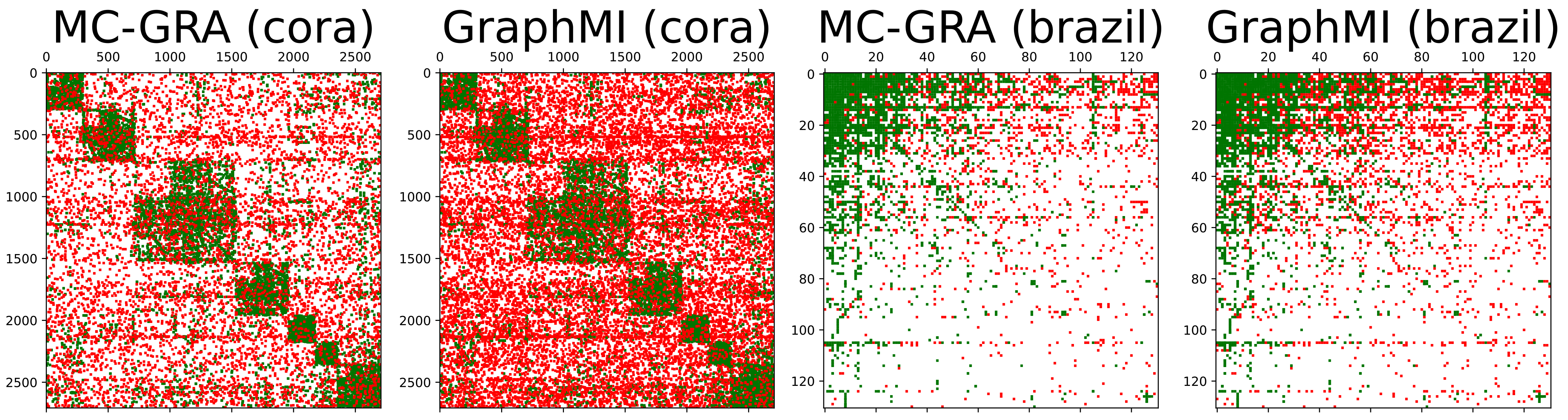}}}
		
		\vspace{-8px}
		\caption{	
			Examples of recovered adjacency.
			Green dots are correctly predicted edges while red dots are wrong ones.
		}
		\label{fig: demo-4-Adj}
		\vspace{-14px}
	\end{figure}

        \vspace{-2px}
	\subsection{Case Visualizations}
	
	\textbf{The recovered adjacency.}
	We show the recovered $\bm{\hat{A}}$ by various GRA methods in Fig.~\ref{fig: demo-4-Adj}.
	Compared with GraphMI, MC-GRA can recover adjacency more accurately,
	with fewer wrong predictions and higher AUC values.
	As for defense,
	MC-GPB significantly degenerates both GRA methods,
	with more failure cases and much lower AUC values.

	\textbf{A further analysis with the graph information plane.}
	Besides,
	we visualize the training procedures of MC-GPB in Fig.~\ref{fig: demo-GPB-graph-information-plane}
	based on the graph information plane introduced in Sec.~\ref{ssec: tracking by graph information plane}.
	As shown, the privacy term is markedly reduced while that in standard training is increased,
	especially in the later training stage.
	It shows the trade-off between accuracy and privacy in training GNNs:
	the accuracy $I(Y; \bm{Z})$ also starts to decrease 
	when the privacy $I(A; \bm{Z})$ is minimized to some extent.
        More visualizations are in Appendix.~\ref{sec: full qualitative results}.
 

	
	
	
	%
	%
	%

	\vspace{-8px}
	\section{Further Discussions}
	\label{sec: limitation}
        
	\textbf{GRA in practice.}
	In real-world examples regarding the threat model, 
	the prior knowledge set $\mathcal{K}$ 
	can be accessed by an adversary in practice. 
	For instance, to train a GNN model for fraudulent account detection, 
	a social network service provider uses 
	the technology of another company. 
	In this case, the provider will frequently send 
	the company the model's outputs $\bm{Y}_A$
	to debug and improve. 
	Similar circumstances apply to 
	the node embeddings $\bm{H}_A$,
	which are typically released. 
	Thus, the inversion of adjacency
        requiring only a subset $\mathcal{K}$ of the informative variables
	can be a privacy threat in real-world scenarios of GNNs, 
	which have been widely 
	used in recommendation systems, social networks, 
        citation networks, and drug discovery. 
	Therefore, the user's privacy should be protected
	especially for personalized relationships and certain sensitive information.
	
	\textbf{The inversion target.}
        Intuitively, the adjacency $A$ and the node features $X$ can be regarded as inversion targets. 
        The key motivations to attack adjacency are the practical risks and understandability to human beings. 
        Unlike visual images that are naturally understandable to humans, the node features are not understandable without the sufficient knowledge of human experts, while the adjacency is much easier to understand. 
        More discussions are in Appendix.~\ref{sec: further discussion on GRA}.
 
	\textbf{Limitations.}
	This work follows the common homophily assumption that 
	connected nodes are likely to be in the same category 
	and possess similar features~\cite{he2021stealing,zhang2021graphmi}.
	We leave the generalization to heterogeneous graphs as future work. 
	Besides, 
	our proposed method requires white-box access to the target model.
	The black-box scenarios with only access to the model's outputs
	can be more practical but also much more challenging.

	\textbf{Future directions.}
	One general direction to enhance GRA
	is to extract information about adjacency from more information sources,
	\textit{e.g.}, with partial edges of the target graph
	or an auxiliary dataset to conduct a transferring attack.
	GRA can be conducted on more GNN architectures
	and cooperated with generative models,
	\textit{e.g.}, the graph auto-encoders or diffusion models.
	Besides, a more fine-grained study
	on graph properties is also intriguing, 
	\textit{e.g.}, density, community, number of triangles.
	To what extent can the above properties be recovered
	will shed insights into the power of GRA 
	and the memorization effect of GNNs.
	Besides,
	applying GRA to more realistic and general settings is also promising,
	\textit{e.g.}, inductive GNNs which can generalize well to unseen nodes,
	or the black-box scenarios
	where the attacker can only get access to the outputs of the target model.
	As for defending against GRA in practice,
	a trained model might be required to completely forget about partial training data
	with limited budgets for updating their weights.
	A promising solution here is machine unlearning, 
	especially on large-scale graphs.

	\textbf{Broader impacts.}
	\textit{The gun is not guilty, the person who pulled the trigger is},
	said the father of the AK-47.
	It must be admitted that
	GRA (or any kind of MIA)
	might be misused to attack real-world targets.
	For this reason, it is essential to  
	raise the awareness of such an adversary and the potential privacy risk.
	More importantly, investigating GRA (MIA) enables 
	to understand the black-box deep learning models,
	to inspire more robust methods,
	to protect privacy in advance, 
	and to make AI products safer and more trustworthy.

	\section{Conclusion}
	
	In this work, 
	we conduct a comprehensive study of
	enhancing and defending Graph Reconstruction Attack.
	We conceptually abstract the problem as 
	approximating the original Markov chain by the attack chain.
	Technically, we derive 
	(1) the chain-based attack method with adaptive designs for extracting more private information;
	and
	(2) the chain-based defense method that sharply reduces the attack fidelity with moderate accuracy loss. 
	%
	Empirically,
	the proposed methods achieve the best results on six datasets and three GNNs.
	
        
	\vspace{-4pt}
        \section*{Acknowledgements}
	ZKZ, CYZ, XL, and BH were supported 
	by NSFC Young Scientists Fund No. 62006202 
	and Guangdong Basic 
	and Applied Basic Research Foundation No. 2022A1515011652, 
	CAAI-Huawei MindSpore Open Fund, 
	and HKBU CSD Departmental Incentive Grant.
        JCY was supported by the National Key R\&D Program of China (No. 2022ZD0160703), STCSM (No. 22511106101, No. 22511105700, No. 21DZ1100100), 111 plan (No. BP0719010).
        
	\clearpage
	
	
	\bibliography{draft}

\begin{thebibliography}{64}
\providecommand{\natexlab}[1]{#1}
\providecommand{\url}[1]{\texttt{#1}}
\expandafter\ifx\csname urlstyle\endcsname\relax
  \providecommand{\doi}[1]{doi: #1}\else
  \providecommand{\doi}{doi: \begingroup \urlstyle{rm}\Url}\fi

\bibitem[Abadi et~al.(2016)Abadi, Chu, Goodfellow, McMahan, Mironov, Talwar,
  and Zhang]{abadi2016deep}
Abadi, M., Chu, A., Goodfellow, I., McMahan, H.~B., Mironov, I., Talwar, K.,
  and Zhang, L.
\newblock Deep learning with differential privacy.
\newblock In \emph{CCS}, 2016.

\bibitem[Adamic \& Glance(2005)Adamic and Glance]{adamic2005political}
Adamic, L.~A. and Glance, N.
\newblock The political blogosphere and the 2004 us election: divided they
  blog.
\newblock In \emph{Proceedings of the 3rd international workshop on Link
  discovery}, 2005.

\bibitem[Bietti et~al.(2022)Bietti, Wei, Dudik, Langford, and
  Wu]{bietti2022personalization}
Bietti, A., Wei, C.-Y., Dudik, M., Langford, J., and Wu, S.
\newblock Personalization improves privacy-accuracy tradeoffs in federated
  learning.
\newblock In \emph{ICML}, 2022.

\bibitem[Carlini et~al.(2021)Carlini, Tramer, Wallace, Jagielski, Herbert-Voss,
  Lee, Roberts, Brown, Song, Erlingsson, et~al.]{carlini2021extracting}
Carlini, N., Tramer, F., Wallace, E., Jagielski, M., Herbert-Voss, A., Lee, K.,
  Roberts, A., Brown, T., Song, D., Erlingsson, U., et~al.
\newblock Extracting training data from large language models.
\newblock In \emph{USENIX Security}, 2021.

\bibitem[Chanpuriya et~al.(2021)Chanpuriya, Musco, Sotiropoulos, and
  Tsourakakis]{chanpuriya2021deepwalking}
Chanpuriya, S., Musco, C., Sotiropoulos, K., and Tsourakakis, C.
\newblock Deepwalking backwards: from embeddings back to graphs.
\newblock In \emph{ICML}, 2021.

\bibitem[Chen et~al.(2021)Chen, Kahla, Jia, and Qi]{chen2021knowledge}
Chen, S., Kahla, M., Jia, R., and Qi, G.-J.
\newblock Knowledge-enriched distributional model inversion attacks.
\newblock In \emph{ICCV}, 2021.

\bibitem[Chen et~al.(2022)Chen, Yang, Zhang, Ma, Liu, Han, and
  Cheng]{chen2022understanding}
Chen, Y., Yang, H., Zhang, Y., Ma, K., Liu, T., Han, B., and Cheng, J.
\newblock Understanding and improving graph injection attack by promoting
  unnoticeability.
\newblock In \emph{ICLR}, 2022.

\bibitem[Cortes et~al.(2012)Cortes, Mohri, and
  Rostamizadeh]{cortes2012algorithms}
Cortes, C., Mohri, M., and Rostamizadeh, A.
\newblock Algorithms for learning kernels based on centered alignment.
\newblock \emph{The Journal of Machine Learning Research}, 2012.

\bibitem[Cover(1999)]{cover1999elements}
Cover, T.~M.
\newblock \emph{Elements of information theory}.
\newblock John Wiley \& Sons, 1999.

\bibitem[Dai et~al.(2018)Dai, Li, Tian, Huang, Wang, Zhu, and
  Song]{dai2018adversarial}
Dai, H., Li, H., Tian, T., Huang, X., Wang, L., Zhu, J., and Song, L.
\newblock Adversarial attack on graph structured data.
\newblock In \emph{ICML}, 2018.

\bibitem[Duddu et~al.(2020)Duddu, Boutet, and Shejwalkar]{duddu2020quantifying}
Duddu, V., Boutet, A., and Shejwalkar, V.
\newblock Quantifying privacy leakage in graph embedding.
\newblock In \emph{MobiQuitous}, 2020.

\bibitem[Fan et~al.(2019)Fan, Ma, Li, He, Zhao, Tang, and Yin]{fan2019graph}
Fan, W., Ma, Y., Li, Q., He, Y., Zhao, E., Tang, J., and Yin, D.
\newblock Graph neural networks for social recommendation.
\newblock In \emph{TheWebConf}, 2019.

\bibitem[Fang et~al.(2021)Fang, Song, Wang, Shen, Wang, and
  Song]{fang2021contrastive}
Fang, G., Song, J., Wang, X., Shen, C., Wang, X., and Song, M.
\newblock Contrastive model inversion for data-free knowledge distillation.
\newblock In \emph{IJCAI}, 2021.

\bibitem[Fano(1961)]{fano1961transmission}
Fano, R.~M.
\newblock Transmission of information: A statistical theory of communications.
\newblock \emph{American Journal of Physics}, 1961.

\bibitem[Fredrikson et~al.(2014)Fredrikson, Lantz, Jha, Lin, Page, and
  Ristenpart]{fredrikson2014privacy}
Fredrikson, M., Lantz, E., Jha, S., Lin, S., Page, D., and Ristenpart, T.
\newblock Privacy in pharmacogenetics: An end-to-end case study of personalized
  warfarin dosing.
\newblock In \emph{USENIX Security}, 2014.

\bibitem[Fredrikson et~al.(2015)Fredrikson, Jha, and
  Ristenpart]{fredrikson2015model}
Fredrikson, M., Jha, S., and Ristenpart, T.
\newblock Model inversion attacks that exploit confidence information and basic
  countermeasures.
\newblock In \emph{CCS}, 2015.

\bibitem[Gilmer et~al.(2017)Gilmer, Schoenholz, Riley, Vinyals, and
  Dahl]{gilmer2017neural}
Gilmer, J., Schoenholz, S.~S., Riley, P.~F., Vinyals, O., and Dahl, G.~E.
\newblock Neural message passing for quantum chemistry.
\newblock In \emph{ICML}, 2017.

\bibitem[Grandvalet \& Bengio(2004)Grandvalet and Bengio]{grandvalet2004semi}
Grandvalet, Y. and Bengio, Y.
\newblock Semi-supervised learning by entropy minimization.
\newblock In \emph{NIPS}, 2004.

\bibitem[Gretton et~al.(2005)Gretton, Bousquet, Smola, and
  Sch{\"o}lkopf]{gretton2005measuring}
Gretton, A., Bousquet, O., Smola, A., and Sch{\"o}lkopf, B.
\newblock Measuring statistical dependence with hilbert-schmidt norms.
\newblock In \emph{ALT}, 2005.

\bibitem[Hamilton et~al.(2017)Hamilton, Ying, and
  Leskovec]{hamilton2017inductive}
Hamilton, W., Ying, Z., and Leskovec, J.
\newblock Inductive representation learning on large graphs.
\newblock In \emph{NeurIPS}, 2017.

\bibitem[He et~al.(2021{\natexlab{a}})He, Jia, Backes, Gong, and
  Zhang]{he2021stealing}
He, X., Jia, J., Backes, M., Gong, N.~Z., and Zhang, Y.
\newblock Stealing links from graph neural networks.
\newblock In \emph{USENIX Security}, 2021{\natexlab{a}}.

\bibitem[He et~al.(2021{\natexlab{b}})He, Wen, Wu, Backes, Shen, and
  Zhang]{he2021node}
He, X., Wen, R., Wu, Y., Backes, M., Shen, Y., and Zhang, Y.
\newblock Node-level membership inference attacks against graph neural
  networks.
\newblock \emph{arXiv preprint arXiv:2102.05429}, 2021{\natexlab{b}}.

\bibitem[Hidano et~al.(2017)Hidano, Murakami, Katsumata, Kiyomoto, and
  Hanaoka]{hidano2017model}
Hidano, S., Murakami, T., Katsumata, S., Kiyomoto, S., and Hanaoka, G.
\newblock Model inversion attacks for prediction systems: Without knowledge of
  non-sensitive attributes.
\newblock In \emph{PST}, 2017.

\bibitem[Ioannidis et~al.(2020)Ioannidis, Zheng, and Karypis]{ioannidis2020few}
Ioannidis, V.~N., Zheng, D., and Karypis, G.
\newblock Few-shot link prediction via graph neural networks for covid-19
  drug-repurposing.
\newblock \emph{arXiv preprint arXiv:2007.10261}, 2020.

\bibitem[Kahla et~al.(2022)Kahla, Chen, Just, and Jia]{kahla2022label}
Kahla, M., Chen, S., Just, H.~A., and Jia, R.
\newblock Label-only model inversion attacks via boundary repulsion.
\newblock In \emph{CVPR}, 2022.

\bibitem[Kipf \& Welling(2016{\natexlab{a}})Kipf and Welling]{kipf2016semi}
Kipf, T.~N. and Welling, M.
\newblock Semi-supervised classification with graph convolutional networks.
\newblock In \emph{ICLR}, 2016{\natexlab{a}}.

\bibitem[Kipf \& Welling(2016{\natexlab{b}})Kipf and
  Welling]{kipf2016variational}
Kipf, T.~N. and Welling, M.
\newblock Variational graph auto-encoders.
\newblock \emph{NIPS Workshop on Bayesian Deep Learning}, 2016{\natexlab{b}}.

\bibitem[Kool et~al.(2019)Kool, Van~Hoof, and Welling]{kool2019stochastic}
Kool, W., Van~Hoof, H., and Welling, M.
\newblock Stochastic beams and where to find them: The gumbel-top-k trick for
  sampling sequences without replacement.
\newblock In \emph{ICML}, 2019.

\bibitem[Kornblith et~al.(2019)Kornblith, Norouzi, Lee, and
  Hinton]{kornblith2019similarity}
Kornblith, S., Norouzi, M., Lee, H., and Hinton, G.
\newblock Similarity of neural network representations revisited.
\newblock In \emph{ICML}, 2019.

\bibitem[LaValle et~al.(2004)LaValle, Branicky, and
  Lindemann]{lavalle2004relationship}
LaValle, S.~M., Branicky, M.~S., and Lindemann, S.~R.
\newblock On the relationship between classical grid search and probabilistic
  roadmaps.
\newblock \emph{The International Journal of Robotics Research}, 2004.

\bibitem[Miao et~al.(2022)Miao, Liu, and Li]{miao2022interpretable}
Miao, S., Liu, M., and Li, P.
\newblock Interpretable and generalizable graph learning via stochastic
  attention mechanism.
\newblock In \emph{ICML}, 2022.

\bibitem[Paszke et~al.(2017)Paszke, Gross, Chintala, Chanan, Yang, DeVito, Lin,
  Desmaison, Antiga, and Lerer]{paszke2017automatic}
Paszke, A., Gross, S., Chintala, S., Chanan, G., Yang, E., DeVito, Z., Lin, Z.,
  Desmaison, A., Antiga, L., and Lerer, A.
\newblock Automatic differentiation in {PyTorch}.
\newblock In \emph{ICLR}, 2017.

\bibitem[Peng et~al.(2022)Peng, Liu, Zhang, Lan, Ye, Liu, and
  Han]{peng2022bilateral}
Peng, X., Liu, F., Zhang, J., Lan, L., Ye, J., Liu, T., and Han, B.
\newblock Bilateral dependency optimization: Defending against model-inversion
  attacks.
\newblock In \emph{SIGKDD}, 2022.

\bibitem[Ribeiro et~al.(2017)Ribeiro, Saverese, and
  Figueiredo]{ribeiro2017struc2vec}
Ribeiro, L.~F., Saverese, P.~H., and Figueiredo, D.~R.
\newblock struc2vec: Learning node representations from structural identity.
\newblock In \emph{SIGKDD}, 2017.

\bibitem[Riesen et~al.(2008)Riesen, Bunke, et~al.]{riesen2008iam}
Riesen, K., Bunke, H., et~al.
\newblock Iam graph database repository for graph based pattern recognition and
  machine learning.
\newblock In \emph{SSPR/SPR}, 2008.

\bibitem[Rong et~al.(2020)Rong, Huang, Xu, and Huang]{rong2020dropedge}
Rong, Y., Huang, W., Xu, T., and Huang, J.
\newblock Dropedge: Towards deep graph convolutional networks on node
  classification.
\newblock In \emph{ICLR}, 2020.

\bibitem[Sen et~al.(2008)Sen, Namata, Bilgic, Getoor, Galligher, and
  Eliassi-Rad]{sen2008collective}
Sen, P., Namata, G., Bilgic, M., Getoor, L., Galligher, B., and Eliassi-Rad, T.
\newblock Collective classification in network data.
\newblock \emph{AI magazine}, 2008.

\bibitem[Shamir et~al.(2010)Shamir, Sabato, and Tishby]{shamir2010learning}
Shamir, O., Sabato, S., and Tishby, N.
\newblock Learning and generalization with the information bottleneck.
\newblock \emph{Theoretical Computer Science}, 2010.

\bibitem[Shen et~al.(2022)Shen, He, Han, and Zhang]{shen2022model}
Shen, Y., He, X., Han, Y., and Zhang, Y.
\newblock Model stealing attacks against inductive graph neural networks.
\newblock In \emph{IEEE SP}, 2022.

\bibitem[Shwartz-Ziv \& Tishby(2017)Shwartz-Ziv and Tishby]{shwartz2017opening}
Shwartz-Ziv, R. and Tishby, N.
\newblock Opening the black box of deep neural networks via information.
\newblock \emph{arXiv preprint arXiv:1703.00810}, 2017.

\bibitem[Struppek et~al.(2022)Struppek, Hintersdorf, Correia, Adler, and
  Kersting]{struppek2022ppa}
Struppek, L., Hintersdorf, D., Correia, A. D.~A., Adler, A., and Kersting, K.
\newblock Plug and play attacks: Towards robust and flexible model inversion
  attacks.
\newblock In \emph{ICML}, 2022.

\bibitem[Szegedy et~al.(2013)Szegedy, Zaremba, Sutskever, Bruna, Erhan,
  Goodfellow, and Fergus]{szegedy2013intriguing}
Szegedy, C., Zaremba, W., Sutskever, I., Bruna, J., Erhan, D., Goodfellow, I.,
  and Fergus, R.
\newblock Intriguing properties of neural networks.
\newblock \emph{arXiv preprint arXiv:1312.6199}, 2013.

\bibitem[Tishby \& Zaslavsky(2015)Tishby and Zaslavsky]{tishby2015deep}
Tishby, N. and Zaslavsky, N.
\newblock Deep learning and the information bottleneck principle.
\newblock In \emph{IEEE information theory workshop}, 2015.

\bibitem[Tishby et~al.(2000)Tishby, Pereira, and Bialek]{tishby2000information}
Tishby, N., Pereira, F.~C., and Bialek, W.
\newblock The information bottleneck method.
\newblock \emph{arXiv preprint physics/0004057}, 2000.

\bibitem[Veli{\v{c}}kovi{\'{c}} et~al.(2018)Veli{\v{c}}kovi{\'{c}}, Cucurull,
  Casanova, Romero, Li{\`{o}}, and Bengio]{velickovic2018graph}
Veli{\v{c}}kovi{\'{c}}, P., Cucurull, G., Casanova, A., Romero, A., Li{\`{o}},
  P., and Bengio, Y.
\newblock Graph attention networks.
\newblock In \emph{ICLR}, 2018.

\bibitem[Wang et~al.(2021)Wang, Fu, Li, Khisti, Zemel, and
  Makhzani]{wang2021variational}
Wang, K.-C., Fu, Y., Li, K., Khisti, A., Zemel, R., and Makhzani, A.
\newblock Variational model inversion attacks.
\newblock In \emph{NeurIPS}, 2021.

\bibitem[Wu et~al.(2020{\natexlab{a}})Wu, Sun, Zhang, Xie, and
  Cui]{wu2020graphRecsys}
Wu, S., Sun, F., Zhang, W., Xie, X., and Cui, B.
\newblock Graph neural networks in recommender systems: a survey.
\newblock \emph{ACM Computing Surveys}, 2020{\natexlab{a}}.

\bibitem[Wu et~al.(2020{\natexlab{b}})Wu, Ren, Li, and Leskovec]{wu2020graph}
Wu, T., Ren, H., Li, P., and Leskovec, J.
\newblock Graph information bottleneck.
\newblock In \emph{NeurIPS}, 2020{\natexlab{b}}.

\bibitem[Xie \& Ermon(2019)Xie and Ermon]{xie2019reparameterizable}
Xie, S. and Ermon, S.
\newblock Reparameterizable subset sampling via continuous relaxations.
\newblock In \emph{IJCAI}, 2019.

\bibitem[Yang et~al.(2022)Yang, Chou, and Chaudhuri]{yang2022understanding}
Yang, Y.-Y., Chou, C.-N., and Chaudhuri, K.
\newblock Understanding rare spurious correlations in neural network.
\newblock \emph{arXiv preprint arXiv:2202.05189}, 2022.

\bibitem[Yang et~al.(2019)Yang, Chang, and Liang]{yang2019adversarial}
Yang, Z., Chang, E.-C., and Liang, Z.
\newblock Adversarial neural network inversion via auxiliary knowledge
  alignment.
\newblock \emph{arXiv preprint arXiv:1902.08552}, 2019.

\bibitem[You et~al.(2020)You, Chen, Sui, Chen, Wang, and Shen]{you2020graph}
You, Y., Chen, T., Sui, Y., Chen, T., Wang, Z., and Shen, Y.
\newblock Graph contrastive learning with augmentations.
\newblock In \emph{NeurIPS}, 2020.

\bibitem[You et~al.(2022)You, Chen, Wang, and Shen]{you2022bringing}
You, Y., Chen, T., Wang, Z., and Shen, Y.
\newblock Bringing your own view: Graph contrastive learning without
  prefabricated data augmentations.
\newblock In \emph{WSDM}, 2022.

\bibitem[Zeng et~al.(2019)Zeng, Zhou, Srivastava, Kannan, and
  Prasanna]{zeng2019graphsaint}
Zeng, H., Zhou, H., Srivastava, A., Kannan, R., and Prasanna, V.
\newblock Graphsaint: Graph sampling based inductive learning method.
\newblock In \emph{ICLR}, 2019.

\bibitem[Zhang \& Chen(2018)Zhang and Chen]{zhang2018link}
Zhang, M. and Chen, Y.
\newblock Link prediction based on graph neural networks.
\newblock In \emph{NeurIPS}, 2018.

\bibitem[Zhang et~al.(2021{\natexlab{a}})Zhang, Li, Xia, Wang, and
  Jin]{zhang2020revisiting}
Zhang, M., Li, P., Xia, Y., Wang, K., and Jin, L.
\newblock Labeling trick: A theory of using graph neural networks for
  multi-node representation learning.
\newblock In \emph{NeurIPS}, 2021{\natexlab{a}}.

\bibitem[Zhang et~al.(2022{\natexlab{a}})Zhang, Hidano, and
  Koushanfar]{zhang2022text}
Zhang, R., Hidano, S., and Koushanfar, F.
\newblock Text revealer: Private text reconstruction via model inversion
  attacks against transformers.
\newblock \emph{arXiv preprint arXiv:2209.10505}, 2022{\natexlab{a}}.

\bibitem[Zhang et~al.(2020)Zhang, Jia, Pei, Wang, Li, and
  Song]{zhang2020secret}
Zhang, Y., Jia, R., Pei, H., Wang, W., Li, B., and Song, D.
\newblock The secret revealer: Generative model-inversion attacks against deep
  neural networks.
\newblock In \emph{CVPR}, 2020.

\bibitem[Zhang et~al.(2021{\natexlab{b}})Zhang, Liu, Huang, Wang, Lu, Liu, and
  Chen]{zhang2021graphmi}
Zhang, Z., Liu, Q., Huang, Z., Wang, H., Lu, C., Liu, C., and Chen, E.
\newblock Graphmi: Extracting private graph data from graph neural networks.
\newblock In \emph{IJCAI}, 2021{\natexlab{b}}.

\bibitem[Zhang et~al.(2022{\natexlab{b}})Zhang, Chen, Backes, Shen, and
  Zhang]{zhang2022inference}
Zhang, Z., Chen, M., Backes, M., Shen, Y., and Zhang, Y.
\newblock Inference attacks against graph neural networks.
\newblock In \emph{USENIX Security}, 2022{\natexlab{b}}.

\bibitem[Zhao et~al.(2022)Zhao, Liu, Wang, Yu, and Jiang]{zhao2022learning}
Zhao, T., Liu, G., Wang, D., Yu, W., and Jiang, M.
\newblock Learning from counterfactual links for link prediction.
\newblock In \emph{ICML}, 2022.

\bibitem[Zhao et~al.(2021)Zhao, Zhang, Xiao, and Lim]{zhao2021exploiting}
Zhao, X., Zhang, W., Xiao, X., and Lim, B.
\newblock Exploiting explanations for model inversion attacks.
\newblock In \emph{ICCV}, 2021.

\bibitem[Zhu et~al.(2023)Zhu, Yao, Liu, Yao, Xu, and Han]{zhu2023combating}
Zhu, J., Yao, J., Liu, T., Yao, Q., Xu, J., and Han, B.
\newblock Combating exacerbated heterogeneity for robust models in federated
  learning.
\newblock In \emph{ICLR}, 2023.

\bibitem[Zhu et~al.(2021)Zhu, Zhang, Xhonneux, and Tang]{zhu2021neural}
Zhu, Z., Zhang, Z., Xhonneux, L., and Tang, J.
\newblock Neural bellman-ford networks: A general graph neural network
  framework for link prediction.
\newblock In \emph{NeurIPS}, 2021.

\end{thebibliography}
	\bibliographystyle{icml2023}

	\clearpage
	\onecolumn
	\appendix
	
	\etocdepthtag.toc{mtappendix}
	\etocsettagdepth{mtchapter}{none}
	\etocsettagdepth{mtappendix}{subsection}
	
	\renewcommand{\contentsname}{Appendix}
	\tableofcontents
	
	\clearpage
	
	
	\section{Theoretical justification}
	
	\subsection{Notations}
	\label{ssec: notations}
	
	With adjacent matrix $A$ and node features $X$,
	an undirected graph is denoted as $\mathcal{G} \! = \! (A, X)$,
	where $A_{ij} \! = \! 1$ means there is an edge $e_{ij}$ between $v_i$ and $v_j$.
	For each node $v_i$,
	its $D$-dimension node feature is denoted as $X_{[i,:]} \! \in \! \mathbb{R}^{D}$,
	and its label $y_i \in Y = \{y_i\}_{i=1}^{N}$ indicates the node class.
	The node classification task is to predict the label $Y=\{y_i\}_{i=1}^{N}$ of each node
	via a parameterized model $f_{\theta}(\cdot)$,
	\textit{i.e.}, $f_{\theta}(A, X) = \hat{Y} \leftrightarrow Y$.
	We summarize the frequently used notations in Table~\ref{tab:notations} as follows.
	
	\begin{table}[ht]
		\centering
		\caption{The most frequently used notations in this work.}
		\label{tab:notations}
		\begin{tabular}{c|c}
			\toprule
			notations  & meanings \\
			\midrule
			$\mathcal{V} = \{v_i\}_{i=1}^{N}$ & the set of nodes \\ \midrule
			$\mathcal{E} \! = \! \{e_{ij}\}_{ij=1}^{M}$ & the set of edges \\ \midrule
			$A\in\{0,1\}^{N \times N}$ & the adjacent matrix with binary elements \\ \midrule
			$X\in\mathbb{R}^{N \times D}$ & the node features \\ \midrule
			$\mathcal{G} \! = \! (A, X)$ & the input graph of a GNN \\ \midrule
			$Y$ & the labels of nodes \\ \midrule
			$\bm{H}_{A}$ & representation of all nodes with adjacency $A$ \\ \midrule
			$H(X)$ & the information entropy of random variable $X$ \\ \midrule
			$H(X, Y)$ & the joint entropy of variable $X$ and $Y$ \\ \midrule
			$I(X ; Y)$ & the mutual information of $X$ and $Y$ \\ \midrule
			$I(X ; Y | Z)$ & the conditional mutual information of $X$ and $Y$ when observing $Z$ \\		
			\bottomrule
		\end{tabular}
	\end{table}
	
	\subsection{Preliminaries for information measures}
	\label{ssec: information measures}
	
	\begin{definition}[Informational Divergence]
		The informational divergence 
		(also called relative entropy or Kullback-Leibler distance)
		between two probability distributions $p$ and $q$ 
		on a finite space $\mathcal{X}$ (\textit{i.e.}, a common alphabet) is defined as
		\begin{align}
			\begin{split}
				D(p || q)
				= \sum_{x \in\mathcal{X}} p(x) \log \frac{p(x)}{q(x)} = \mathbb{E}_{p}\big[ \log \frac{p(X)}{q(X)} \big]
			\end{split}
		\end{align}
	\end{definition}
	
	\begin{remark}
		$D(p || q)$ measures the \textit{distance} between $p$ and $q$.
		However, $D(\cdot || \cdot)$ is not a true metric,
		and it does not satisfy the triangular inequality.
		$D(p || q)$ is non-negative and $D(p || q)=0$ if and only if $p = q$.
	\end{remark}

	\begin{definition}[Mutual Information]
		Given two discrete random variables $X$ and $Y$,
		the mutual information (MI) $I(X; Y)$
		is the relative entropy between the joint distribution $p(x,y)$
		and the product of the marginal distributions $p(x)p(y)$, namely,
		\begin{align}
			\begin{split}
				I(X; Y) =& \; D(p(x,y) || p(x)p(y))  \\
				=& \; \sum_{x \in X, y \in Y} p(x,y) \log \big( \frac{p(x,y)}{p(x)p(y)} \big) \\
				=& \; \sum_{x \in X, y \in Y} p(x,y) \log \big( \frac{p(x|y)}{p(x)} \big). \\
			\end{split}
		\end{align}
	\end{definition}
	
	\begin{remark}
		$I(X; Y)$ is symmetrical in $X$ and $Y$, \textit{i.e.}, 
		$I(X; Y)= H(X) - H(X|Y) = H(Y) - H(Y|X) = I(Y;X)$. 
	\end{remark}
	
	\begin{proposition}[Chain Rule for Entropy]
		$H(X_1, X_2, \cdots, X_n) = \sum_{i=1}^{n} H(X_i | X_1, X_2, \cdots, X_{i-1})$.
	\end{proposition}
	
	\begin{proposition}[Chain Rule for Conditional Entropy]
		$H(X_1, X_2, \cdots, X_n | Y) = \sum_{i=1}^{n} H(X_i | X_1, X_2, \cdots, X_{i-1}, Y)$.
	\end{proposition}
	
	\begin{proposition}[Chain Rule for Mutual Information]
		$I(X_1, X_2, \cdots, X_n; Y) = \sum_{i=1}^{n} I(X_i; Y | X_1, X_2, \cdots, X_{i-1})$.
	\end{proposition}
	
	\begin{corollary}
		${\forall A, Z_i, Z_j}, I(A;Z_i, Z_j) \geq \max \big( I(A;Z_i), I(A; Z_j) \big)$.
		\begin{proof}
			As $ I(A; Z_i | Z_j) \geq 0$,
			$I(A;Z_i, Z_j) = I(A; Z_i) + I(A; Z_i | Z_j) \geq I(A;Z_i)$.
			Similarly, $I(A;Z_i, Z_j) \geq I(A;Z_j)$ can be obtained.
			Thus, we have $I(A;Z_i, Z_j) \geq \max \big( I(A;Z_i), I(A; Z_j) \big)$.
		\end{proof}
	\end{corollary}
	
	\begin{proposition}[Chain Rule for Conditional Mutual Information]
		${I(X_1, \! \cdots \!, X_n; Y \! | \! Z) \! = \! \sum_{i=1}^{n} I(X_i; Y \! | \! X_1, \! \cdots \!, X_{i-1}, Z)}$.
	\end{proposition}

	
	\begin{definition}[Markov]
		A discrete stochastic process is called Markov if it satisfies
		$p(x_{i+1} | x_{i} , x_{i-1}, x_{i-2}, \cdots, x_{1}) = p(x_{i+1} | x_{i}) \; \forall i$.
		As such,
		$\forall n > 1$,
		$p(x_{1}, x_{2}, \dots, x_{n}) = p(x_{1})p(x_{2}|x_{1})p(x_{3}|x_{2}) \dots p(x_{n}|x_{n-1})$.
	\end{definition}

	\begin{definition}[Causally Similarity]
		Suppose we have two stochastic processes $X(t)$ and $Y(t)$ defined
		on the ordered set $R$ with associated probability functions $p$ and $q$ and the same
		outcome sets $\{\overrightarrow{x}(t)\}$. 
		We say that the two processes are \textit{causally similar} if
		$p(\overrightarrow{x}(t) | \overrightarrow{x}(t-a)) = q(\overrightarrow{x}(t) | \overrightarrow{x}(t-a))$
		$\; \forall t \; \text{and} \; \forall a > 0$.
	\end{definition}

	\begin{remark}
		\textit{
			Two stochastic processes are causally similar
			if they are time homogenous, Markov, and share the same transition matrix.
			Besides,
			two Markov processes are also causally similar
			that is not necessarily time homogenous 
			if they share the same transition matrix at the same time step.
		}
	\end{remark}

	\begin{lemma}
		Suppose we have two causally similar stochastic processes 
		with probability functions at time t of 
		$p(\overrightarrow{x_t})$ and $q(\overrightarrow{x_t})$.
		Then
		$D(p(\overrightarrow{x_t}) || q(\overrightarrow{x_t})) \leq D(p(\overrightarrow{x_s}) || q(\overrightarrow{x_s}))$
		when $t > s$.
	\end{lemma}

	\begin{lemma}
		In a stationary Markov process, the entropy conditioned 
		on the initial condition is non-decreasing.
		\begin{proof}
			That is, $H(X_{n}|X_{1}) \geq H(X_{n}|X_{1}, X_{2})$ 
			as further conditioning reduces entropy.
			Besides,
			$H(X_{n} | X_{1}, X_{2}) = H(X_{n}|X_{2}) = H(X_{n-1}|X_{1})$.
			Thus, $H(X_{n}|X_{1}) \geq H(X_{n-1}|X_{1})$,
			which shows that $H(X_{n}|X_{1})$ is non-decreasing.
		\end{proof}
	\end{lemma}
	
	\subsection{Proof for Theorem~\ref{theorem: reducing MI with two chains}}
	\label{ssec: proof of reducing MI with two chains}
	
	\begin{lemma}[Invertible Transformations Are Invariant to MI]
		The mutual information is invariant
		to any invertible transformations $\psi(\cdot), \phi(\cdot)$, 
		namely, $I(X; Y) =  I(\psi(X); Y) = I(X; \phi(Y)) = I(\psi(X); \phi(Y))$.
		\label{lemma: Invertible transformations are invariant to MI}
	\end{lemma}
	
	\begin{lemma}[Non-invertible Transformation Reduces MI]
		For any non-invertible transformation $\psi(\cdot)$, 
		it reduce the MI between $X$ and $Y$ as
		$I(X; Y) \geq I(\psi(X); Y) \geq I(\psi(X); \psi(Y))$.
		\label{lemma: Non-invertible transformation reduces MI}
	\end{lemma}
	
	\begin{proof}
		
		Based on Lemma~\ref{lemma: Invertible transformations are invariant to MI}
		and Lemma~\ref{lemma: Non-invertible transformation reduces MI},
		we have the following deduction.
		As $\bm{Z}_{A}^{i+1} = \sigma(\psi(A) \cdot \bm{Z}_{A}^i \cdot \bm{\theta}^{i})$,
		where graph convolution kernel $\psi(A) \in \mathbb{R}^{N \times N}$
		and weights $\bm{\theta}^{i} \in \mathbb{R}^{D \times D}$ are invertible transformations.
		Note the activate function $\sigma(\cdot)$ (\textit{e.g.}, ReLU) 
		is a non-invertible transformation that $H(X) \geq H(\sigma(X))$,
		and $I(X; Y) \geq I(\sigma(X); \sigma(Y))$.		
		
		\begin{align}
			\begin{split}
				I(\bm{Z}_{A}^{i+1} \; ; \; \bm{Z}_{\bm{\hat{A}}}^{i+1}) 
				\; = \; & I \big( \sigma(\psi(A) \cdot \bm{Z}_{A}^i \cdot \bm{\theta}^{i}) \; ; \; \sigma(\psi(\bm{\hat{A}}) \cdot \bm{Z}_{\bm{\hat{A}}}^i \cdot \bm{\theta}^{i}) \big) \\
				\; \leq \; & I \big( \psi(A) \cdot \bm{Z}_{A}^i \cdot \bm{\theta}^{i} \; ; \; \psi(\bm{\hat{A}}) \cdot \bm{Z}_{\bm{\hat{A}}}^i \cdot \bm{\theta}^{i} \big) \\
				\; = \; & I \big( \psi(A) \cdot \bm{Z}_{A}^i \; ; \; \psi(\bm{\hat{A}}) \cdot \bm{Z}_{\bm{\hat{A}}}^i \big)
				= I \big( \bm{Z}_{A}^i \cdot \bm{\theta}^{i} \; ; \; \bm{Z}_{\bm{\hat{A}}}^i \cdot \bm{\theta}^{i} \big) \\
				\; = \; & I(\bm{Z}_{A}^{i} \; ; \; \bm{Z}_{\bm{\hat{A}}}^{i}). 
			\end{split}
		\end{align}
		
		Thus, $\forall i \in [L-1], I(\bm{Z}_{A}^{i+1} ; \bm{Z}_{\bm{\hat{A}}}^{i+1}) \! \leq \! I(\bm{Z}_{A}^{i} ; \bm{Z}_{\bm{\hat{A}}}^{i})$.
		The layer-wise MI of the two Markov chains is decreasing by layers.
	\end{proof}

	\subsection{Proof for Theorem~\ref{theorem: GRA attack lower bound}}
	\label{ssec: proof of GRA attack lower bound}
	
	
	\begin{lemma}[Finite sample bounds]
		%
		Assume all variables' empirical estimates of the mutual information are based on finite support,
		\textit{i.e.}, $K = |\hat{X}| \approx 2^{I(\hat{X}; X)}$.
		%
		Then, we denote by $\hat{I}(\cdot ; \cdot)$ the
		finite sample distribution $\hat{p}(x, y)$ for a given sample of size $n$. 
		The generalization bounds in \citep{shamir2010learning} guarantee that
		\begin{align}
			\begin{split}
				I(\hat{X}; Y) \leq& \; \hat{I}(\hat{X}; Y) + O(\frac{K|\mathcal{Y}|}{\sqrt{n}}), \\
				I(\hat{X}; X) \leq& \; \hat{I}(\hat{X}; X) + O(\frac{K}{\sqrt{n}}).
			\end{split}
		\end{align}
	\end{lemma}
	
	\begin{proof}
		Let the random variables $X$ and $Y$ represent input and output messages with a joint probability
		$P(x,y)$. Let $e$ represent an occurrence of error, \textit{i.e.}, that $X \neq \hat{X}$
		with $\hat{X}=f(Y)$ being an approximate version of $X$.
		Fano's inequality~\citep{fano1961transmission} 
		(also known as the Fano converse) 
		states that the conditional entropy
		\begin{align}
			\begin{split}
				H(X | Y) =& \; - \sum_{i,j} P(x_i, y_j) \log P(x_i | y_j) \\
				\leq& \; H_b(e) + P(e) \log(|\mathcal{X}| + 1),
			\end{split}
		\end{align}
		where the probability of the communication error
		$P(e) \triangleq P(X \neq \hat{X}) 
		\geq \frac{H(X | Y) - 1}{\log(|\mathcal{X}|)}
		$,
		and 
		$H_b(e)$ is the corresponding binary entropy
		that computed as 
		$H_b(e) = -e \log_2(e) - (1-e) \log_2(1-e)$.

		Note that the data processing inequality~\citep{cover1999elements}
		indicates that any three variables $X, Y, Z$ that form a Markov chain $X \! \to \!  Y \!  \to \!  Z$,
		satisfy $I(X; Y) \; \geq \; I(X; Z)$ and $I(Y; Z) \; \geq \; I(X; Z)$.
		As $(A,X) \xrightarrow[]{\bm{\theta}^1} \bm{Z}_{A}^{1}$  
		and $(\bm{\hat{A}},X) \xrightarrow[]{\bm{\theta}^1} \bm{Z}_{\bm{\hat{A}}}^{1}$, 
		we have
		\begin{align}
			\begin{split}
				I(A; \bm{\hat{A}}) 
				\; \geq \; I(\bm{Z}_{A}^{1} \; ; \; \bm{Z}_{\bm{\hat{A}}}^{1}) 
				\; \geq \; I(\bm{Z}_{A}^{2} \; ; \; \bm{Z}_{\bm{\hat{A}}}^{2}) 
				\; \geq \;  \cdots \; \geq \;  I(\bm{Z}_{A}^{L} \; ; \; \bm{Z}_{\bm{\hat{A}}}^{L}) 
				\; =  \; I(\bm{H}_{A} \; ; \; \bm{H}_{\bm{\hat{A}}}) \\
			\end{split}
		\end{align}
		
		Then, according to Fano's inequality~\citep{fano1961transmission},
		the lower bound of MI $I(\bm{H}_{A} ; \bm{H}_{\bm{\hat{A}}})$ is
		\begin{align}
			\begin{split}
				I(\bm{H}_{A} \; ; \; \bm{H}_{\bm{\hat{A}}}) \; = \; & H(\bm{H}_{A}) - H(\bm{H}_{A} \; | \; \bm{H}_{\bm{\hat{A}}}) \\
				\; \geq \; & H(\bm{H}_{A}) - H_b(e) - P(e) \log(|\mathcal{H}|), \\
			\end{split}
		\end{align}
		where entropy $H(\bm{H}_{A}) = \mathbb{E}_{x \in \bm{H}_{A}}[-\log p(x)] = - \sum_{x \in \bm{H}_{A}} p(x) \log p(x)$,
		the probability of approximation error
		$P(e) = P(\bm{H}_{A} \neq \bm{H}_{\bm{\hat{A}}})$,
		and the binary entropy $H_b(e) = -e \log_2(e) - (1-e) \log_2(1-e)$.
		$\mathcal{H}$ denotes the support of $\bm{H}_{A}$.
		Specifically, the approximation fidelity
		$I(A; \bm{\hat{A}}) \; \geq \; 
		- \sum_{x \in \bm{H}_{A}} p(x) \log p(x)
		+e \log_2(e) + (1-e) \log_2(1-e)
		- P(e) \log(|\mathcal{H}|)
		$.
	\end{proof}

	\subsection{Proof for Theorem~\ref{theorem: optimal fidelity in GRA}}
	\label{ssec: proof of optimal fidelity in GRA}
	
	\begin{proof}
		
		To learn $\bm{\hat{A}}$
		given the prior knowledge $\mathcal{K} = \{X,Y,\bm{H}_{A} \}$,
		we have 
		$H(\bm{\hat{A}}) \leq H(\mathcal{K})$, and
		$\forall Z, I(Z; \bm{\hat{A}}) \leq I(Z; \mathcal{K}) = I(Z; X,Y,\bm{H}_{A})$.
		Thus,
		the recovering fidelity of $\bm{\hat{A}}$ satisfies
		${I(A; X,Y,\bm{H}_{A}) - I(A; \bm{\hat{A}}) \geq 0}$.
		Then, we obtain the upper bound of the attack fidelity
		with the optimal recover adjacency $\bm{\hat{A}}^{*}$, namely,
		\begin{align}
			\begin{split}
				\bm{\hat{A}}^{*} = \max_{\bm{\hat{A}}} I(A; \bm{\hat{A}}) = I(A;\mathcal{K}) = I(A; X,Y,\bm{H}_{A}), \\
				s.t.\; \; I(A; \mathcal{K} | \bm{\hat{A}}^{*}) =  I(A; \bm{\hat{A}}^{*} | \mathcal{K}) =  0.
			\end{split}
		\end{align}

		Solving MC-GRA (Eq.~\eqref{eqn: MC-GRA})
		that $\exists \alpha_{1}, \alpha_{2} \in \mathbb{R}$,
		$\bm{\hat{A}}^* \! = \! \arg \max_{\bm{\hat{A}}} 
		\sum_{i=1}^{L} 	
		\alpha_{1} I(\bm{H}_A ; \bm{H}^{i}_{\bm{\hat{A}}})
		+ \alpha_{2} I(Y ; \bm{Y}_{\bm{\hat{A}}})$
		yields a sufficient solution to achieve the optimal fidelity, \textit{i.e.},
		$\bm{\hat{A}}^{*}: I(A; \bm{\hat{A}}^{*}) \! = \! I(A; X,Y,\bm{H}_{A})$.
		However, the optimal $\bm{\hat{A}}^{*}$
		does not necessarily mean exactly recover the original $A$,
		as $H(A | \bm{\hat{A}}^{*}) = H(A) - I(A; \bm{\hat{A}}^{*}) \geq 0$.
		Intuitively,
		the perfect recovery can not be achieved
		due to the data compression nature of the learning process.
		Besides,
		$H(A) \geq \max_{Z \in \mathcal{K}} H(Z)$ is a usual case
		as the hidden dimension $D 	\ll N$.
		The remaining information, \textit{i.e.},
		$H(A | \bm{\hat{A}}^{*}) = H(A | \mathcal{K})$,
		that is unobservable from $\mathcal{K} = \{X,Y,\bm{H}_{A}\}$,
		can not be recovered unless additional information is provided.
		
	\end{proof}

	\subsection{Proof for Theorem~\ref{theorem: maximum adjacency information}}
	\label{ssec: proof of maximum adjacency information}
	\begin{proof}
		$\forall X, Y$, we have
		$I(X; Y) \leq I(X; X) = H(X)$.
		Thus, the MI between 
		representations $\bm{H}_{A}$ and adjacency $A$ satisfies that
		$I(A; \bm{H}_{A}) \leq I(A;A) = H(A)$.
		Which means,
		the upper bound of the MI, 
		\textit{i.e.}, the worst privacy leakage as $I(A; \bm{H}_{A}) \leq H(A)$,
		is that all the private information $H(A)$ about the adjacency 
		is obtained for the attacker.
	\end{proof}

	\subsection{Proof for Theorem~\ref{theorem: minimum adjacency information}}
	\label{ssec: proof of minimum adjacency information}
	
	\begin{proof}
		For any sufficient graph representations $\bm{H}_{A}$ 
		of adjacency $A$ \textit{w.r.t.} task $Y$
		introduced in Proposition~\ref{prop: sufficient graph representation},
		its MI with $A$ satisfies that
		$I(A; Y) = I(\bm{H}_{A}; Y)$,
		as $\bm{H}_{A}$ can be seen as extracted from $A$.
		However,
		$I(A; \bm{H}_{A}) \! \geq \! I(A; Y)$
		as the data processing inequality~\citep{cover1999elements}
		in Markov chain $A \to \bm{H}_{A} \to Y$.
		Based on the two above conditions,
		the minimum information $I(A; \bm{H}_{A}) = I(A; Y)$
		can be achieved  if and only if $I(A; \bm{H}_{A} | Y) = 0$.
		That is, the optimal representations $\bm{H}_{A}^*$ satisfy 
		(1) sufficient condition that $I(A; Y) = I(\bm{H}_{A}^*; Y)$,
		and (2) minimal condition that $I(A; \bm{H}_{A}^*) = I(A; Y)$.
		
		Thus, $I(A; \bm{H}_{A}^* | Y) = I(A; \bm{H}_{A}^*, Y) - I(A,Y) 
		= I(A,Y) - I(A,Y) = 0$.
	\end{proof}

	\subsection{Proof for Theorem~\ref{theorem: GPB approximate the optimal representation}}
	\label{ssec: proof of GPB approximate the optimal representation}
	\begin{proof}
		
		When degenerate $\beta_c \! = \! 0$ and $\beta^{i} \! = \! \beta$,
		MC-GPB is equivalent to
		minimizing the Information Bottleneck Lagrangian~\citep{shwartz2017opening}, \textit{i.e.},
		$\mathcal{L}(p(\bm{Z}|A)) = H(Y|\bm{Z}) + \beta I(\bm{Z}; A)$,
		in the limit ${\beta \to 0}$.
		Specifically,
		$\mathcal{L}(p(\bm{Z}|A))
		= H(Y| \bm{Z}) + \beta I(\bm{Z}; A)
		= H(Y) - I(\bm{Z}; Y) + \beta I(\bm{Z}; A)
		\propto - I(\bm{Z}; Y) + \beta I(\bm{Z}; A)$, 
		where entropy $H(Y)$ is a constant.
		Then, we deduce the optimal case of 
		$\min \mathcal{L}(p(\bm{Z}|A)) = \max I(\bm{Z}; Y) - \beta I(\bm{Z}; A)$
		as follows.
		\begin{align}
			\begin{split}
				&\max I(\bm{Z}; Y) - \beta I(\bm{Z}; A)  \\
				=&	\max \big(I(Y; \bm{Z}, A)	- I(A; Y | \bm{Z})\big) - \beta \big(I(\bm{Z}; A, Y) - I(A; Y | \bm{Z})\big) \\
				=&	\max I(Y; \bm{Z}, A)	- (1-\beta) I(A; Y | \bm{Z}) - \beta I(\bm{Z}; A, Y) \\
				= &\max 	I(Y; A)	- (1-\beta) I(A; Y | \bm{Z}) - \beta I(\bm{Z}; A, Y) \\
				=&	\max (1-\beta) I(A; Y) - (1-\beta) I(A ;Y | \bm{Z}) - \beta I(\bm{Z} ; A | Y)	\\
				=& (1-\beta) I(A; Y).
			\end{split}
		\end{align} 
		
		As the two MI terms
		$I(A ;Y | \bm{Z}) \geq 0$ and $I(\bm{Z} ; A | Y) \geq 0$,
		the optimal $\bm{Z}^*$
		should satisfies that $I(A ;Y | \bm{Z}^*) = I(\bm{Z}^* ; A | Y) = 0$.
		As such, it yields a sufficient representation $\bm{Z}$ of 
		data $A$ for task $Y$,
		that is an approximation to the minimal and sufficient representations $\bm{Z}^*$ 
		in Proposition~\ref{prop: minimal graph representation},
		\textit{i.e.},
		$\bm{Z}^* = \arg \min_{\bm{Z}: \; I(\bm{Z};Y) = I(A;Y)} I(\bm{Z}; A)$.
	\end{proof}

	\section{Full empirical study}

	\subsection{Datasets}
	\label{ssec: datasets}
	Six common datasets are utilized in experiments,
	which are collected from four diverse domains:
	(1) Cora and Citeseer~\citep{sen2008collective} are citation networks
	where nodes are documents, and edges indicate citations among them;
	(2) Polblogs~\citep{adamic2005political} is a social network of political blogs
	where nodes represent blogs with political leaning while edges are citations;
	(3) USA and Brazil~\citep{ribeiro2017struc2vec}
	are air-traffic networks where nodes are airports and edges denote airlines;
	(4) AIDS~\citep{riesen2008iam} 
	is a chemical network where
	each node is an atom, and each edge is a chemical bond.
	The data statistics are in Tab.~\ref{tab:statistics}.

	\begin{table}[H]
		\centering
		\caption{Dataset statistics. 
			The hard homophily of an edge $e_{ij}$ is computed as $\mathbb{I}(y_i, y_j)$ with node labels,
			while the soft homophily is calculated by $\cos(x_i, x_j)$ with node features.
                "---" means this dataset is intrinsic without node features.
		}
		\setlength\tabcolsep{10pt}
		\vspace{-6px}
		\begin{tabular}{c|ccccccc}
			\toprule
			dataset &  {    Cora    }  & Citeseer &  Polblogs   & USA   &   Brazil &  AIDS     \\
			\midrule
			\# Nodes &	2,708	&	3,327	&	 1490 &	 1190	&	131	&	1429		\\
			\# Edges &	5,278	&	4,676	&	33430	&	27164	&	2077	&	2948		\\
			\# Class &	7	&	6	&	2	&	4	&	4	&	14	\\
			\# Features  &	 1433	&	3703	&	---	&	---	&	---	&	4		\\
			Soft homophily & 0.83    &   0.81	&	--- 	&	--- 	&	---	& 0.06			\\
			Hard homophily & 0.81    &	0.74	& 0.91		& 0.70		&	0.45 &	0.51		\\
			\bottomrule
		\end{tabular}
		\label{tab:statistics}
	\end{table}

	\subsection{Full quantitative results}
	\label{sec: full quantitative results}

	
	\textbf{A further comparison of attack methods. }
	For the attack, 
	we further compare the proposed method (MC-GRA) 
	with three baselines in the below table. 
	Here, the evaluation is also with the AUC metric, 
	where a higher value indicates a better attack performance. 
	The \textbf{boldface} numbers represent the best results. 
	As can be seen from the table below, 
	the MC-GRA consistently achieves the best in all six datasets, 
	outperforming all the baselines by a large margin.
	
	
	\begin{table}[H]
		\centering		\caption{
			A further quantitative comparison of attack methods (with AUC metric).
		}
		\setlength\tabcolsep{10pt}
		\vspace{-6px}
		\begin{tabular}{c|cccccc}
			\toprule
			dataset &  {    Cora    }  & Citeseer &  Polblogs   & USA   &   Brazil &  AIDS     \\
			\midrule
			single MI (Tab.~\ref{tab:understanding-MI-term-comparison}) & .815     & .881     & .763     & .850     & .758     & .584     \\
			ensemble (Tab.~\ref{tab:understanding-MI-term-ensemble})      & .849     & .907     & .781     & .852     & .717     & .522     \\
			GraphMI          & .812     & .781     & .791     & .769     & .680     & .575     \\
			MC-GRA (Tab.~\ref{tab:exp-attack-results})        &\textbf{ .904} & \textbf{.931} & \textbf{.853} & \textbf{.870} & \textbf{.760} & \textbf{.588}  \\
			\bottomrule
		\end{tabular}
	\end{table}

	\textbf{A further comparison of defense methods.}
	For the defense, 
	we compare the proposed MC-GPB with two additional defense methods, 
	\textit{i.e.}, 
	adding random noise and differentiable privacy. 
	Specifically, we inject Gaussian noise into the model prediction, 
	termed random noise. 
	While another baseline, 
	termed differential privacy~\cite{abadi2016deep}, 
	is achieved by adding Gaussian noise to the clipped gradients in each training iteration. 
	The empirical results are shown in the below table, 
	where GraphMI~\cite{zhang2021graphmi} is used as the attack method. 
	As can be seen, 
	the defending power of random noise 
	and differential privacy comes at the price of sharply degenerating the model's accuracy. 
	By contrast, our proposed MC-GPB significantly 
	degenerates GRA with much lower AUC 
	while maintaining high accuracy simultaneously. 
	
	
	\begin{table}[H]
		\centering
		\caption{
			A further quantitative comparison of defense methods.
		}
		\setlength\tabcolsep{2pt}
		\vspace{-6px}
		\begin{tabular}{c|cc|cc|cc|cc|cc|cc}
			\toprule
			\multirow{2}{*}{dataset} &  \multicolumn{2}{c|}{{    Cora    }}  & \multicolumn{2}{c|}{Citeseer} &  \multicolumn{2}{c|}{Polblogs}   & \multicolumn{2}{c|}{USA}   &   \multicolumn{2}{c|}{Brazil} &  \multicolumn{2}{c}{AIDS}     \\
			& ACC$\uparrow$ & AUC$\downarrow$  & ACC$\uparrow$ & AUC$\downarrow$  & ACC$\uparrow$ & AUC$\downarrow$  & ACC$\uparrow$ & AUC$\downarrow$  & ACC$\uparrow$ & AUC$\downarrow$  & ACC$\uparrow$ & AUC$\downarrow$  \\
			\midrule
			No defense                 & .757 & .812     & .630 & .781     & .833 & .791     & .470 & .769     & .769 & .680     & .668 & .575     \\
			Random noise               & .620 & .657     & .570 & .727     & .802 & .759     & .440 & .754     & .634 & .713     & .572 & .559     \\
			Differential privacy     & .315 & .500     & .224 & .500     & .553 & .502     & .263 & .500     & .423 & .706     & .131 & .502     \\
			MC-GPB               & .734 & .625 & .602 & .691 & .830 & .506 & .391 & .300 & .808 & .609 & .668 & .514 \\
			\bottomrule
		\end{tabular}
	\end{table}

    \textbf{Ablation study of similarity measurement.}
	We also conduct experiments with the influence of similarity measurement since our implementation depends on the estimation of mutual information, shown as Tab.~\ref{tab:ablation-similarity-metrics}.
    As can be seen, the MC-GRA has consistent performance across different similarity measurements, while the MC-GPB exhibits a high variance for different similarity measurements.
    Therein, the HSIC and CKA are generally good choices.
    %
	\begin{table}[H]
		\centering
		\caption{
			Ablation study of similarity measurements
                (with AUC metric).
		}
		\vspace{-6px}
		\label{tab:ablation-similarity-metrics}
		\begin{tabular}{cc|cccc|cccc}
			\toprule
			\multirow{2}{*}{type} & \multirow{2}{*}{case}  & \multicolumn{4}{c|}{Cora} & \multicolumn{4}{c}{USA} \\
			& & DP & HSIC & CKA & KDE & DP & HSIC & CKA & KDE  \\
			\midrule
			\multirow{3}{*}{attack} 
			& $\mathcal{K} \! = \! \{X,  Y\}$
			& .876 & .871 &  .873 &  .876 &  .791 &  .800  &  .802 &  .802  \\
			& $\mathcal{K} \! = \!  \{X,  Y, \bm{H}_{A}\}$ 
			& .892 & .890 &  .892 &  .895 &  .856 &  .850  &  .845 &  .851  \\
			& $\mathcal{K} \! = \!  \{X,  Y, \bm{H}_{A}, \bm{\hat{Y}} \}$
			& .898 & .898 &  .904 &  .896 &  .846 &  .852  &  .818 &  .840  \\
			\midrule
			\multirow{3}{*}{defense} 
			& $I(A; \bm{H}_A)$
			& .476 & .751 &  .701 &  .706 &  .716 &  .873 &  .879  &  .883 \\
			& $I(A; \bm{\hat{Y}}_A)$
			& .508 & .688 &  .705 &  .704 &  .587 &  .542 &  .872  &  .873 \\
			& $I(A; \bm{H}_{\bm{\hat{A}}})$
			& .505 & .644 &  .644 &  .625 &  .300 &  .467 &  .770  &  .728 \\
			& Acc.
			& .306 & .635 & .758  & .734  &  .391 &  .319 &  .431 &  .447  \\
			\bottomrule
		\end{tabular}
		\vspace{-6px}
	\end{table}

        \textbf{Ablation study of parameterization methods.}
        We also provide the empirical result of different parameterization methods of MC-GRA, which are mentioned in Sec.~\ref{sec: GRA attack}. Overall, the GNNs method achieves the best score out of the three methods, especially for the dataset with given node features~(Cora, Citeseer, AIDS), and exceeds its counterparts by a large margin. Therein, the Gaussian parameterization generally performs better than its counterparts for graphs without node features. 
        
	\begin{table}[H]
		\centering
		\caption{
			Attack with different parameterization methods (with AUC metric).
		}
		\vspace{-6px}
		\label{tab:ablation-parameterization}
		\setlength\tabcolsep{10pt}
		\begin{tabular}{c|cccccc}
			\toprule
			variant & Cora & Citeseer &  Polblogs   & USA   &   Brazil &  AIDS   \\
			\midrule
			Direct Matrix                
			& .890 & .580 & .684  & .737 & .521 & .540  \\
			Gaussian                
			& .893 & .654 & .777  & .846 & .758 & .567  \\
			GNNs                
			& .891 & .889 & .803  & .776 & .731 & .662  \\
			\bottomrule
		\end{tabular}
	\end{table}

	\textbf{A further empirical study about the evaluation metric.}
	Without requiring any estimation, we utilize the AUC (area under the curve) metric with the ground-truth edges in $A$ to quantify the mutual information $I(A; \bm Z)$. And alternatively, the metric can be replaced by AP (Average Precision), MRR (Mean Reciprocal Rank), Hit@K (the ratio of positive edges that are ranked at the K-th place or above), which are also common in the link prediction task.
	
	The metrics mentioned above treat all edges equally indeed. An objective measurement is required to further discriminate between high-value and low-value links. Here, the link homophily is a proper measurement. For instance, the link (Jaime, Tyrion) in Figure.~\ref{fig: motivation} can be seen as a \textit{homogeneous} link because Jaime and Tyrion have the same node labels (\textit{i.e.}, Lannister). On the other hand, the link (Daenerys, Jon) can be seen as a \textit{heterogeneous} link because Daenerys and Jon do not have the same node labels (as audiences in the earlier period, we do not know they are Targaryens, and they have a kinship). 
	Formally, the homogeneous links can be denoted as $\{e_{ij}: y_i = y_j\}$ while the heterogeneous links are $\{e_{ij}: y_i \neq y_j\}$, where $y_i,y_j$ are node labels of node $i$, node $j$. 
	In what follows, we further investigate the effectiveness of GRA on these two kinds of links.
	
	\textbf{(1)} For the attack, the homogeneous links are much easier to recover, as shown in the table below. More importantly, it is observed that the high-value heterogeneous links are naturally protected but can still be recovered to some extent.
	
	
	\begin{table}[H]
		\centering
		\caption{
			A further quantitative comparison 
			of attack methods
			on homogeneous or heterogeneous links
                (with AUC metric).
		}
		\setlength\tabcolsep{10pt}
		\vspace{-6px}
		\begin{tabular}{c|cccccc}
			\toprule
			dataset &  {    Cora    }  & Citeseer &  Polblogs   & USA   &   Brazil &  AIDS     \\
			\midrule
			MC-GRA (Homogeneous links)    & .960 & .917     & .896     & .951 & .891   &  .585  \\
			MC-GRA (Heterogeneous links)  & .684 & .861     & .298     & .716 & .564   &  .551  \\
			GraphMI (Homogeneous links)   & .724 & .799     & .717     & .919 & .871   &  .707  \\
			GraphMI (Heterogeneous links) & .569 & .675     & .391     & .666 & .728   &  .437  \\
			\bottomrule
		\end{tabular}
	\end{table}
	
	\textbf{(2)} For defense, we apply the proposed MC-GPB to protect GCN against the GRA by GraphMI. In addition, we also implement a revised version, \textit{i.e.}, MC-GPB-hetero, which only focuses on protecting the heterogeneous links of the original adjacency matrix. As results are shown in the table below, the recovery of heterogeneous links is significantly degenerated by MC-GPB and further degenerated by MC-GPB-hetero. Thus, we justify that MC-GPB and its revised version are capable of protecting the high-value heterogeneous links.
	
	
	\begin{table}[H]
		\centering
		\caption{
			A further quantitative comparison 
			of attack methods
			on homogeneous or heterogeneous links
                (with AUC metric).
		}
		\setlength\tabcolsep{10pt}
		\vspace{-6px}
		\begin{tabular}{c|cccccc}
			\toprule
			dataset &  {    Cora    }  & Citeseer &  Polblogs   & USA   &   Brazil &  AIDS     \\
			\midrule
			No defense (Heterogeneous links)         & .569 & .675    & .391    & .666 & .728  & .437  \\
			MC-GPB (Heterogeneous links)        & .532 & .584    & .453    & .552 & .530  & .471  \\
			MC-GPB-hetero (Heterogeneous links) & .493 & .515    & .210    & .494 & .416  & .423  \\
			\bottomrule
		\end{tabular}
	\end{table}

	\textbf{Empirical results on large-scale datasets.}
	Here, we conduct an empirical study on two large-scale datasets 
	for node property prediction, \textit{i.e.}, 
	ENZYME (6254 nodes and 23914 edges) and 
	OGB-Arxiv (8532 nodes and 26281 edges). 
	Detailed statistics are shown below.
        Specifically,
	we use GraphSAINT~\cite{zeng2019graphsaint} random walk sampler to extract the subgraph for illustration. 
	The dataset split setting of train/validate/test sets is consistent with other datasets used in this work.
	
	
	\begin{table}[H]
		\centering
		\caption{Dataset statistics of the two large-scale datasets.}
		\setlength\tabcolsep{10pt}
		\vspace{-6px}
		\begin{tabular}{c|ccccc}
			\toprule
			dataset &  \# Nodes  & \# Edges  &  \# Class   & \# Features   &   Hard homophily     \\
			\midrule
			ENZYME    & 6254    & 23914   & 3         & 18          & 0.629   \\
			OGB-Arxiv & 8532    & 26281   & 40        & 128         & 0.618 \\
			\bottomrule
		\end{tabular}
	\end{table}

        \textbf{(1)}
        Then, we evaluate the performance of MC-GRA with $\mathcal{K} = \{X, Y\}$, as shown in the table below (the higher the AUC value, the better the attack performance). As can be seen, 
        MC-GRA is also effective on these two large-scale datasets. 
        Besides, MC-GRA outperforms the baseline GRA method by a large margin.
	
	
	\begin{table}[H]
		\centering
		\caption{
			A comparison 
			of attack methods on large-scale datasets
                (with AUC metric).
		}
		\setlength\tabcolsep{10pt}
		\vspace{-6px}
		\begin{tabular}{c|cc}
			\toprule
			dataset &  ENZYME &  OGB-Arxiv     \\
			\midrule
			GraphMI		& .494  &  .828		\\
			MC-GRA     & .761   &  .891      \\
			\bottomrule
		\end{tabular}
	\end{table}

        \textbf{(2)}
        We also conduct the experiment of our defense method MC-GPB on these two datasets, with GraphMI as the attack method. As shown in the table below (the lower, the better), MC-GPB degenerates both kinds of GRA (\textit{i.e.}, MC-GRA and GraphMI), which empirically proves the effectiveness of our defense method on large-scale datasets. 
	
	
	\begin{table}[H]
		\centering
		\caption{
			A comparison 
			of attack methods on large-scale datasets
			our defense method MC-GPB
                (with AUC metric).
		}
		\setlength\tabcolsep{10pt}
		\vspace{-6px}
		\begin{tabular}{c|cc}
			\toprule
			dataset &  ENZYME &  OGB-Arxiv     \\
			\midrule
			GraphMI		& .488 (1.2\%$\downarrow$)  &  .533  (35.6\%$\downarrow$) 	\\
			MC-GRA    &  .607 (20.2\%$\downarrow$) & .848 (4.8\%$\downarrow$)     \\
			\bottomrule
		\end{tabular}
	\end{table}

	\textbf{Attacks without node feature.}
	We further implement the experiments without node features $X$ in the following. Note that the usair, brazil, and polblogs datasets have no initial node feature. Therefore, calculating $I(A;X)$ in Table.~\ref{tab:understanding-MI-term-comparison} with these datasets is infeasible due to the lack of $X$. Here, we present the attack results of our method on Cora, Citeseer, and AIDS datasets without using the node feature. 
	
	
	\begin{table}[H]
		\centering
		\caption{
			A further quantitative comparison of attack methods
			without access to the node feature
                (with AUC metric).
		}
		\setlength\tabcolsep{10pt}
		\vspace{-6px}
		\begin{tabular}{c|ccc}
			\toprule
			dataset &  {    Cora    }  & Citeseer &  AIDS     \\
			\midrule
			Dot-Product (Tab.~\ref{tab:understanding-MI-term-ensemble})    &  .849    &     .907     &     .521     \\
			GraphMI (without $X$)    &  .802    &     .759     &     .575     \\
			MC-GRA  ($\mathcal{K}=\{\bm{H}_A\}$)    &  .834    &     .887     &     .575     \\
			MC-GRA ($\mathcal{K}=\{\hat{\bm{Y}}_A\}$)    &  .771    &     .890     &     .540     \\
			MC-GRA ($\mathcal{K}=\{\bm{Y}\}$)    &  .864    &     .853     &     .525     \\
			MC-GRA ($\mathcal{K}=\{\bm{H}_A, \hat{\bm{Y}}_A\}$)    &  .828    &     .918     &     .525     \\
			MC-GRA ($\mathcal{K}=\{\bm{H}_A, \bm{Y}\}$)    &  .875    &     .919     &     .539     \\
			MC-GRA ($\mathcal{K}=\{\hat{\bm{Y}}_A, \bm{Y}\}$)    &  .867    &     .896     &     .539     \\
			MC-GRA ($\mathcal{K}=\{\bm{H}_A, \hat{\bm{Y}}_A, \bm{Y}\}$)    &  .883    &     .914     &     .580     \\
			\bottomrule
		\end{tabular}
	\end{table}

	As shown in the above table, 
	without using node features $X$ as the prior knowledge, 
	the MC-GRA is still effective in recovering adjacency with considerable AUC results. 
	Besides, 
	the performance of MC-GRA is still better than the baselines. 
        In fact, node features do not always exist, \textit{e.g.}, 
	for the Polblogs, USA, and Brazil datasets. 
	While on the other hand, the characteristic adjacency $A$ is indispensable for graph learning. Besides, we further justify the feasibility of our MC-GRA without node features. 
	For the extension of GRA, a more fine-grained study on graph properties is intriguing, \textit{e.g.}, density, community, number of triangles \textit{w.r.t.} adjacency $A$. To what extent can the above properties be recovered will shed insights into the power of GRA and the memorization effect of GNNs.
 

	\textbf{Why would some of the model accuracy benefit 
	from the defense mechanism?}
        As shown in Tab.~\ref{tab:exp-defense-results},
        MC-GPB can also bring improvement in classification accuracy
        on partial datasets.
	We speculate that the reason is 
	\textit{forgetting more might also lead to learning better in some cases}.
	We provide a three-fold analysis from the information-theory perspective
        as follows.
        
        \textbf{(1)}
	For brevity, we consider a simplified objective of 
	the graph privacy bottleneck in Eq.~\eqref{eqn: tighter-defense-MI}, 
	\textit{i.e.}, to solve $\min -I(\bm{H};Y) + \beta \cdot I(\bm{H};A)$ 
	\textit{w.r.t.} representations $\bm{H}$, graph adjacency $\bm{A}$, and node labels $\bm{Y}$. 
	The maximin game here is to encourage the accuracy by a higher $I(\bm{H};Y)$, 
	and reduce the complexity by regularizing the $I(\bm{H};A)$ with $\beta$ for the trade-off. 

        \textbf{(2)}
	In this case, the spurious correlation measured by $I(\bm{H};A|Y)$ will also be reduced and help the inference in test time. The reason is that absorbing too much irrelevant information between $\bm{A}$ and $\bm{Y}$, which can be superficially but not causally associated, will lead to degenerated test performance for GNNs~\citep{zhao2022learning}. Thus, a lower $I(\bm{H};A|Y)$ here encourages the forgetting of adjacency and might bring a better generalization power in the test-time inference of GNNs. While the optimal case, \textit{i.e.}, $I(\bm{H};A|Y)=0$, is also discussed in Theorem~\ref{theorem: minimum adjacency information}.

        \textbf{(3)}
	Another supporting material is that only relying on a subgraph for reasoning can also boost the test-time performance~\cite{miao2022interpretable}. The method GSAT~\cite{miao2022interpretable} aims to extract a subgraph $G_s$ as the interpretation. It inherits the same spirit of information bottleneck in its optimization, \textit{i.e.}, $\min -I(G_s;Y) + \beta \cdot I(G_s;G)$. The integrated subgraph sampler can explicitly remove the spurious correlation or noisy information in the entire graph $G$.
	
	Besides, we should note that in the cases where the model does not suffer from severe spurious correlations. The defense mechanism usually induces the drop trade-off regarding the model accuracy.

	\subsection{Full qualitative results}
	\label{sec: full qualitative results}

	\textbf{The recovered adjacency.}
        Fig.~\ref{appd:adj:cora}-\ref{appd:adj:polblogs} shows the recovered adjacency of each dataset, which is grouped by node label under different attack strategies. 
        In addition, we also provide the recovered adjacency on protected GNN, which is training with our proposed MC-GPB mechanism~(sub-figure (d) and (e)). As can be seen, GNN training with MC-GPB successfully resists both attacks in terms of a larger amount of wrong prediction compared to the normal GNN. 
        For example, for the Cora dataset, MC-GRA~(Fig.~\ref{appd:adj:cora:gra_normal}) achieves a better result compared to GraphMI~(Fig.~\ref{appd:adj:cora:gmi_normal}) under both normal training strategy, in terms of fewer error predictions~(red dots). Whereas MC-GPB successfully defended the MC-GRA~(Fig.~\ref{appd:adj:cora:gra_protected}) and GraphMI~(Fig.~\ref{appd:adj:cora:gmi_protected}), and MC-GRA still have better performance compared to GraphMI with protected GNN.
        
	\begin{figure}[H]
		\centering
            \hfill		
            \subfigure[Ground truth]
		{{\includegraphics[height=3.3cm]{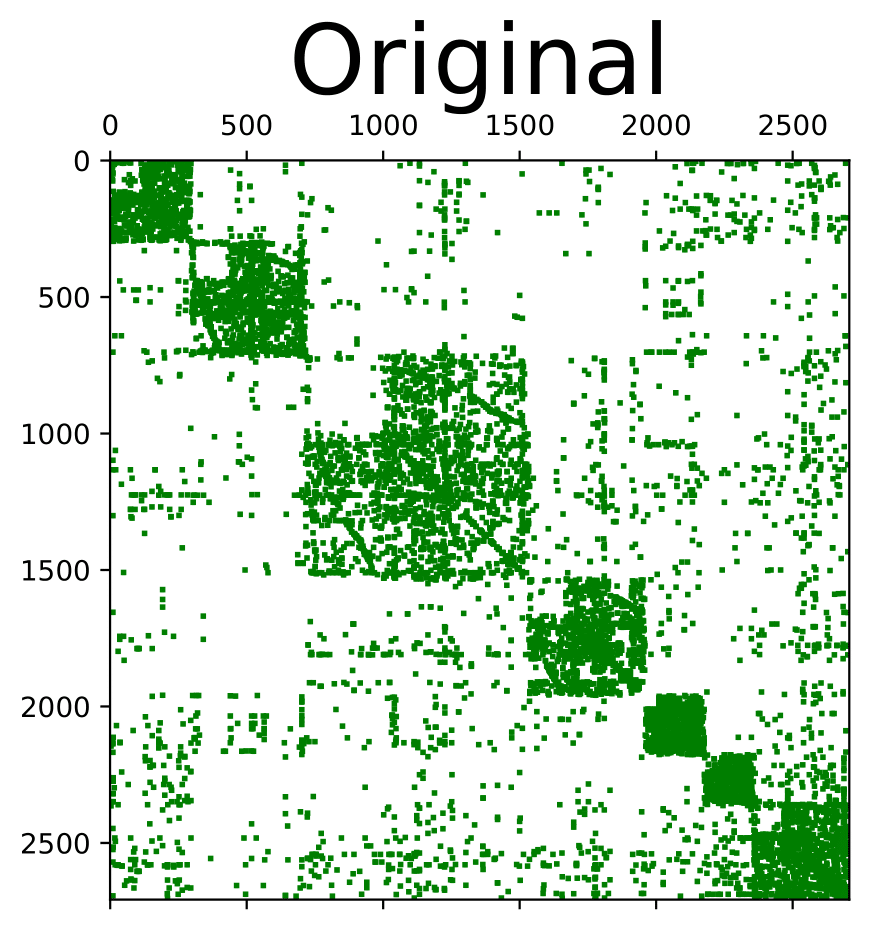}}}
		\hfill
		\subfigure[With normal GNN]
		{{
                \includegraphics[height=3.3cm]{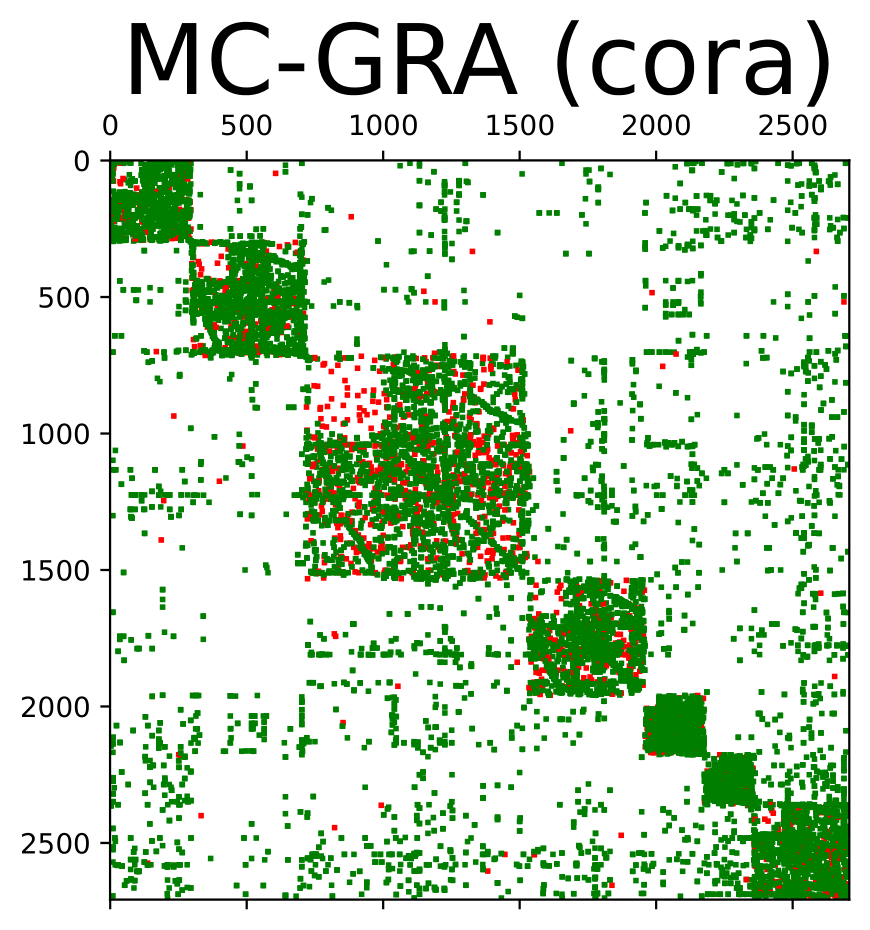} 
                \label{appd:adj:cora:gra_normal}
            }}
		\hfill
		\subfigure[With normal GNN]
		{{
                \includegraphics[height=3.3cm]{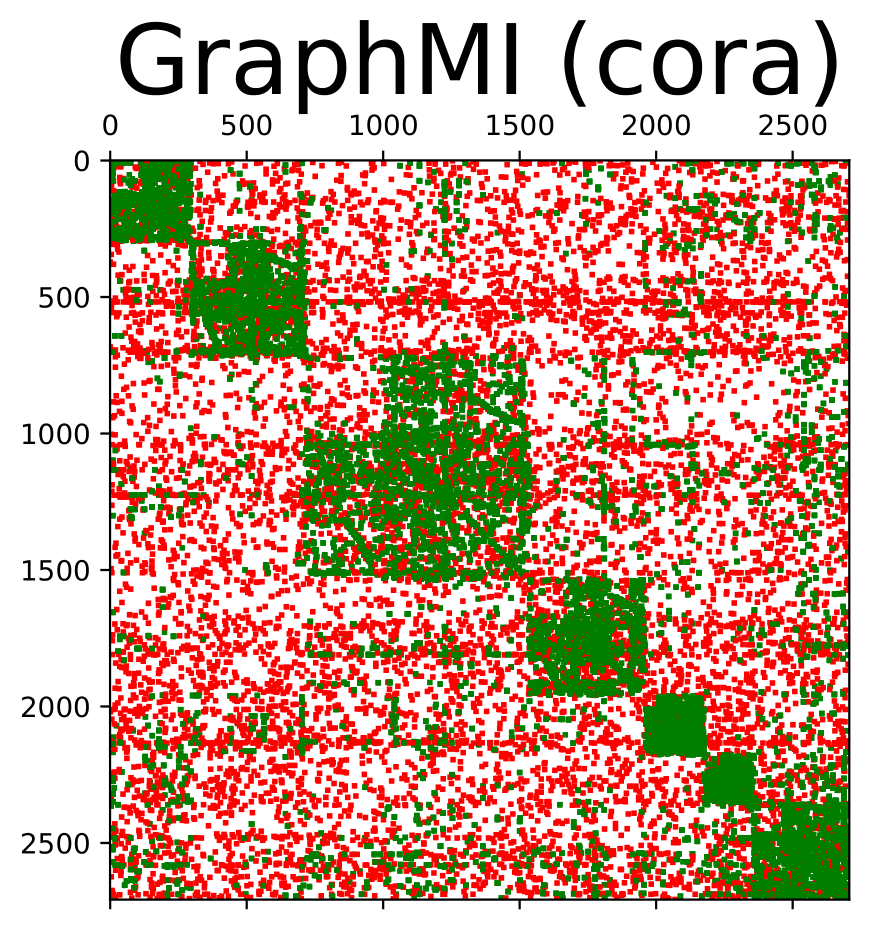}
                \label{appd:adj:cora:gmi_normal}
            }}
		\hfill
		\subfigure[With protected GNN]
		{{
                \includegraphics[height=3.3cm]{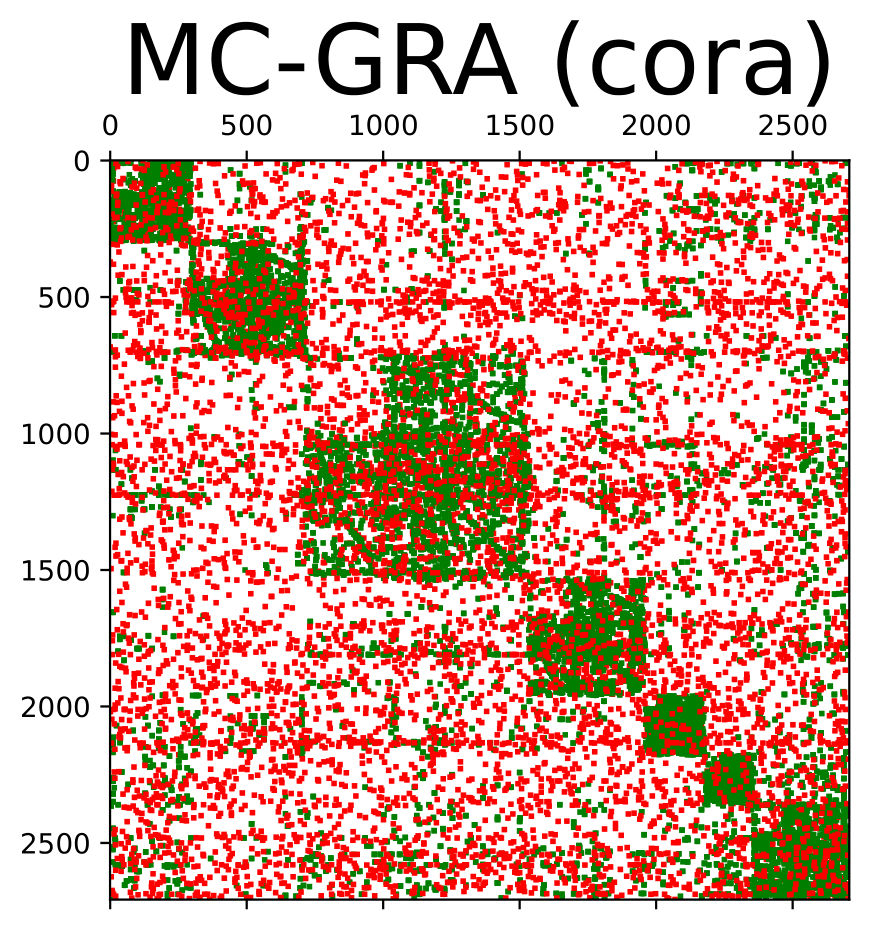}
                \label{appd:adj:cora:gra_protected}
            }}
		\hfill
		\subfigure[With protected GNN]
		{{
                \includegraphics[height=3.3cm]{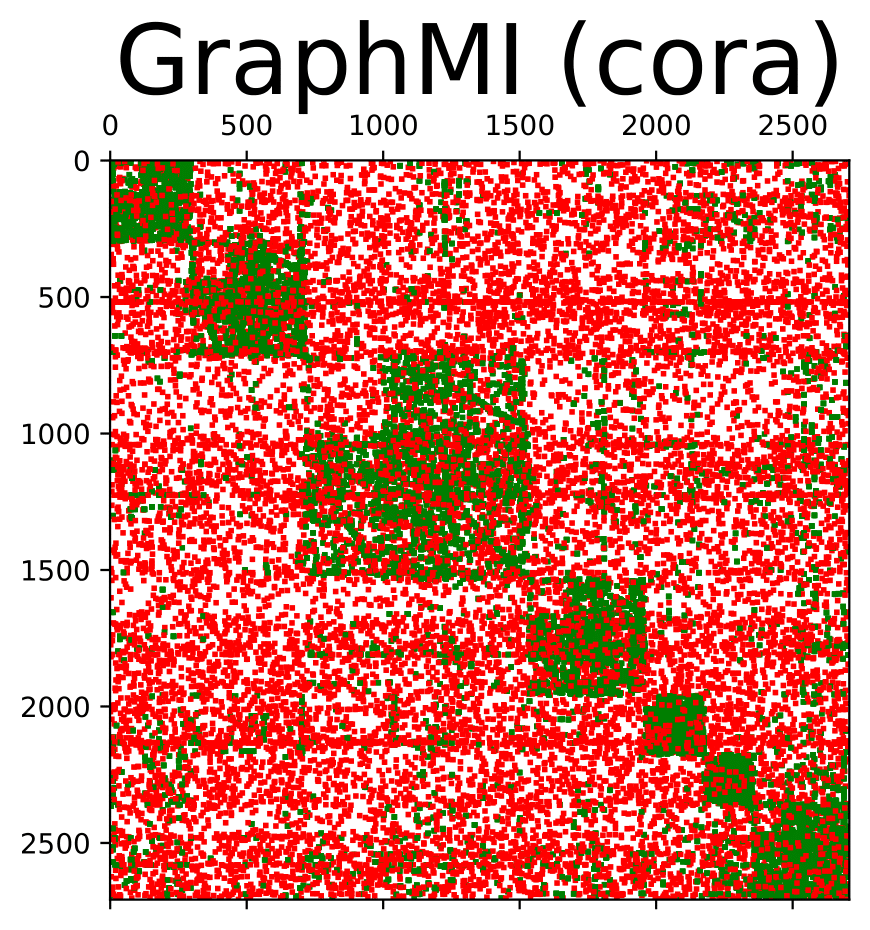}
                \label{appd:adj:cora:gmi_protected}
            }}
		\vspace{-8pt}
		\caption{
			Recovered adjacency on Cora dataset.
		}
        \label{appd:adj:cora}
	\end{figure}

	\begin{figure}[H]
		\centering
		\subfigure[Ground truth]
		{{\includegraphics[height=3.3cm]{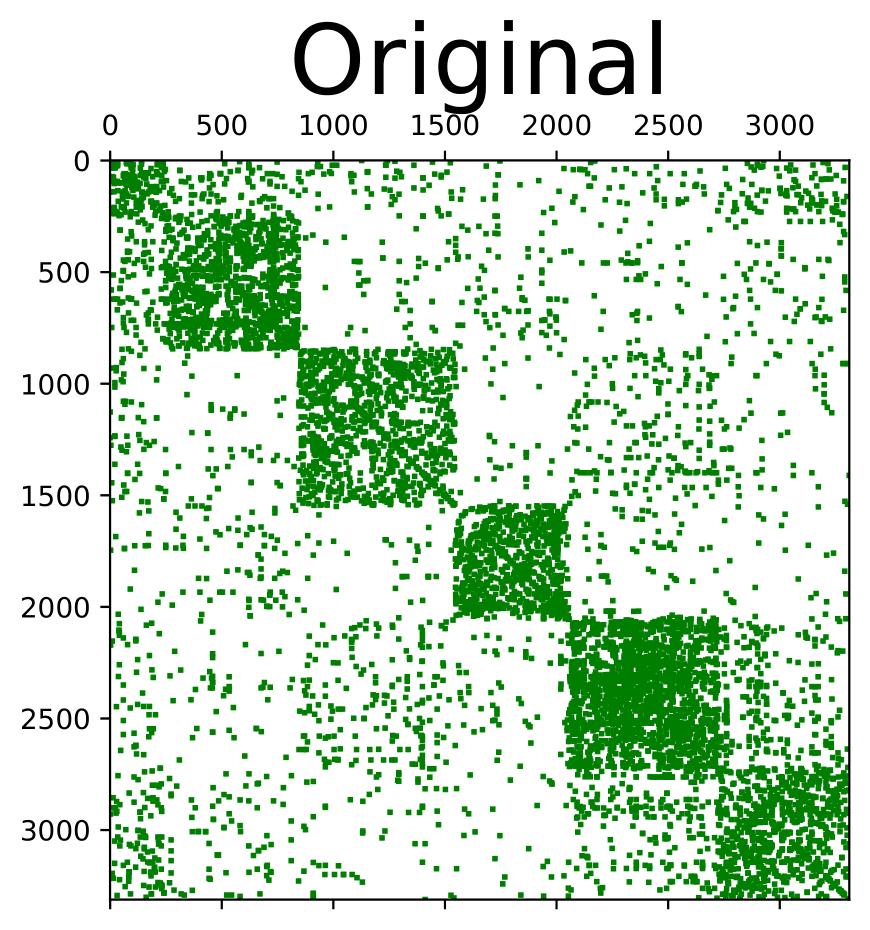}}}
		\hfill
		\subfigure[With normal GNN]
		{{\includegraphics[height=3.3cm]{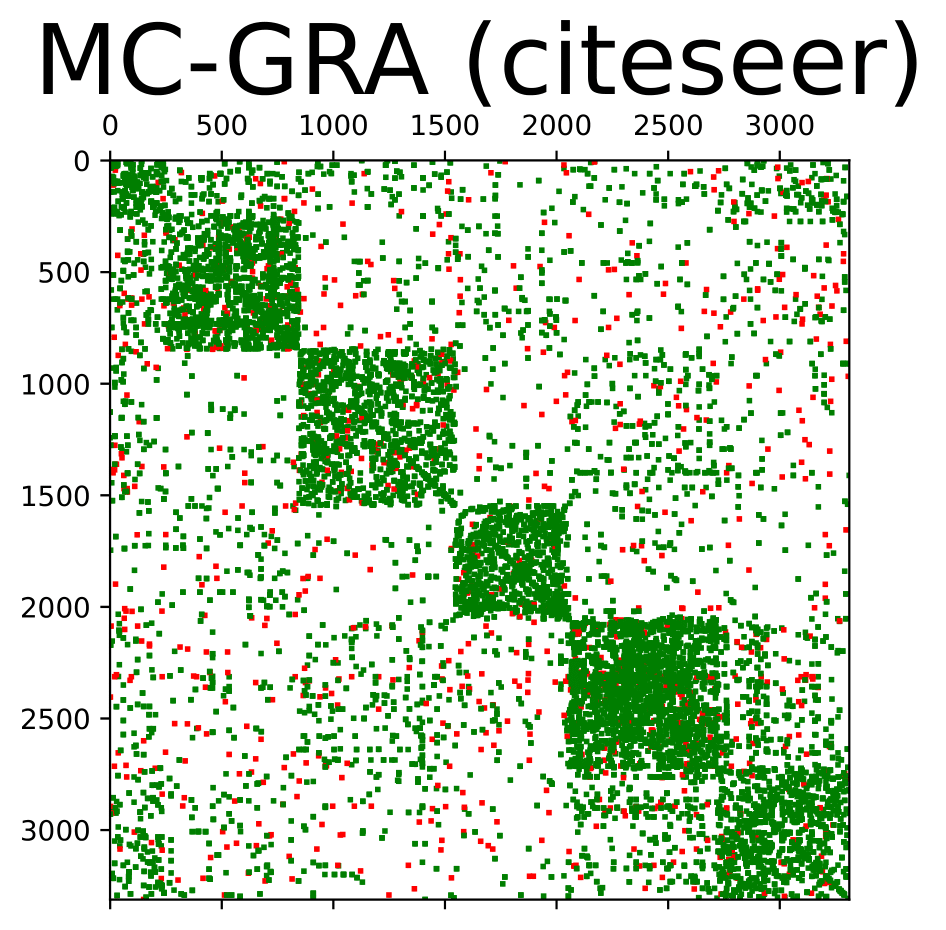}}}
		\hfill
		\subfigure[With normal GNN]
		{{\includegraphics[height=3.3cm]{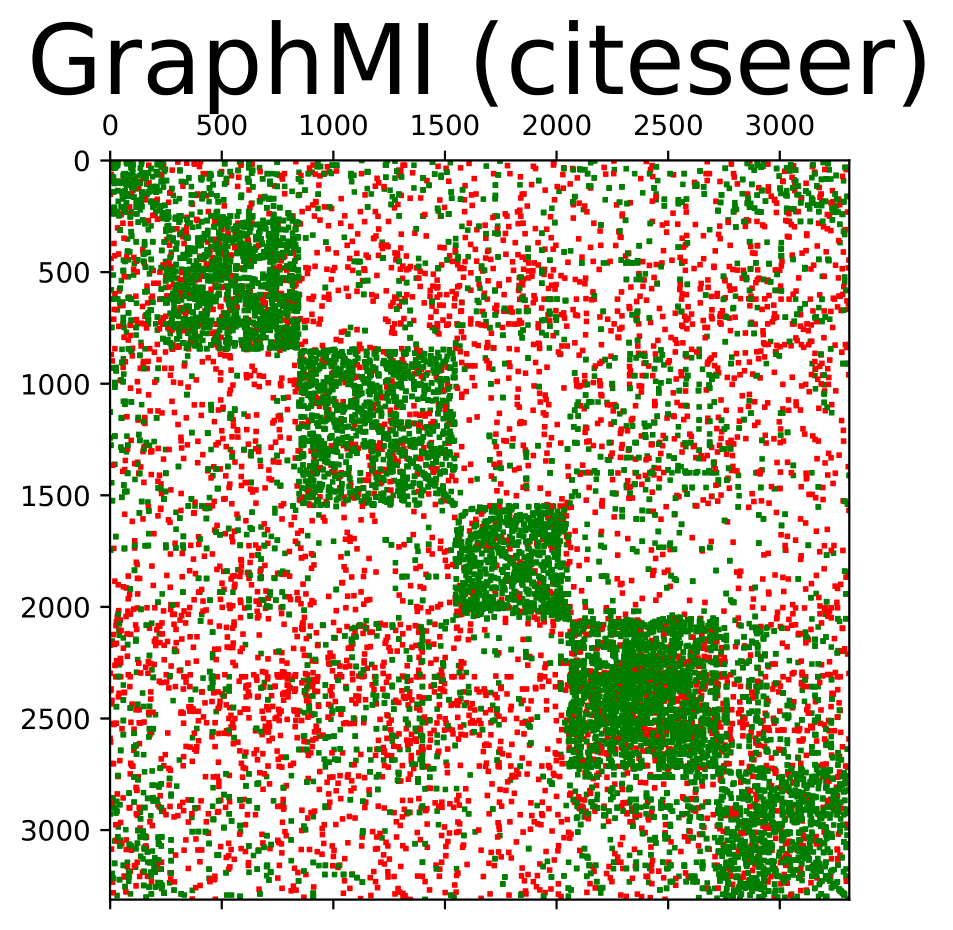}}}
		\hfill
		\subfigure[With protected GNN]
		{{\includegraphics[height=3.3cm]{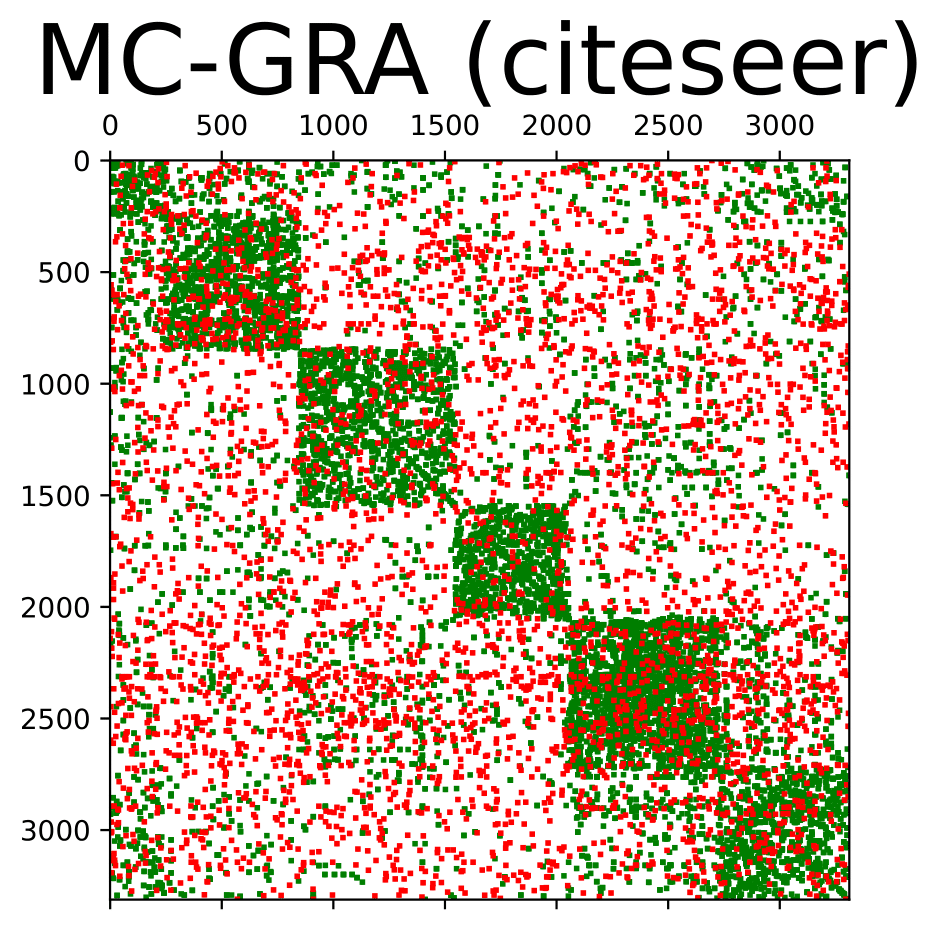}}}
		\hfill
		\subfigure[With protected GNN]
		{{\includegraphics[height=3.3cm]{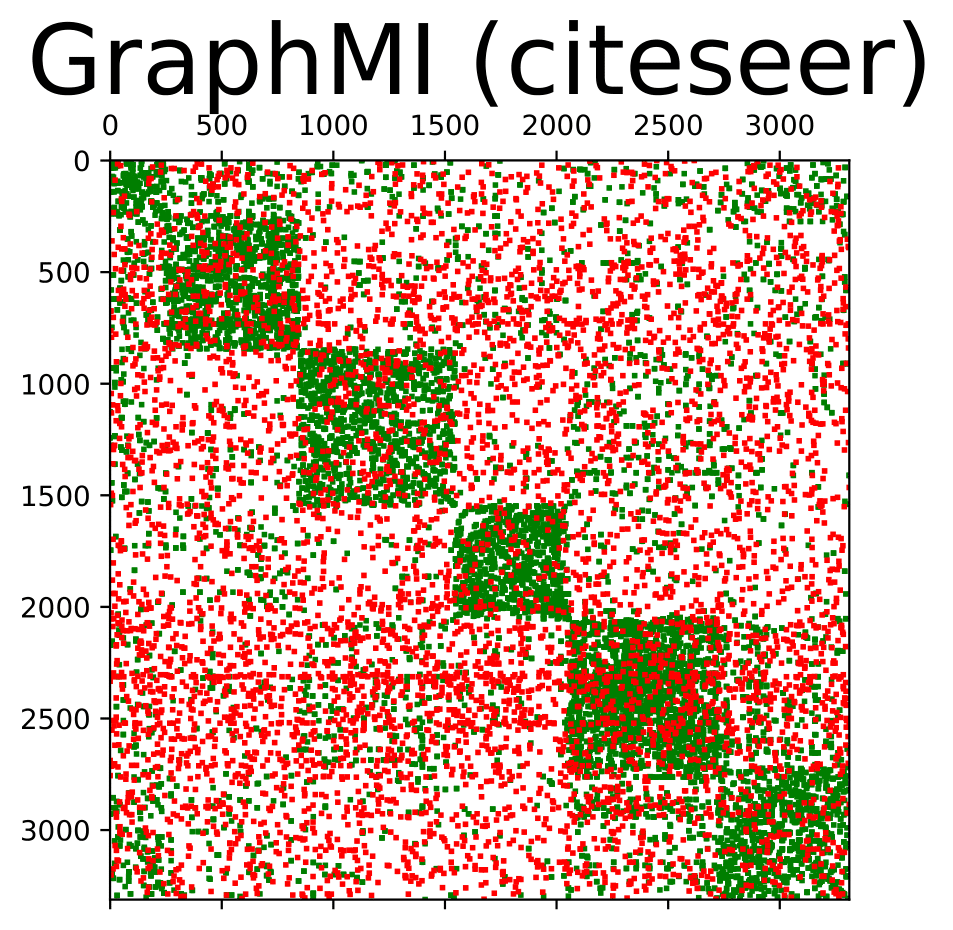}}}
		\vspace{-8pt}
		\caption{
			Recovered adjacency on Citeseer dataset.
		}
        \label{appd:adj:citeseer}
	\end{figure}
	
	\begin{figure}[H]
		\centering
		\subfigure[Ground truth]
		{{\includegraphics[height=3.3cm]{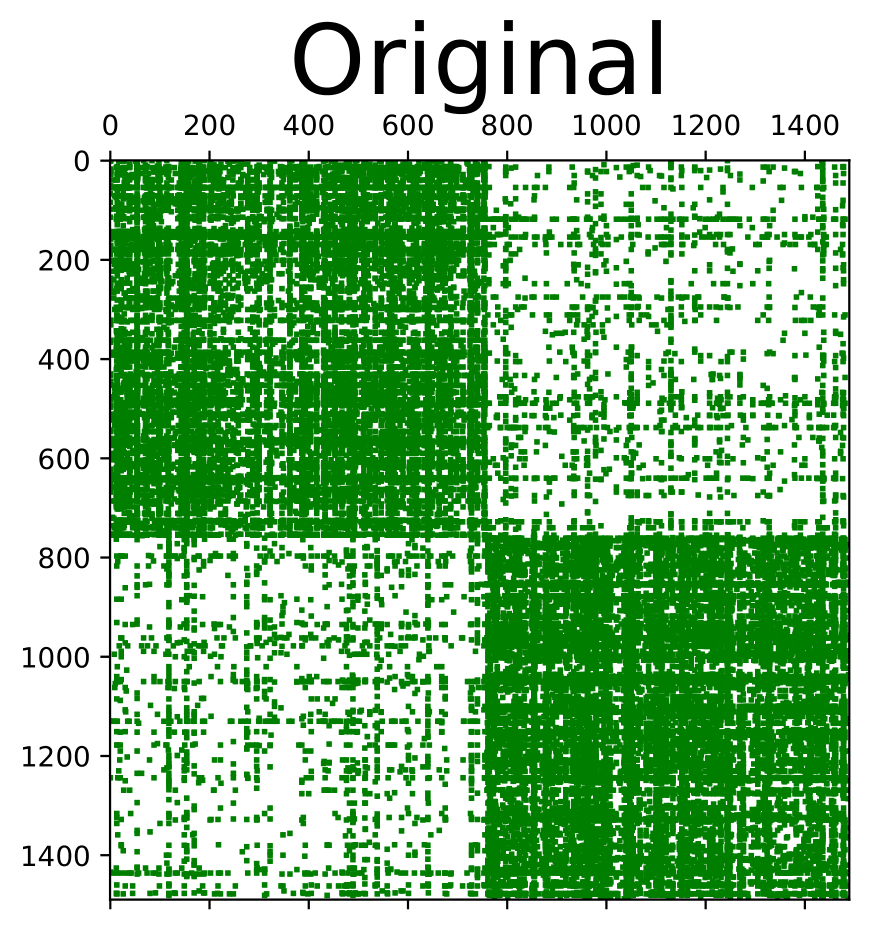}}}
		\hfill
		\subfigure[With normal GNN]
		{{\includegraphics[height=3.3cm]{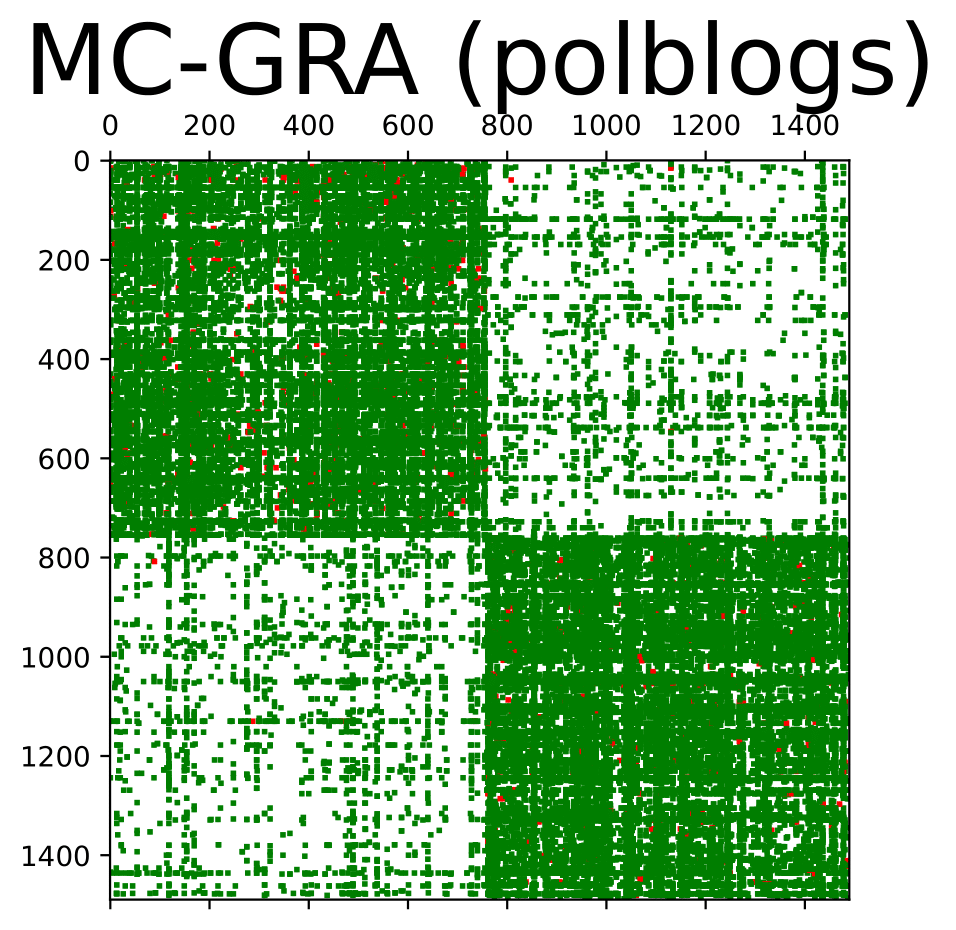}}}
		\hfill
		\subfigure[With normal GNN]
		{{\includegraphics[height=3.3cm]{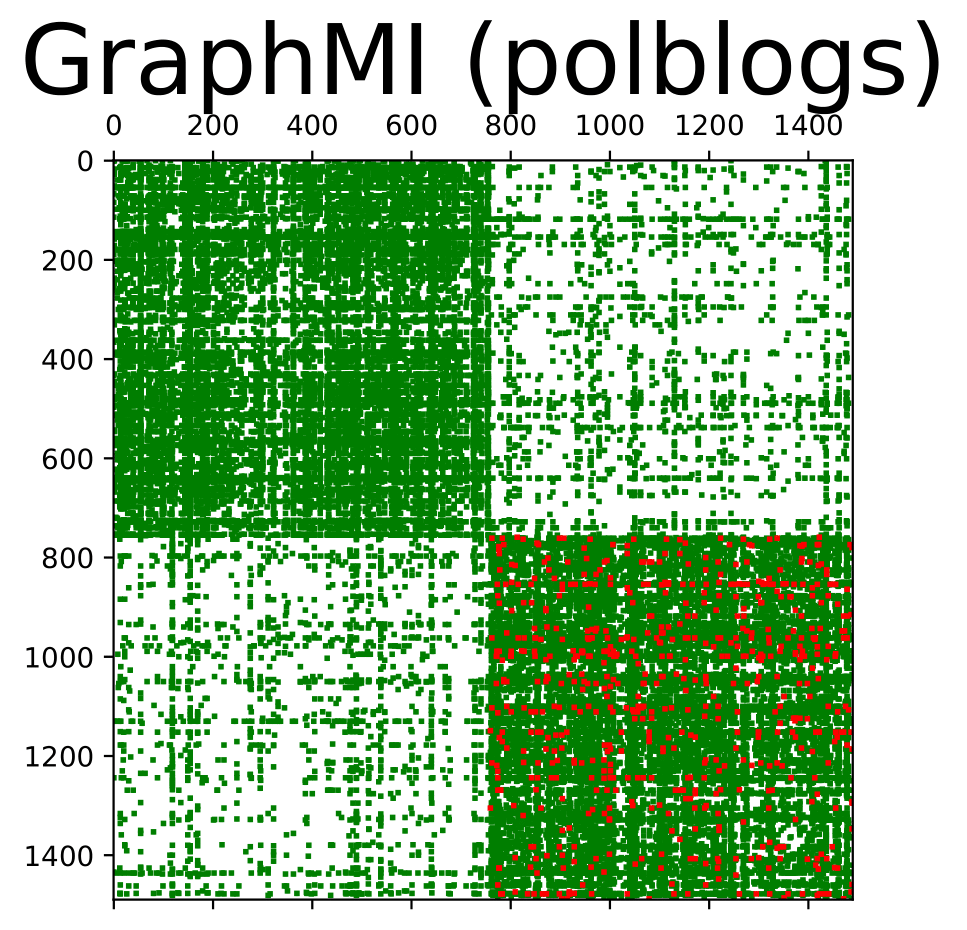}}}
		\hfill
		\subfigure[With protected GNN]
		{{\includegraphics[height=3.3cm]{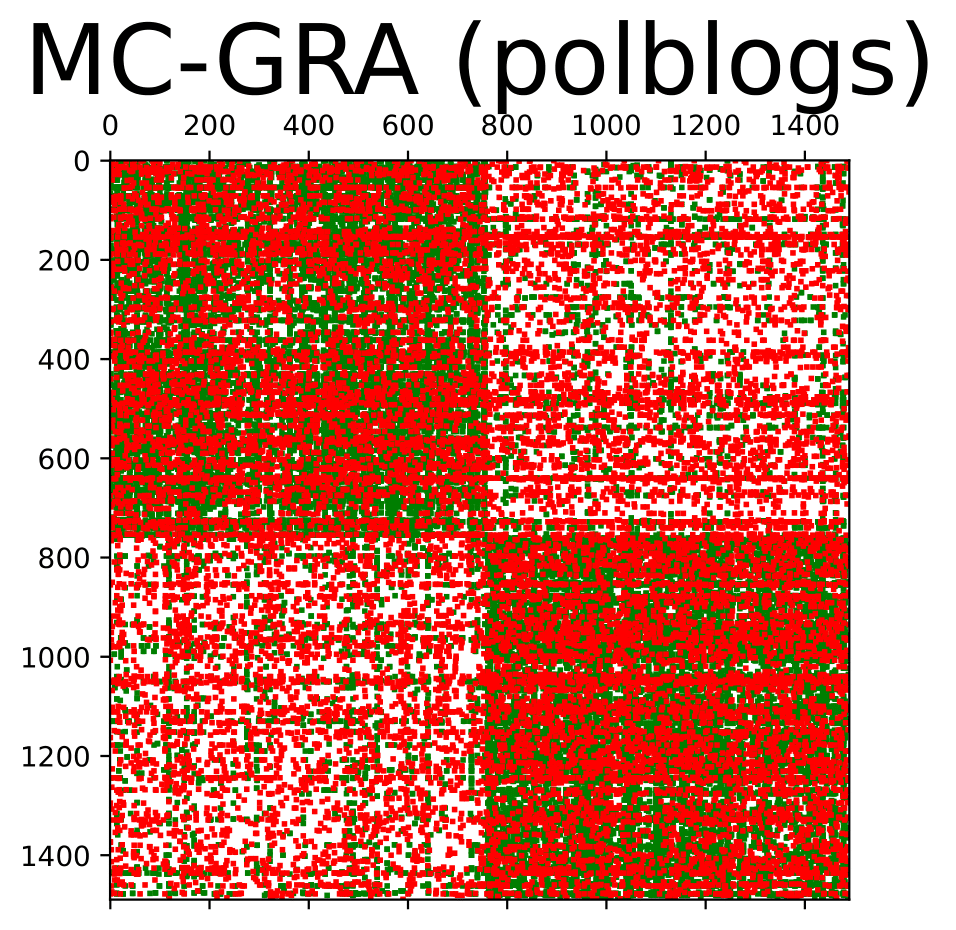}}}
		\hfill
		\subfigure[With protected GNN]
		{{\includegraphics[height=3.3cm]{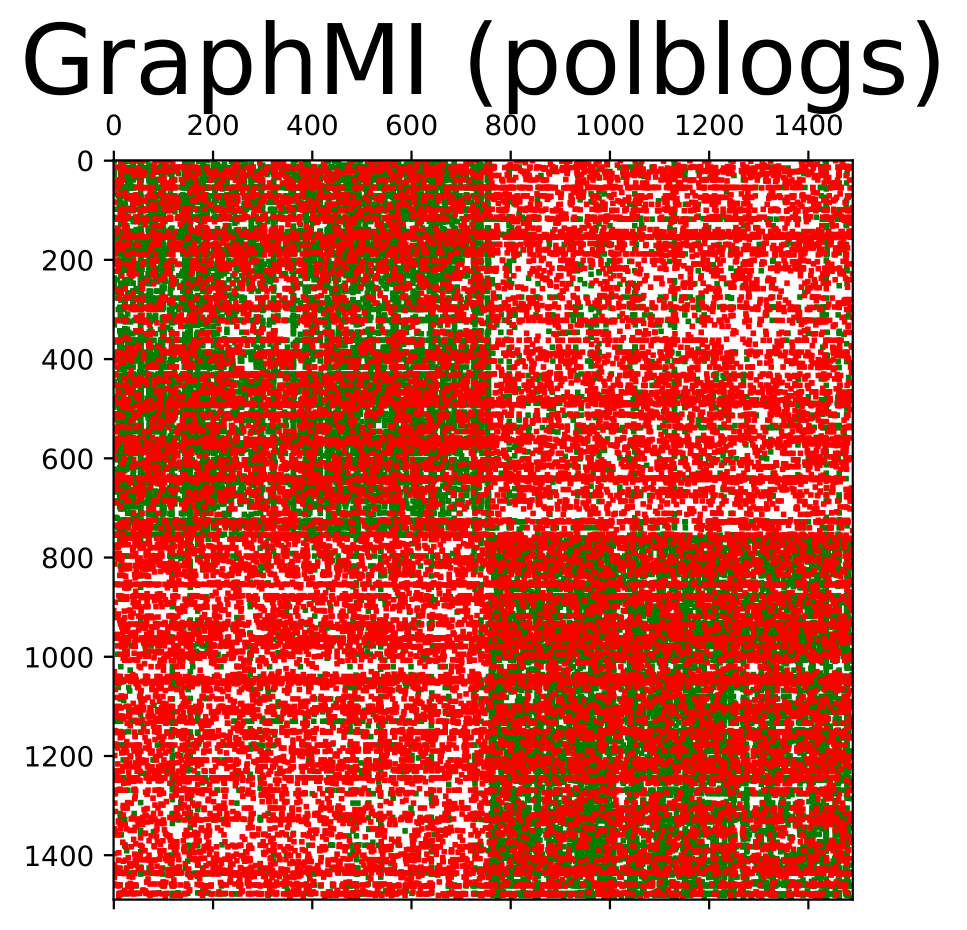}}}
		\vspace{-8pt}
		\caption{
			Recovered adjacency on Polblogs dataset.
		}
        \label{appd:adj:polblogs}
	\end{figure}
	
	\begin{figure}[H]
		\centering
		\subfigure[Ground truth]
		{{\includegraphics[height=3.3cm]{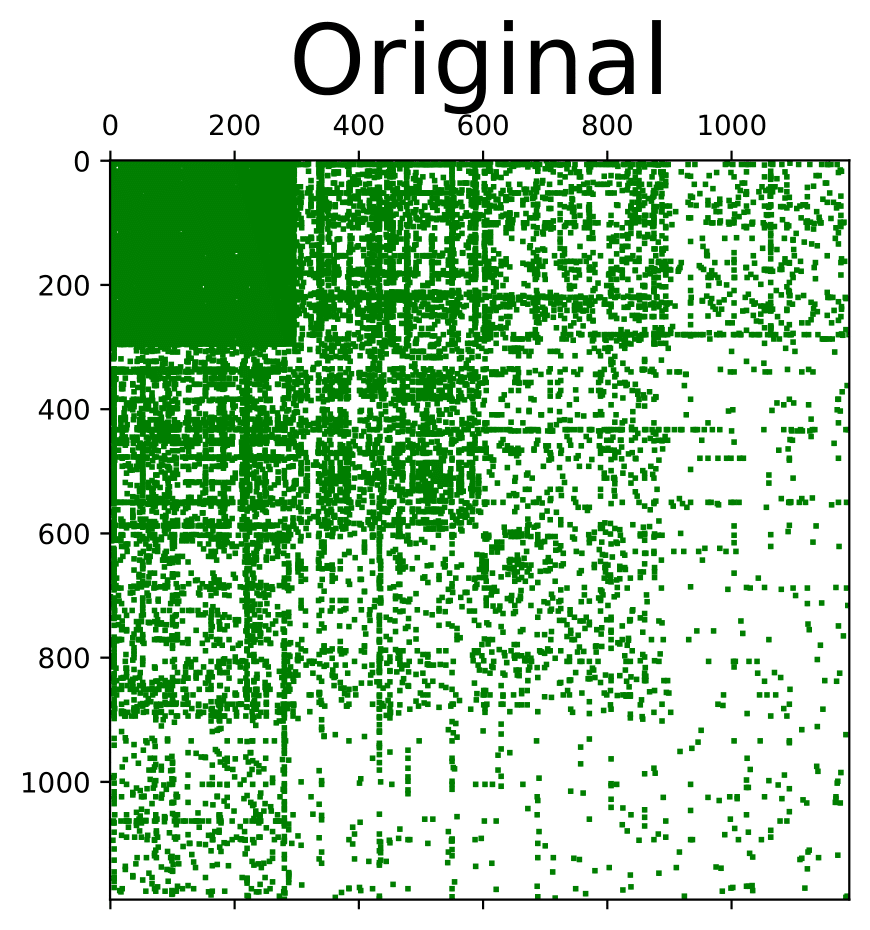}}}
		\hfill
		\subfigure[With normal GNN]
		{{\includegraphics[height=3.3cm]{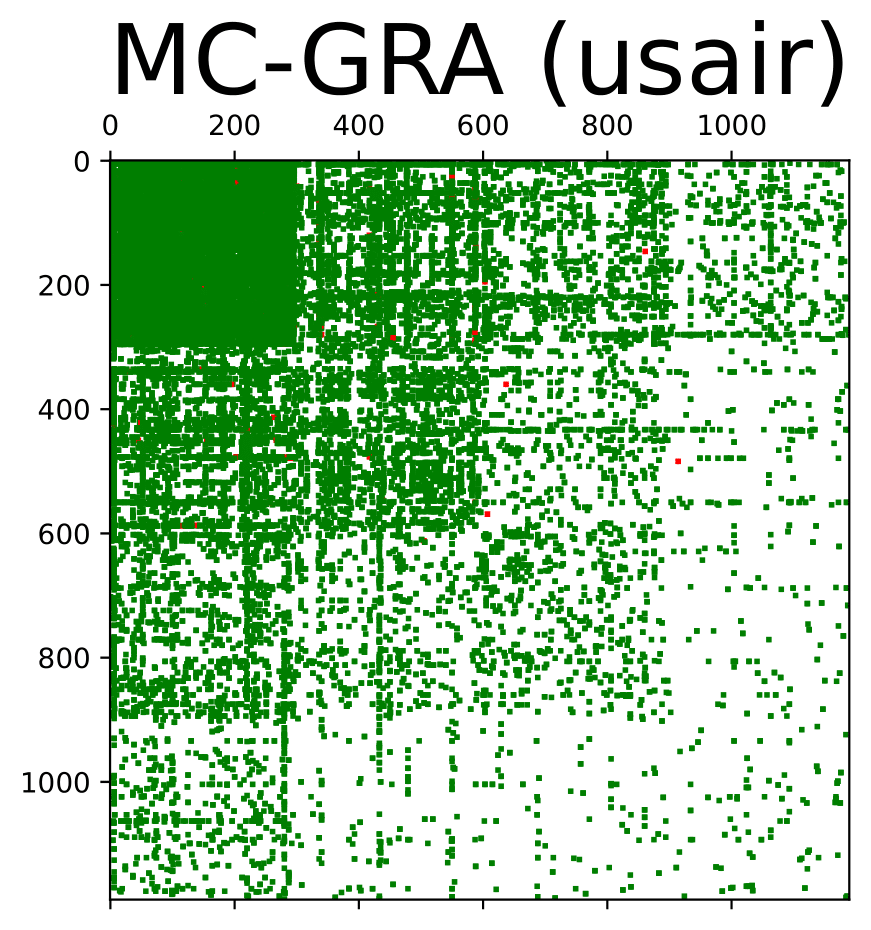}}}
		\hfill
		\subfigure[With normal GNN]
		{{\includegraphics[height=3.3cm]{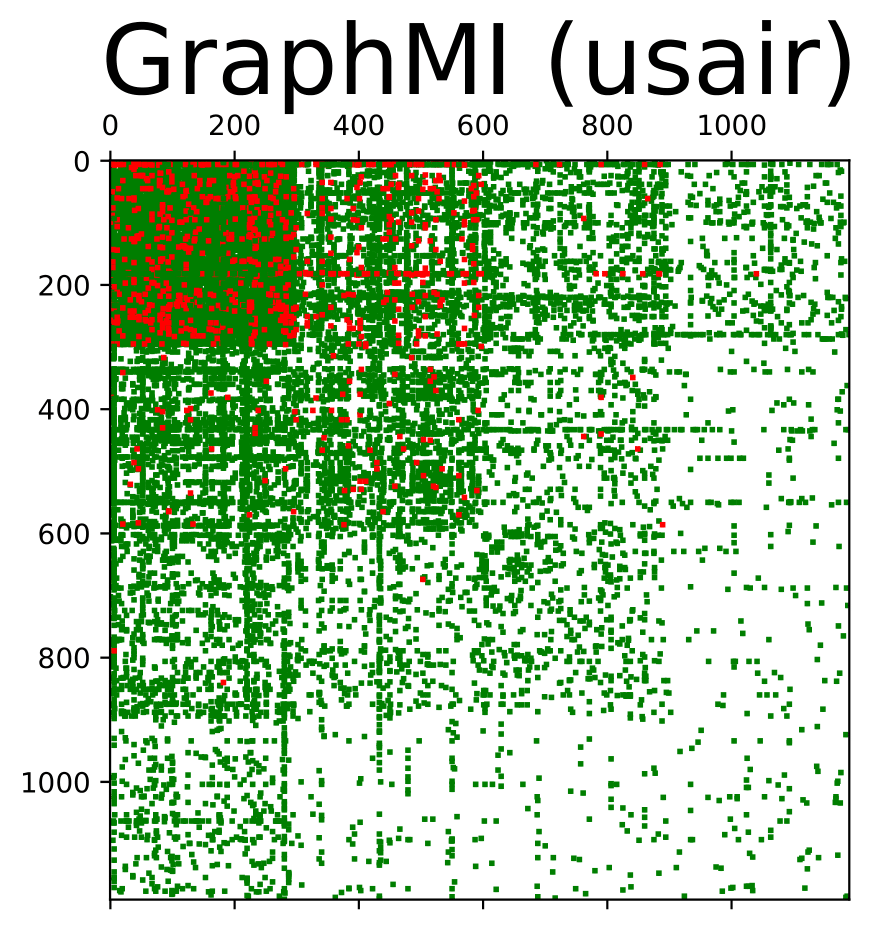}}}
		\hfill
		\subfigure[With protected GNN]
		{{\includegraphics[height=3.3cm]{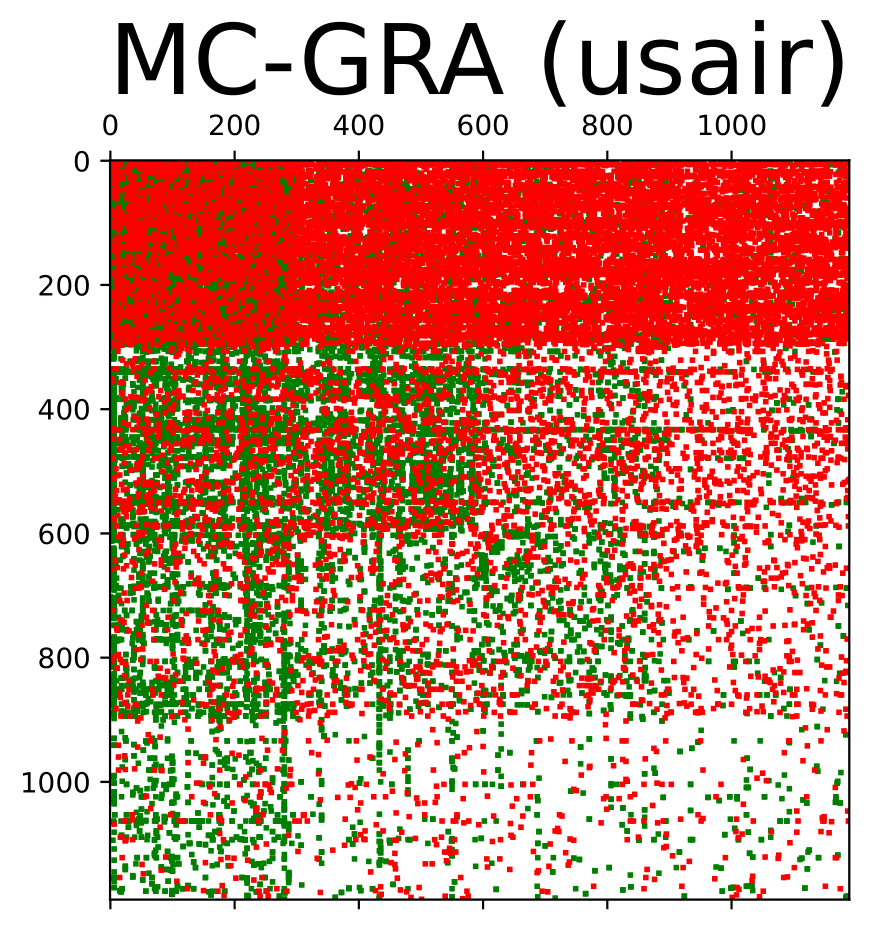}}}
		\hfill
		\subfigure[With protected GNN]
		{{\includegraphics[height=3.3cm]{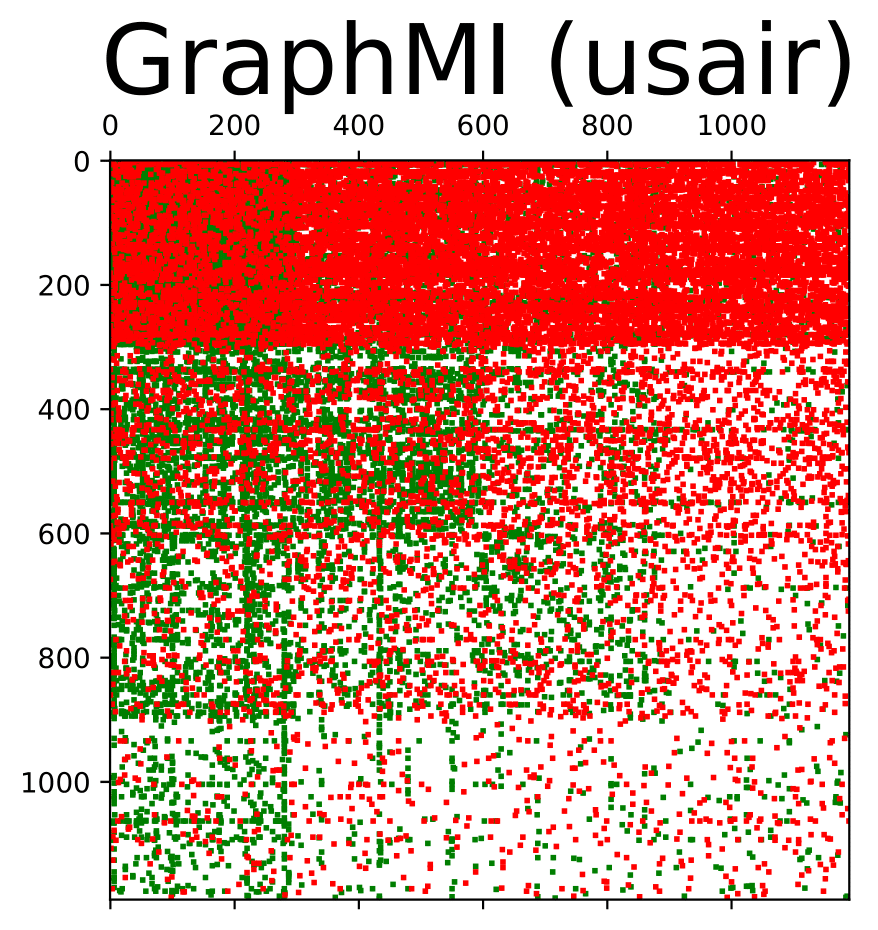}}}
		\vspace{-8pt}
		\caption{
			Recovered adjacency on USA dataset.
		}
	\end{figure}

	\begin{figure}[H]
		\centering
		\subfigure[Ground truth]
		{{\includegraphics[height=3.3cm]{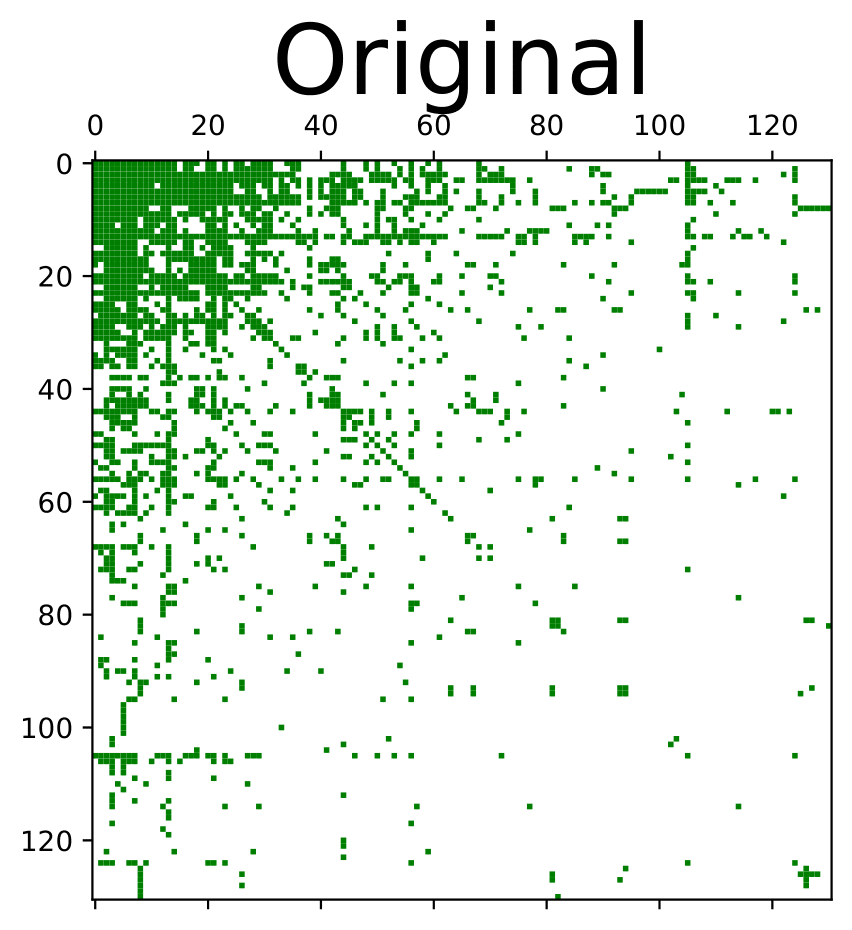}}}
		\hfill
		\subfigure[With normal GNN]
		{{\includegraphics[height=3.3cm]{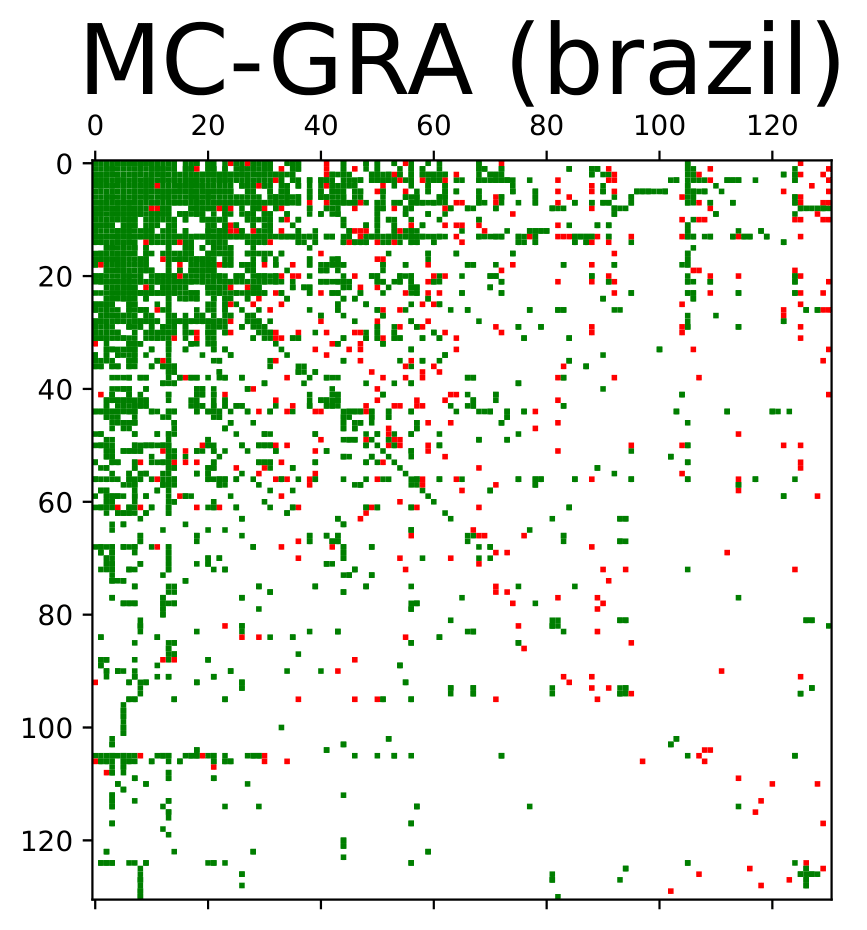}}}
		\hfill
		\subfigure[With normal GNN]
		{{\includegraphics[height=3.3cm]{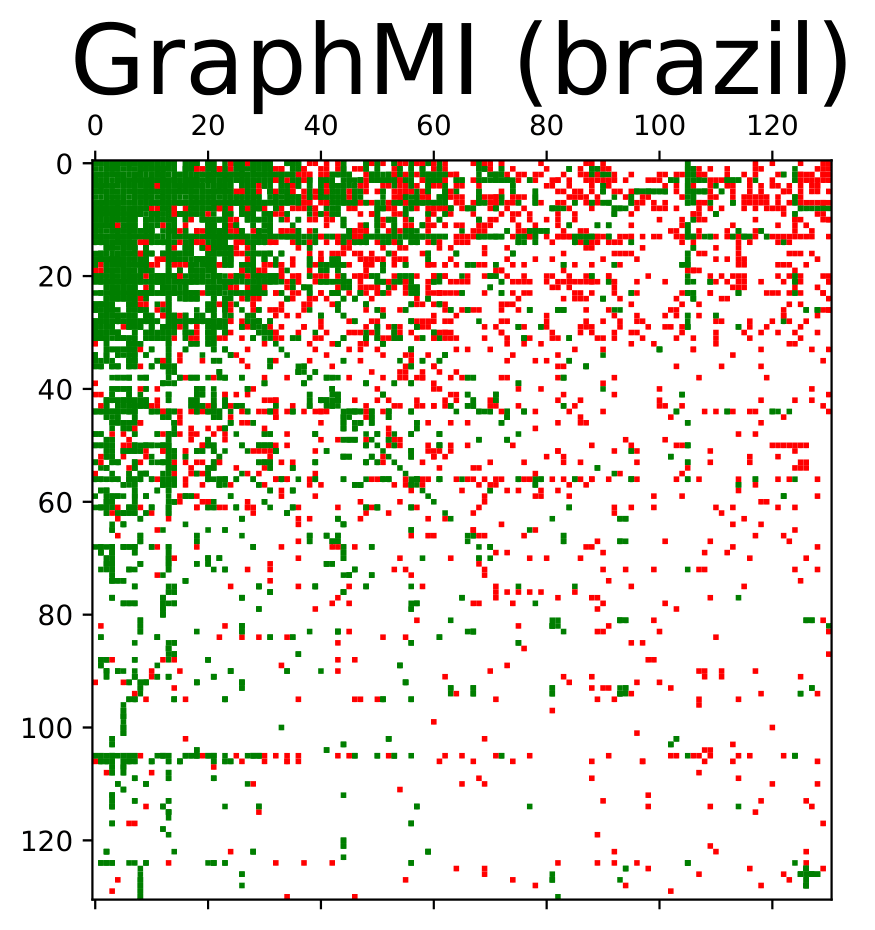}}}
		\hfill
		\subfigure[With protected GNN]
		{{\includegraphics[height=3.3cm]{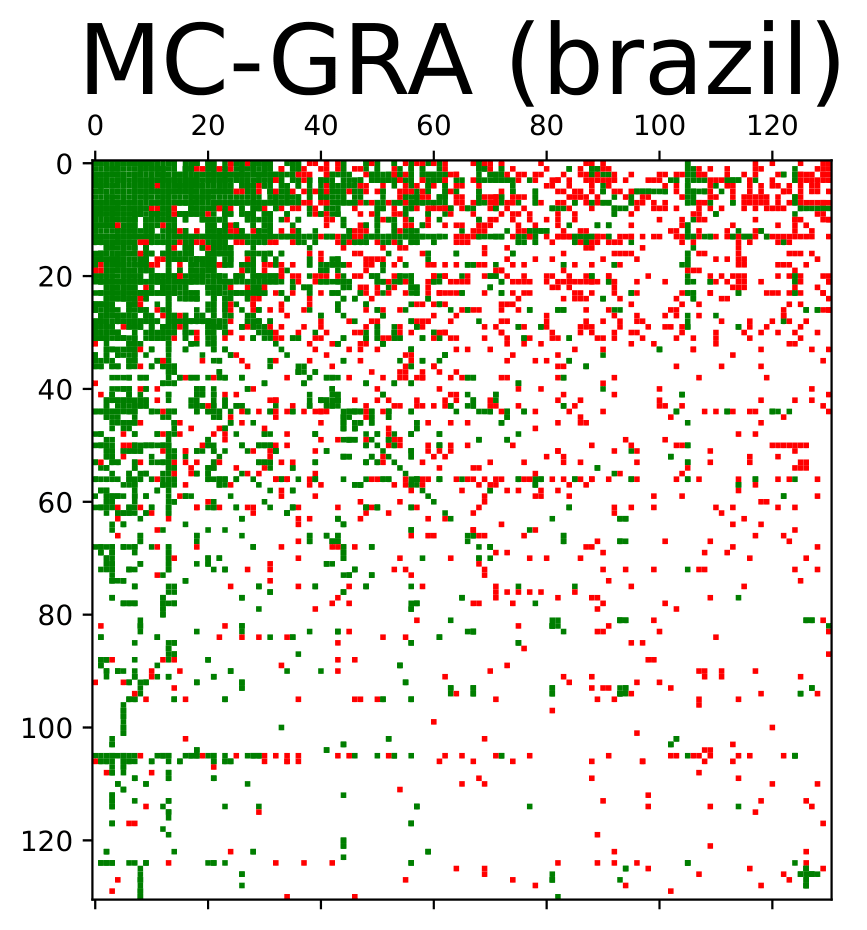}}}
		\hfill
		\subfigure[With protected GNN]
		{{\includegraphics[height=3.3cm]{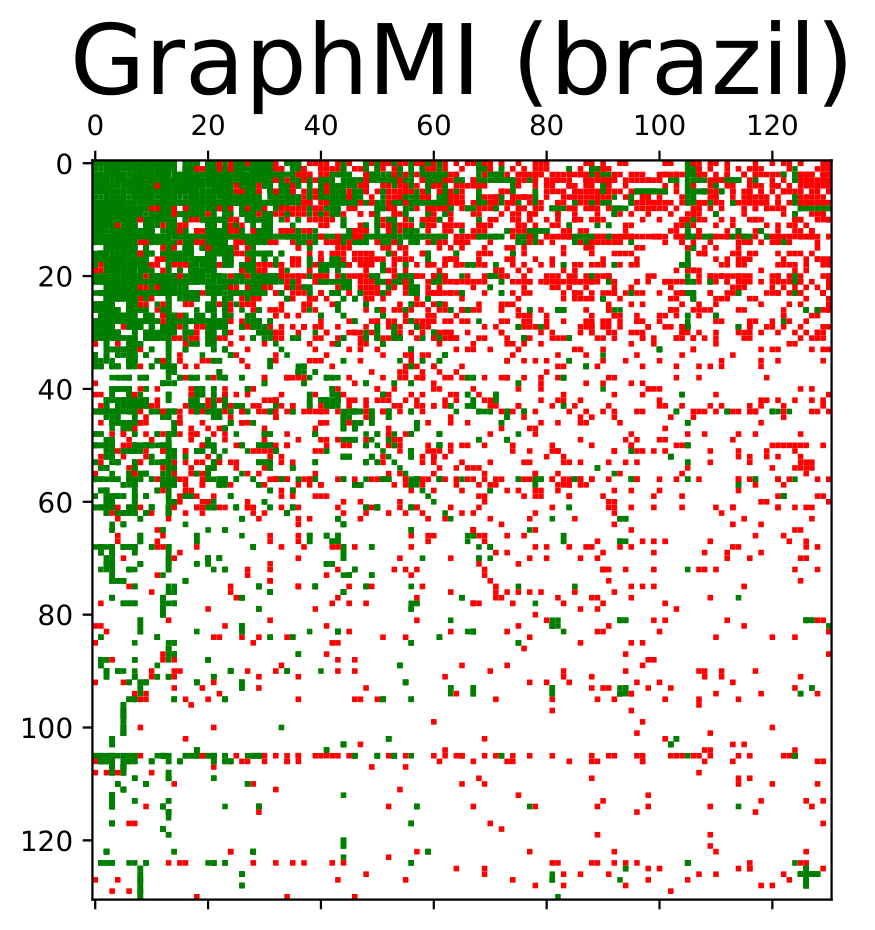}}}
		\vspace{-8pt}
		\caption{
			Recovered adjacency on Brazil dataset.
		}
	\end{figure}
	
	\begin{figure}[H]
		\centering
		\subfigure[Ground truth]
		{{\includegraphics[height=3.3cm]{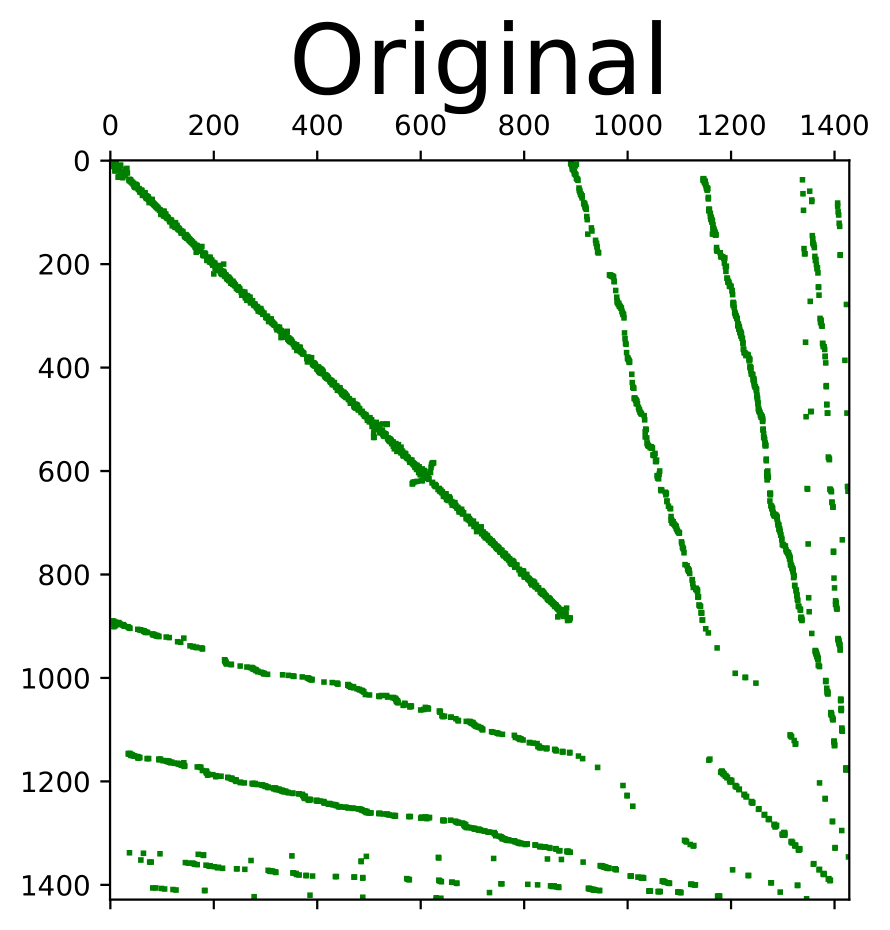}}}
		\hfill
		\subfigure[With normal GNN]
		{{\includegraphics[height=3.3cm]{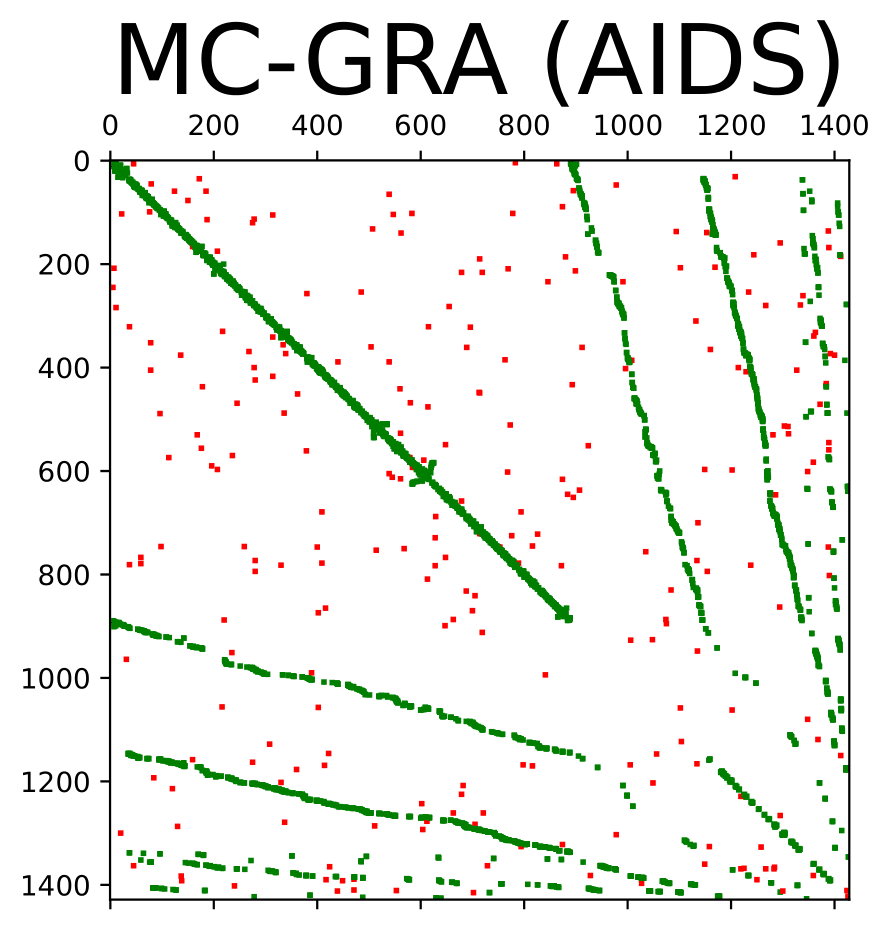}}}
		\hfill
		\subfigure[With normal GNN]
		{{\includegraphics[height=3.3cm]{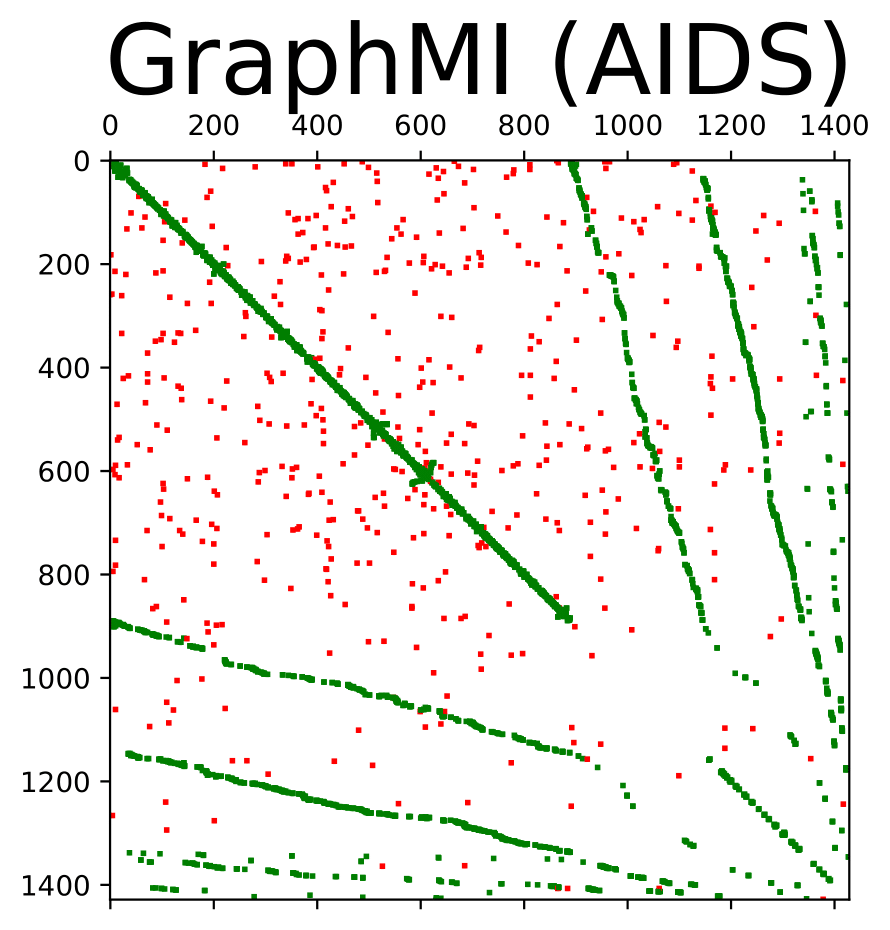}}}
		\hfill
		\subfigure[With protected GNN]
		{{\includegraphics[height=3.3cm]{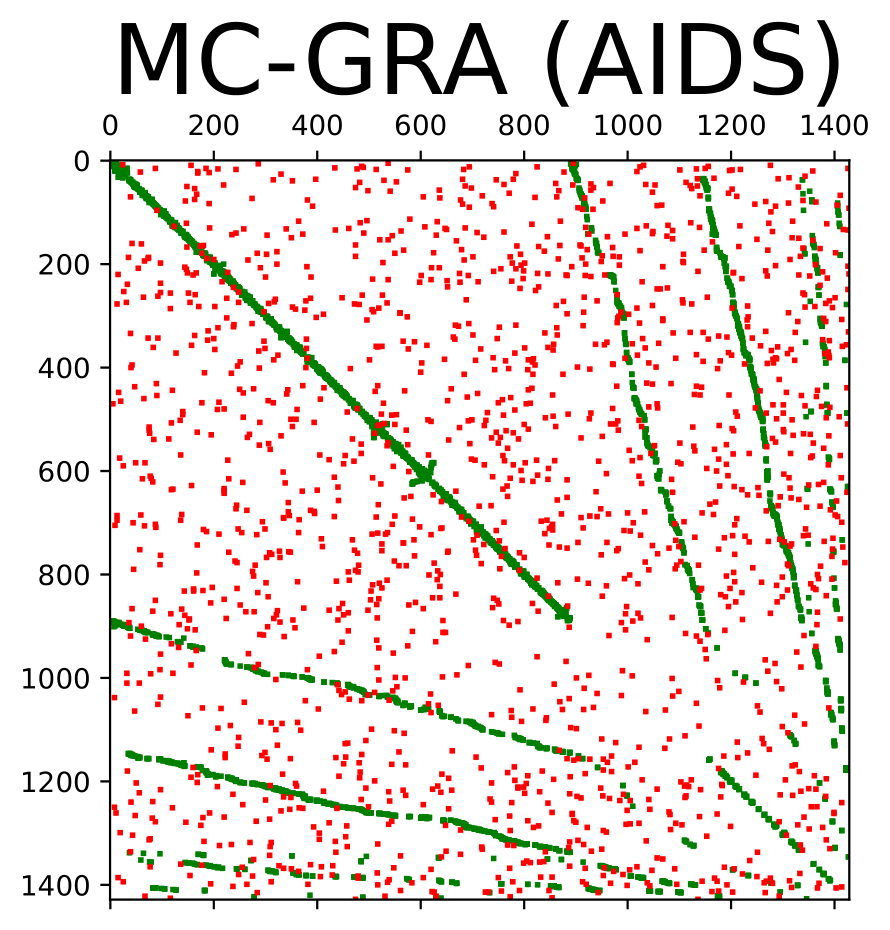}}}
		\hfill
		\subfigure[With protected GNN]
		{{\includegraphics[height=3.3cm]{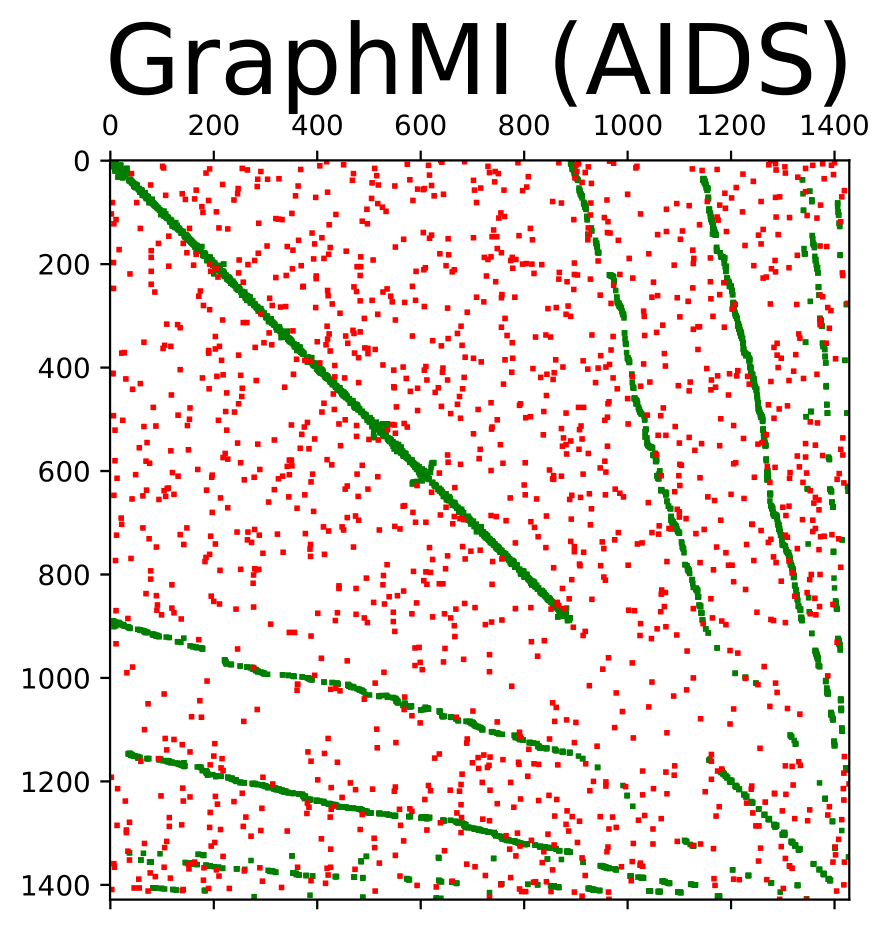}}}
		\vspace{-8pt}
		\caption{
			Recovered adjacency on AIDS dataset.
		}
        \label{appd:adj:AIDS}
	\end{figure}

	\textbf{Tracking the MI terms.}
    We show the learning curves of MC-GRA and MC-GPB on each dataset as follows.    
     
     For MC-GRA~(~Fig.~\ref{appd:curve:mc-gra}), most of the output and propagation loss converged to near zero, showing that the model efficiently approximates the original Markov chain. 
     %
     %
     For MC-GPB, we track three constraints layer-wise and average them out to visualize the overall trend. We also record the model accuracy constraint, \textit{i.e.}, the cross-entropy of model output, to check whether the layer and the full model have consistent patterns.
     Both privacy and complexity constraints, despite the fluctuation of the former in some datasets, show a downward trend throughout the training, especially for the usair dataset. The accuracy curves also have similar patterns. 
     
	\begin{figure}[H]
		\centering
		\hfill
		\includegraphics[width=5.5cm]{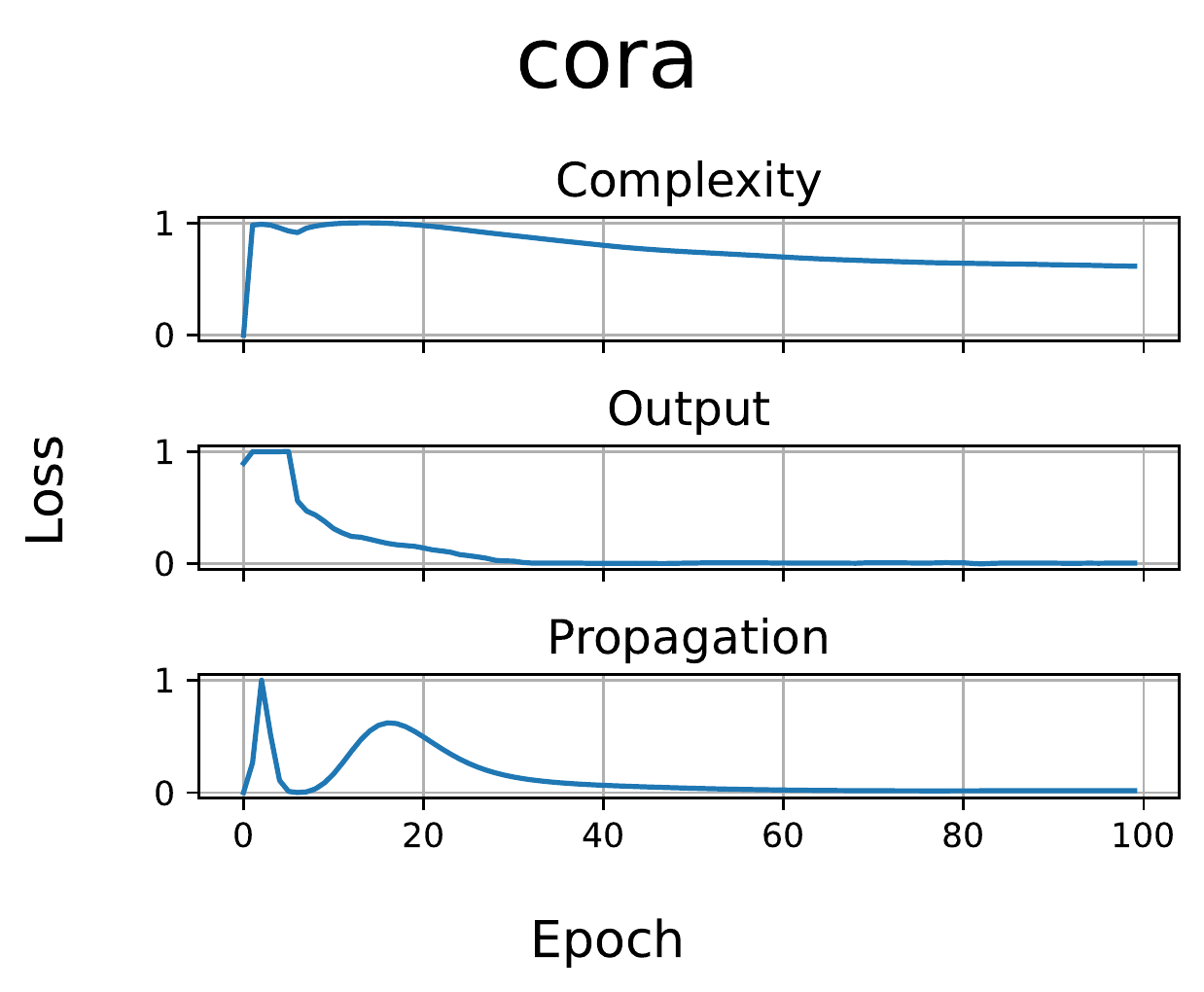}
		\hfill
		\includegraphics[width=5.5cm]{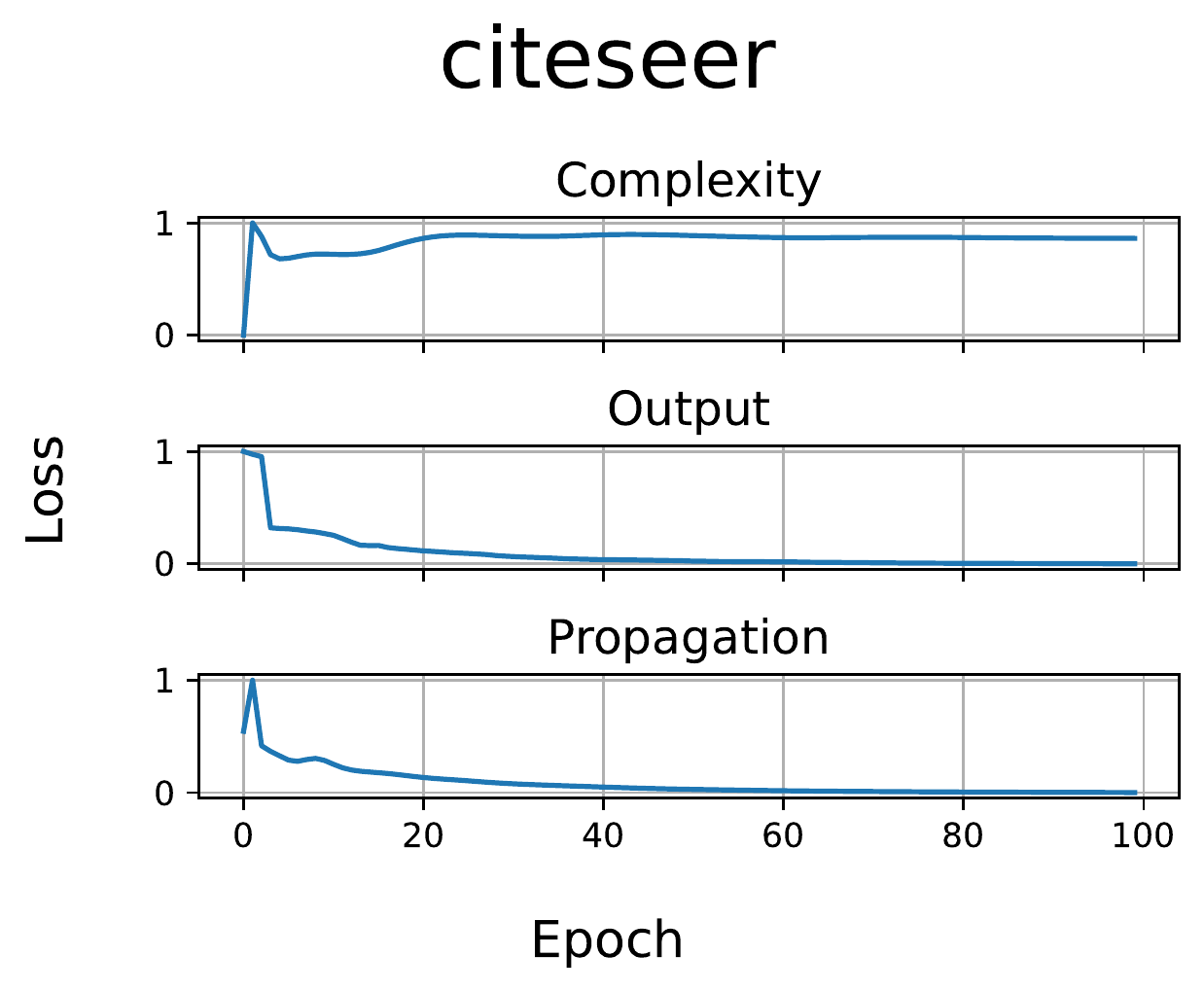}
		\hfill
		\includegraphics[width=5.5cm]{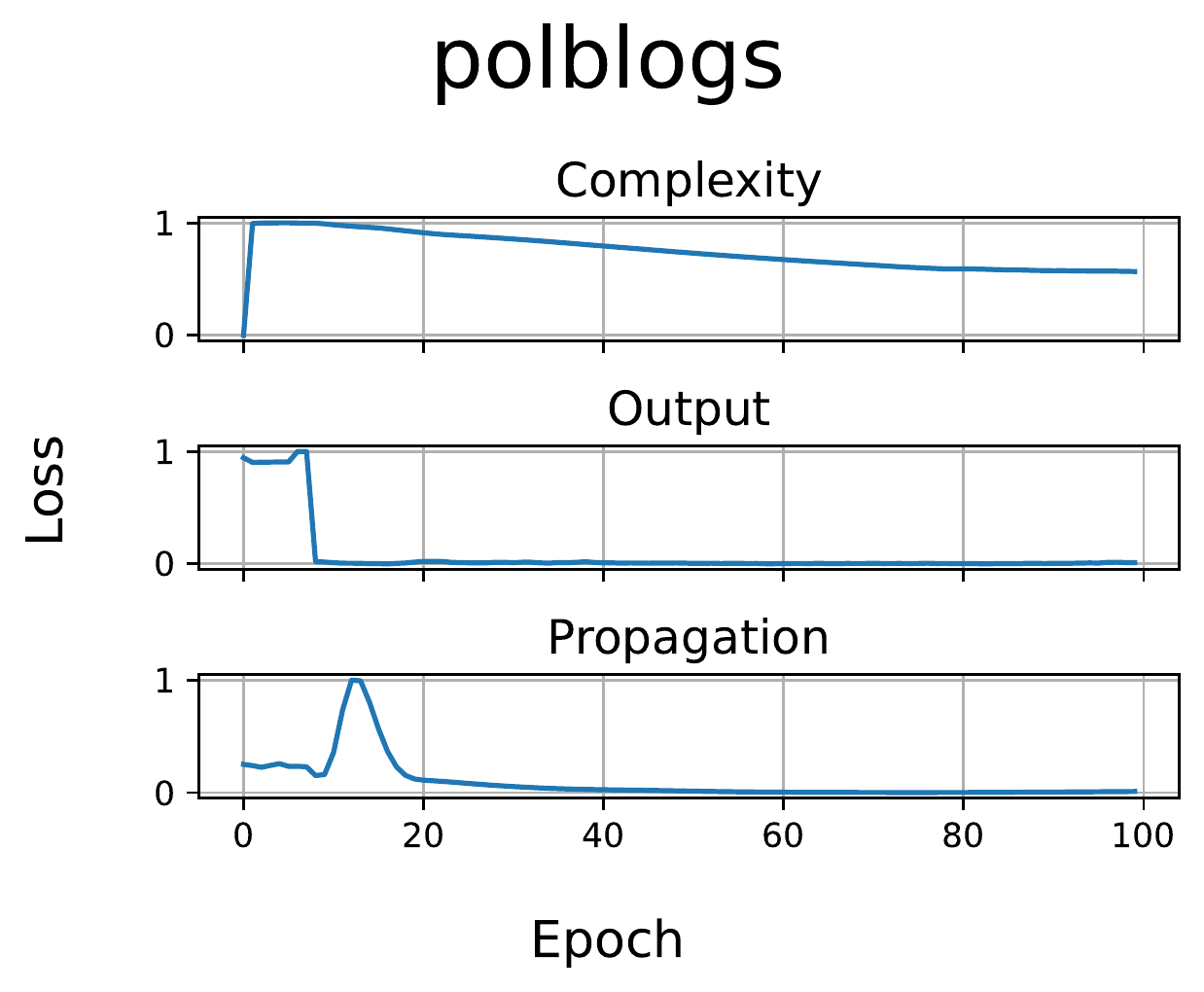}
		\hfill
		\\
		\vspace{0.3cm}
		\hfill
		\includegraphics[width=5.5cm]{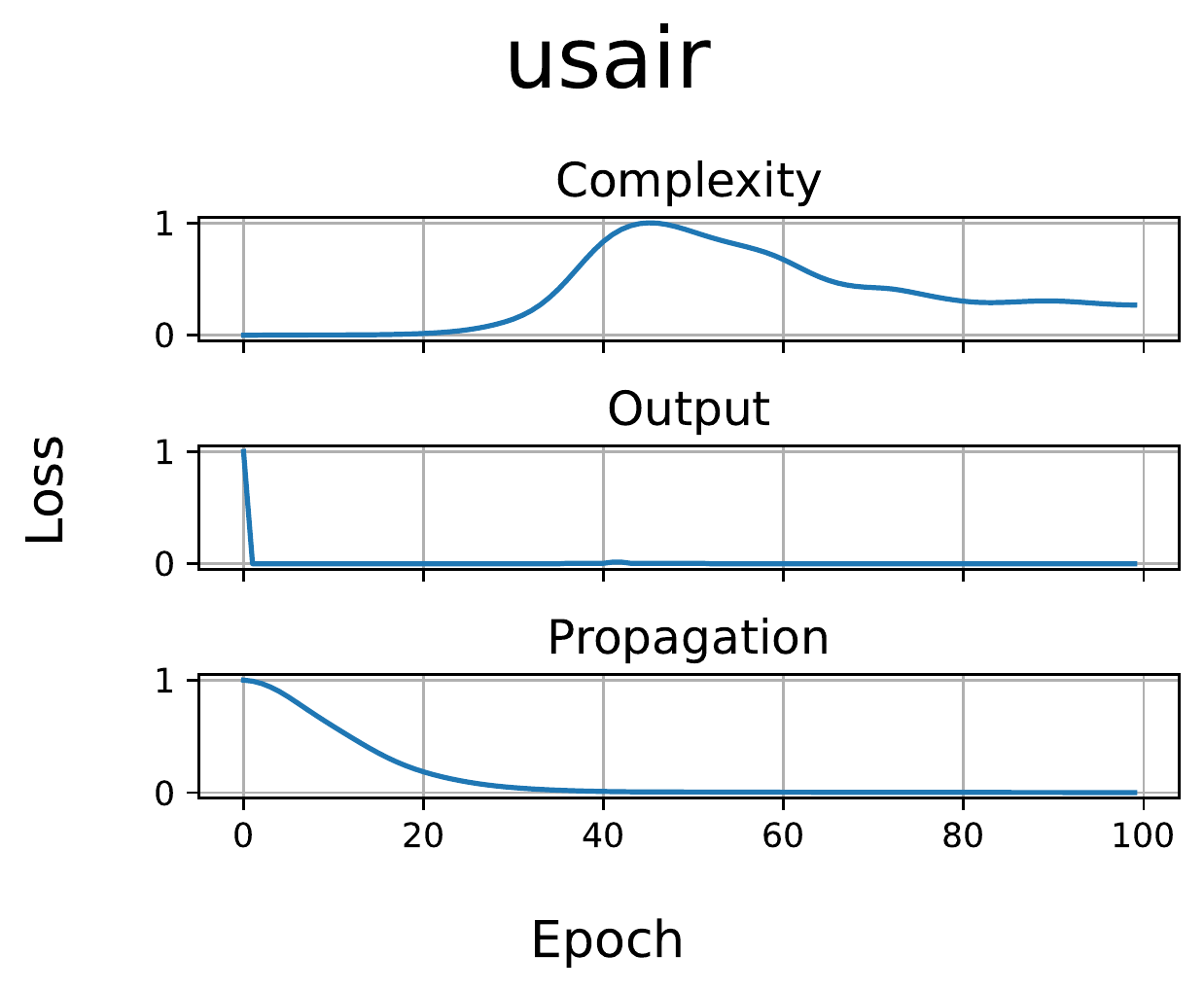}
		\hfill
		\includegraphics[width=5.5cm]{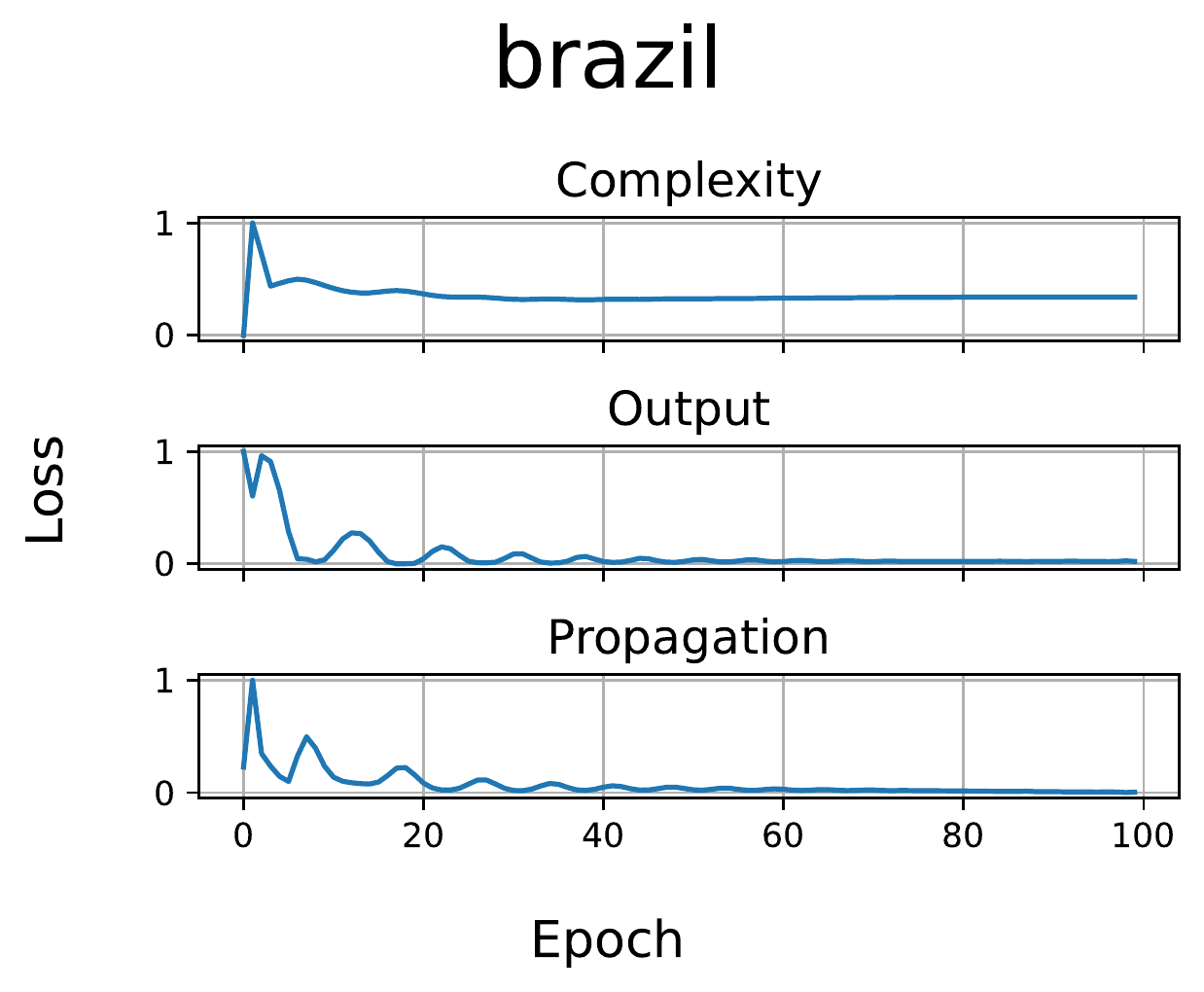}
		\hfill
		\includegraphics[width=5.5cm]{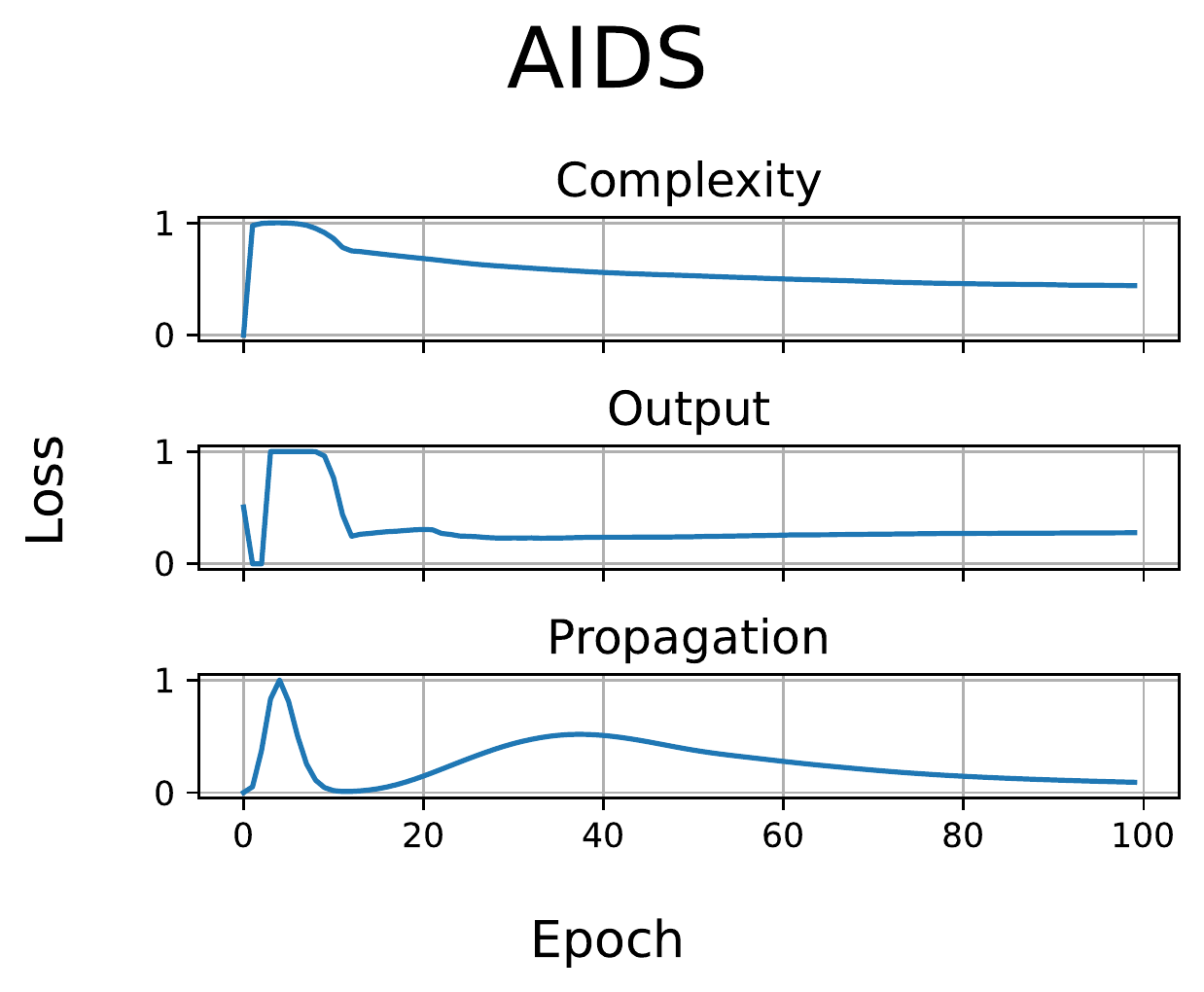}
		\hfill
		\vspace{-6px}
		\caption{
			Training curves of MC-GRA on each dataset.
		}
            \label{appd:curve:mc-gra}
	\end{figure}
	
	\begin{figure}[H]
        \vspace{-20pt}
		\centering
		\hfill
		\includegraphics[width=5.5cm]{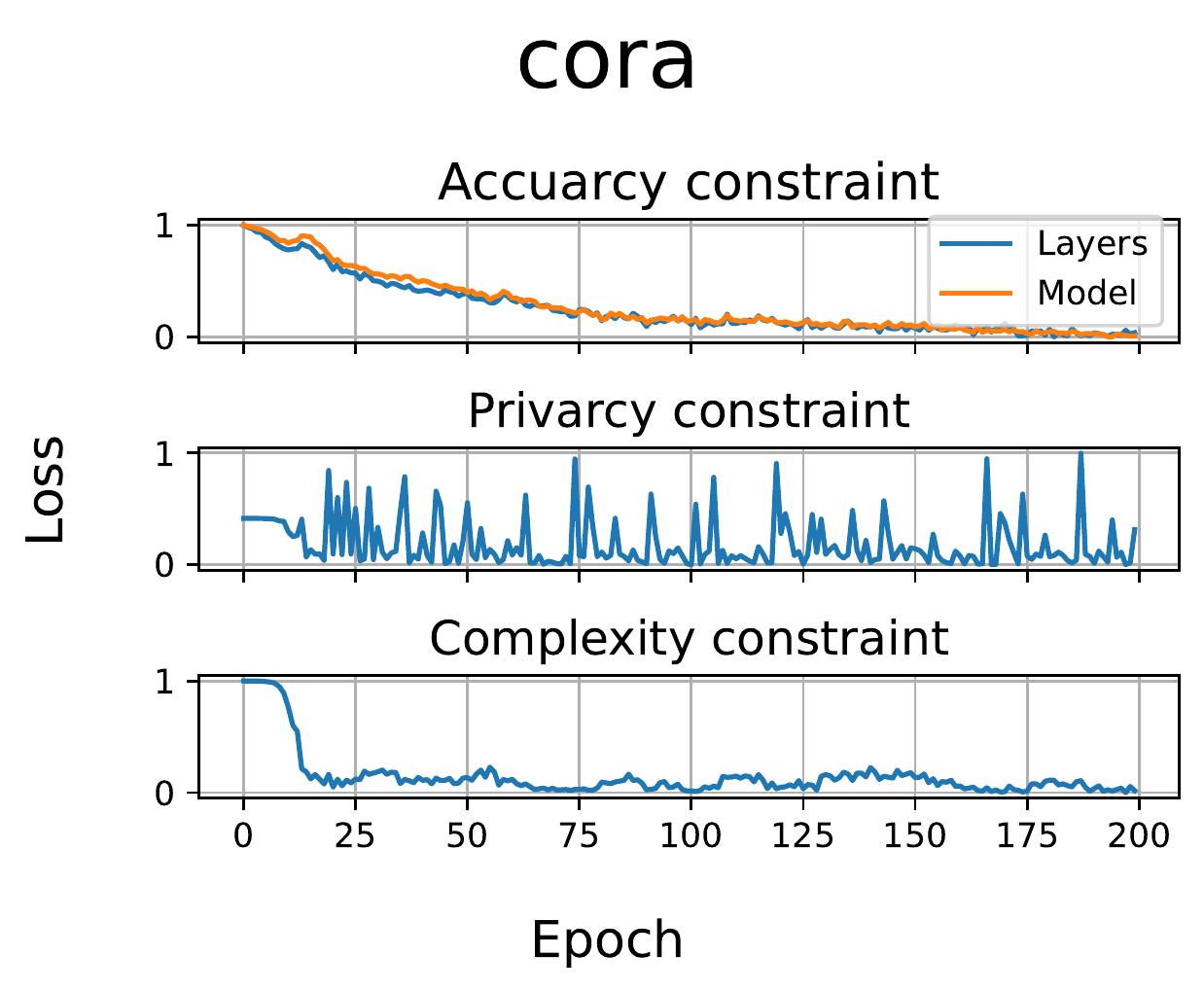}
		\hfill
		\includegraphics[width=5.5cm]{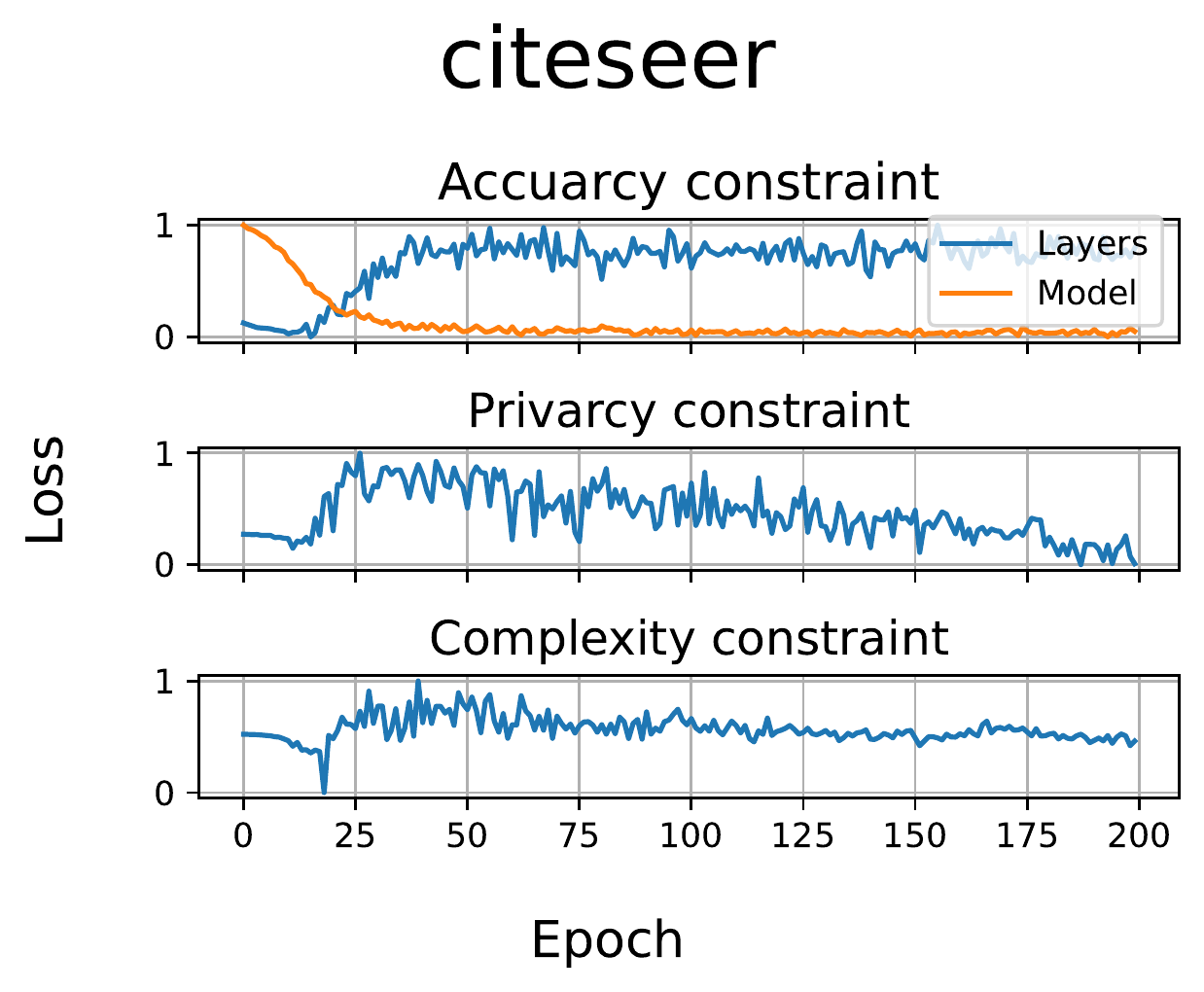}
		\hfill
		\includegraphics[width=5.5cm]{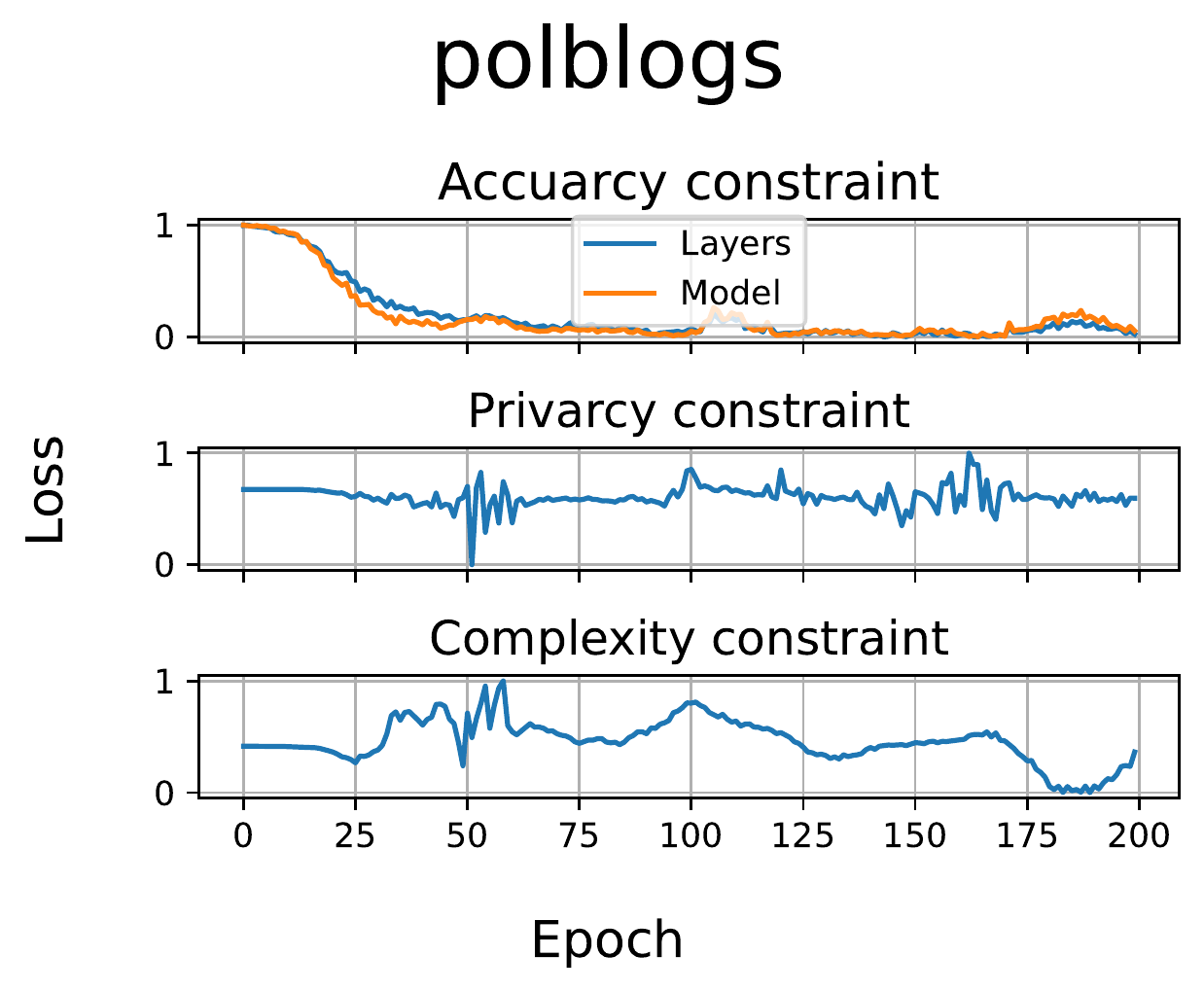}
		\hfill
		\\
		\vspace{0.3cm}
		\hfill
		\includegraphics[width=5.5cm]{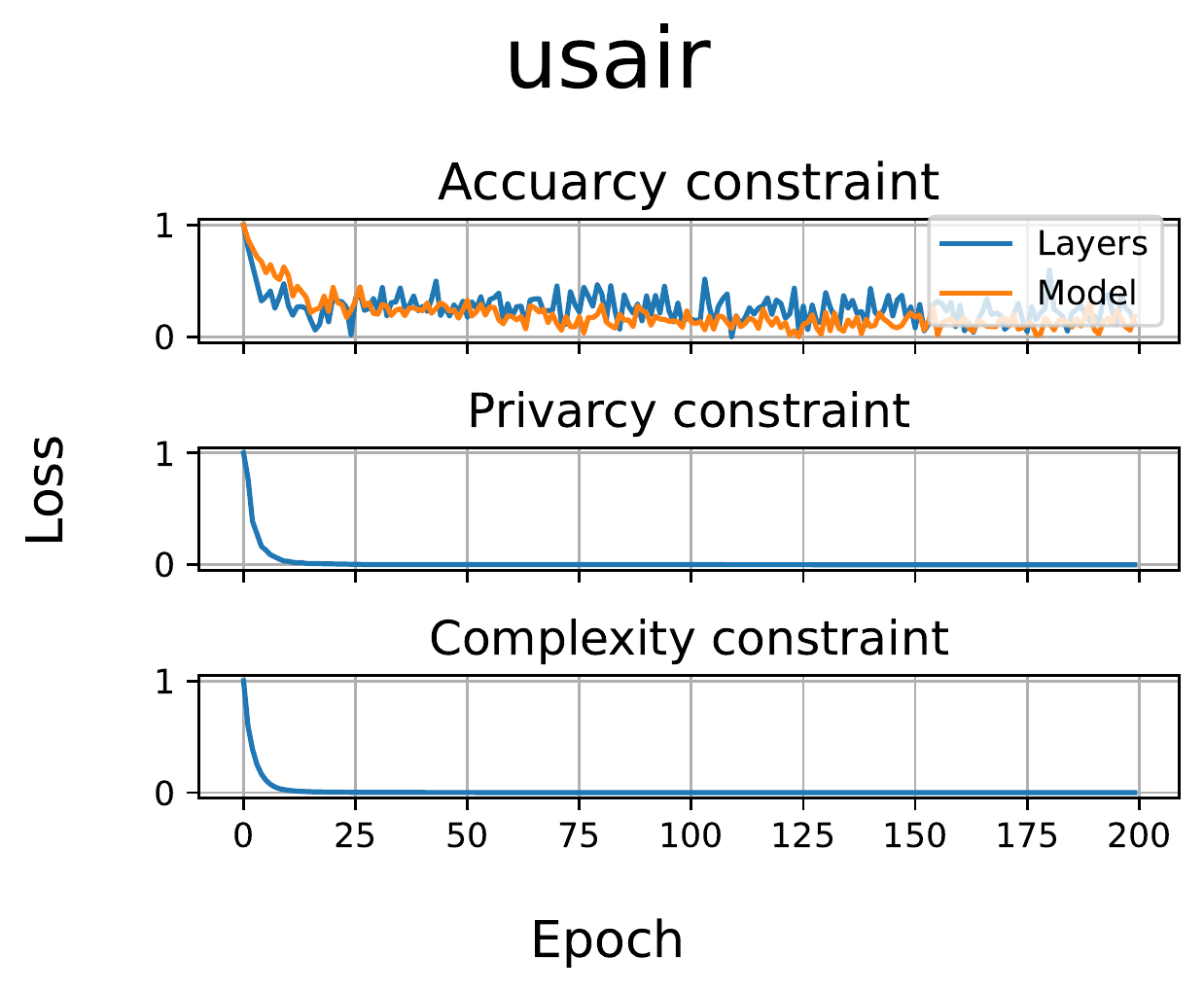}
		\hfill
		\includegraphics[width=5.5cm]{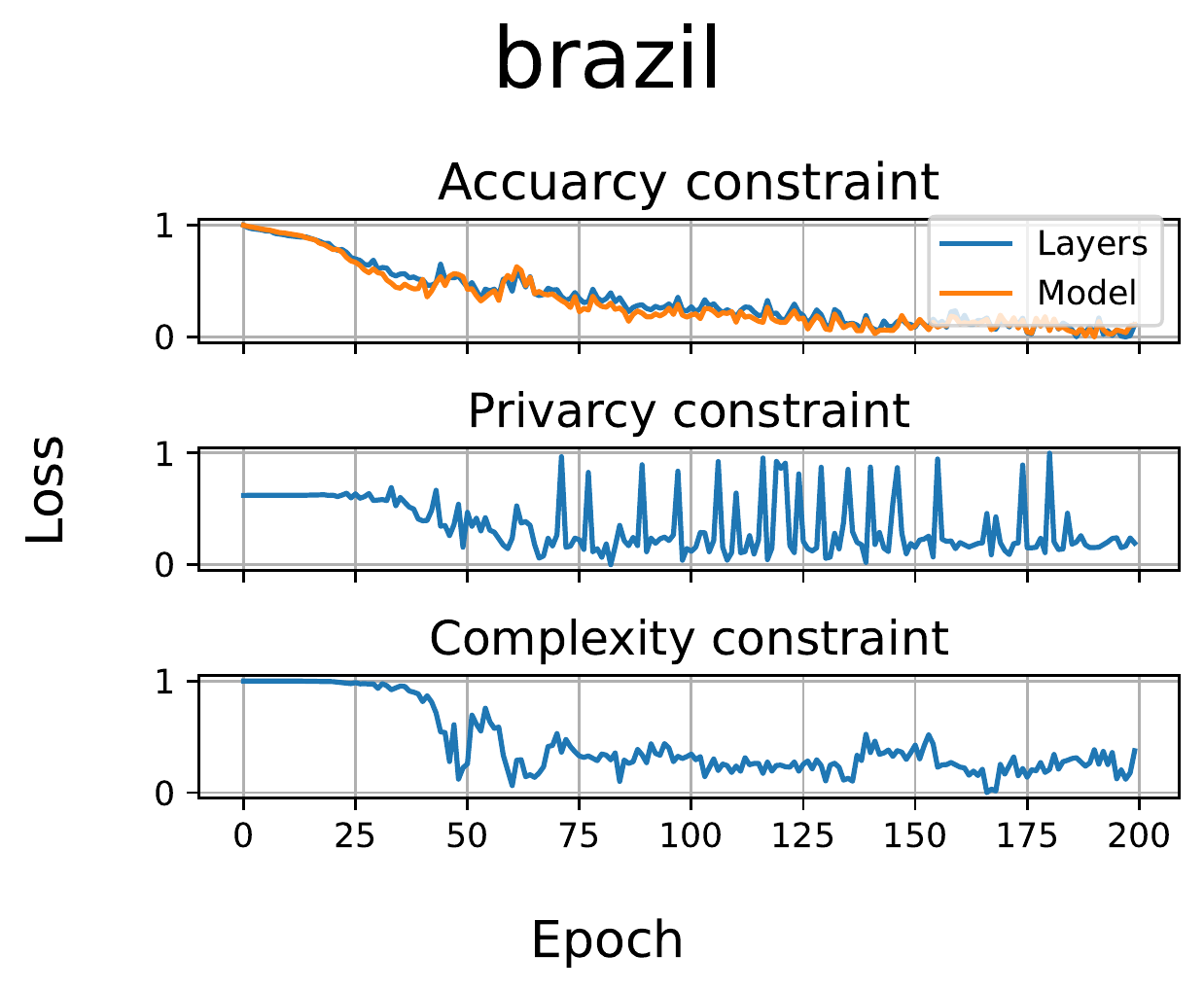}
		\hfill
		\includegraphics[width=5.5cm]{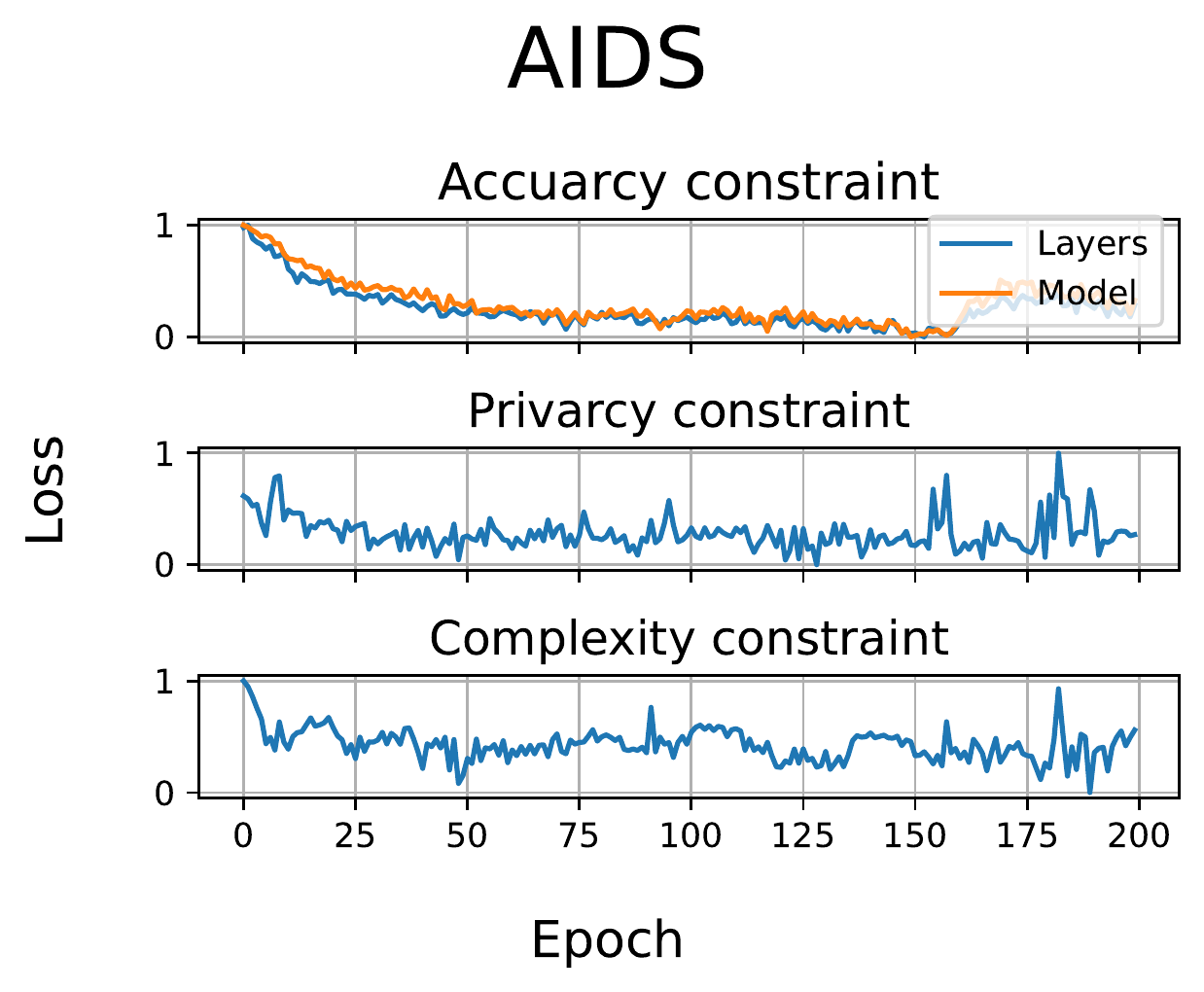}
		\hfill
		\vspace{-6px}
		\caption{
			Training curves of MC-GPB on each dataset.
		}
	\end{figure}

	\textbf{A further analysis with the training dynamics.}
	We show the graph information planes with/without MC-GPB as follows.
	The model training without MC-GPB memorizes the privacy information at the beginning of training before gradually forgetting it. By applying our MC-GPB, we can enhance such forgetting procedures while preventing the model from discarding task-relevant information that might lead to a drop in accuracy.

	\begin{figure}[H]
		\centering
		\subfigure[GNN with normal training]
		{\includegraphics[width=8.3cm]{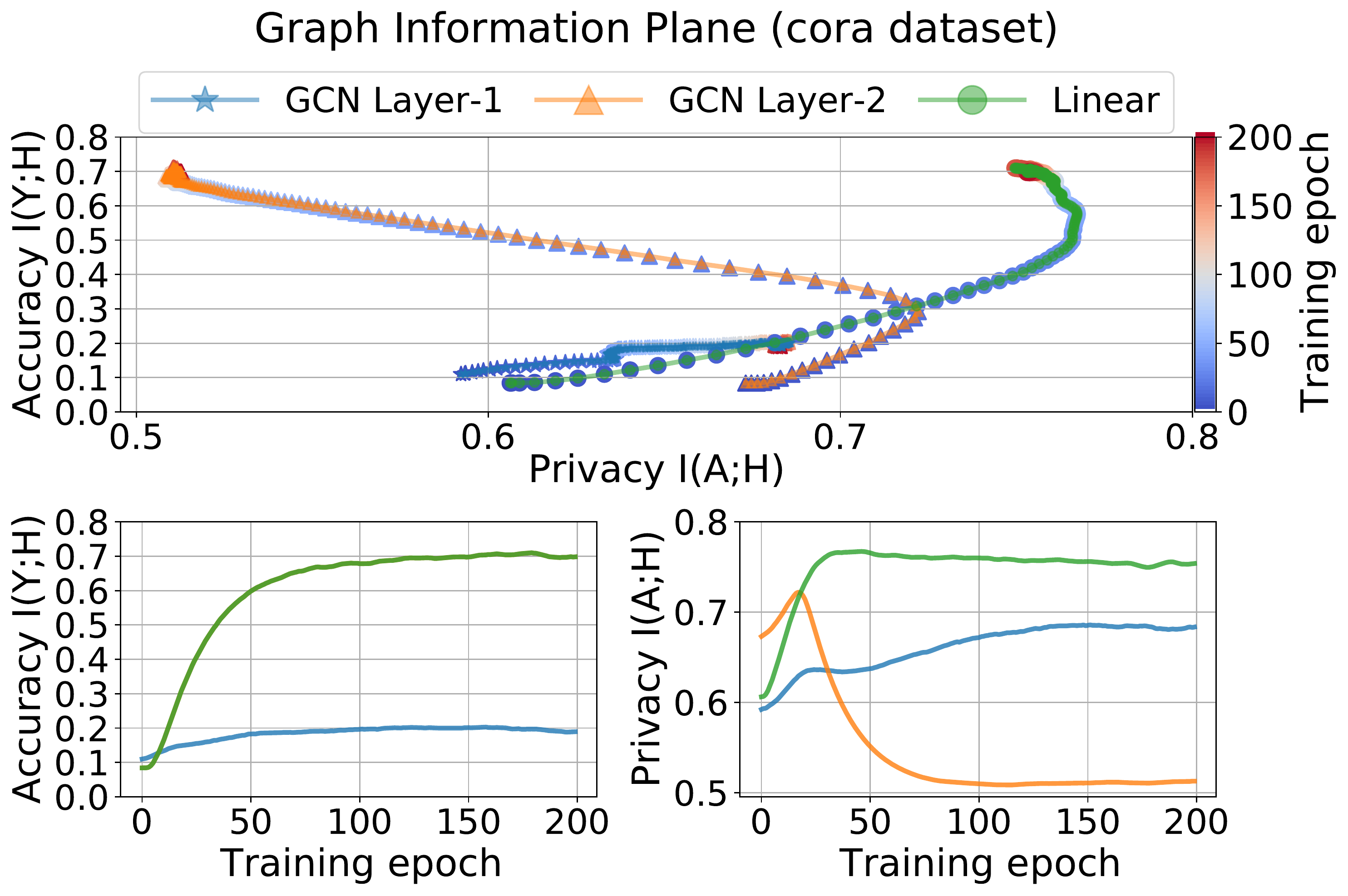}}
		\subfigure[GNN trained with MC-GPB]
		{\includegraphics[width=8.3cm]{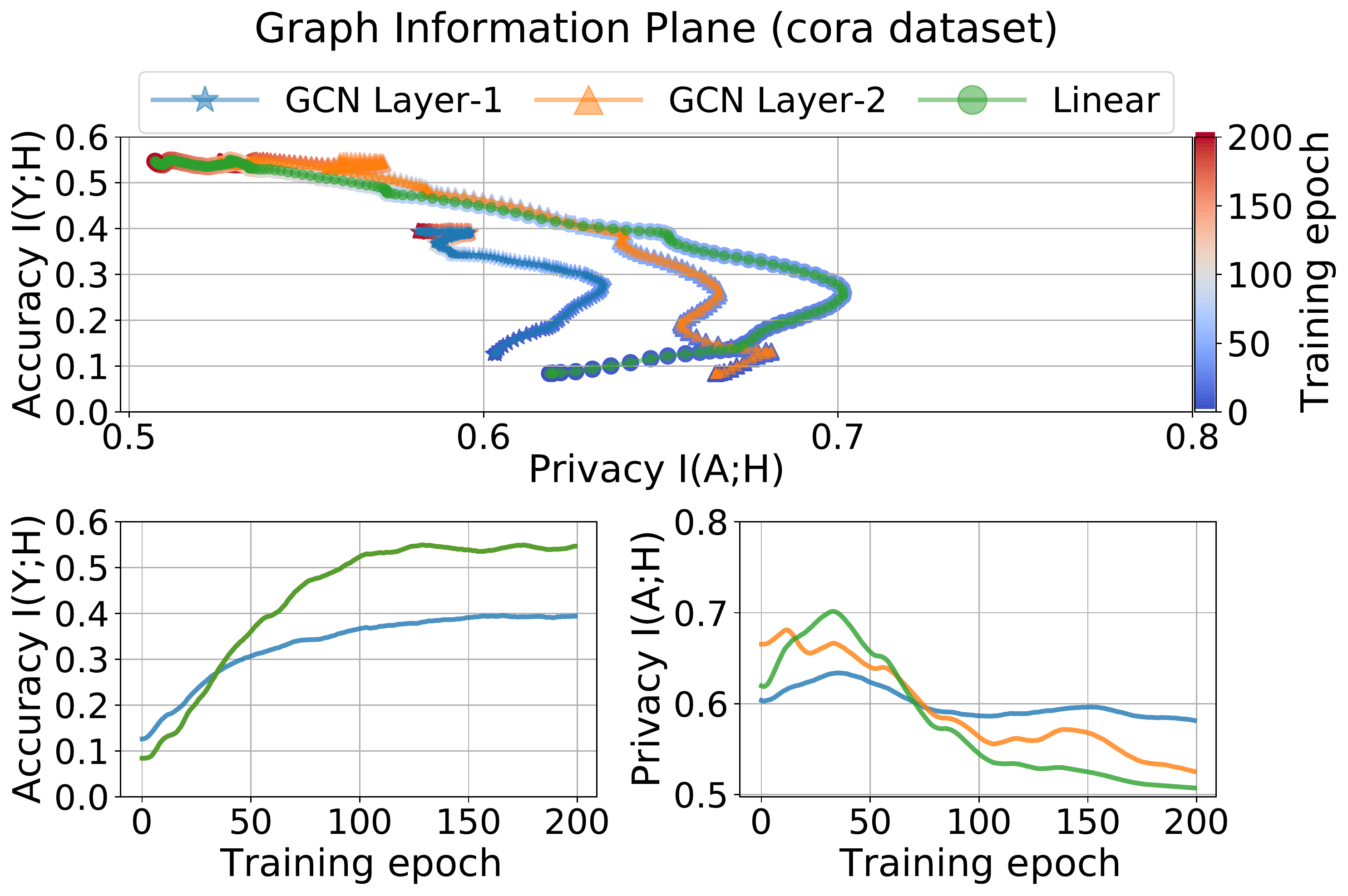}}
		\vspace{-6px}
		\caption{
			Graph information plane on Cora dataset.
		}
            \label{appd:gip:cora}
	\end{figure}
	
	\begin{figure}[H]
		\centering
		\subfigure[GNN with normal training]
		{\includegraphics[width=8.3cm]{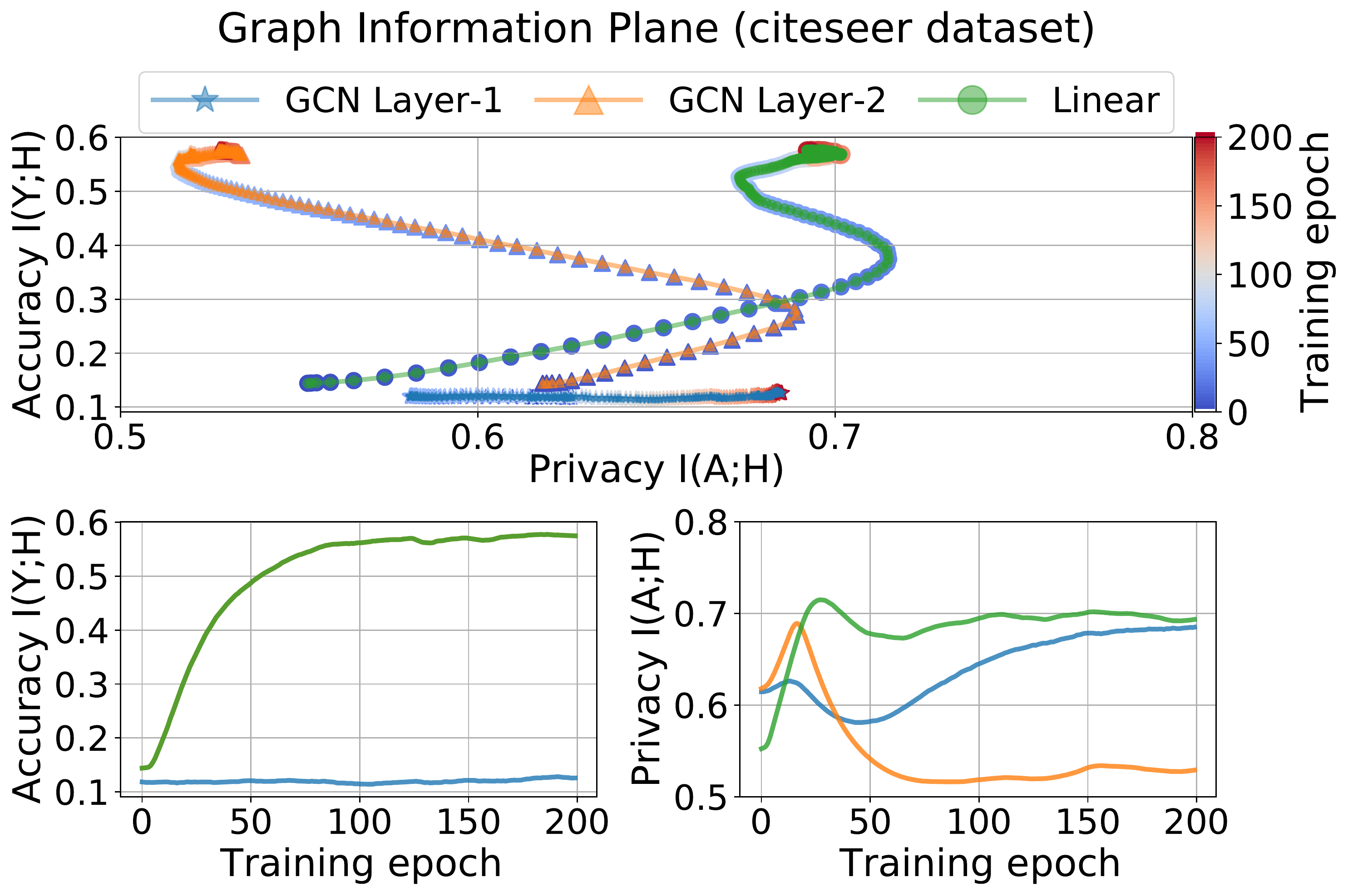}}
		\subfigure[GNN trained with MC-GPB]
		{\includegraphics[width=8.3cm]{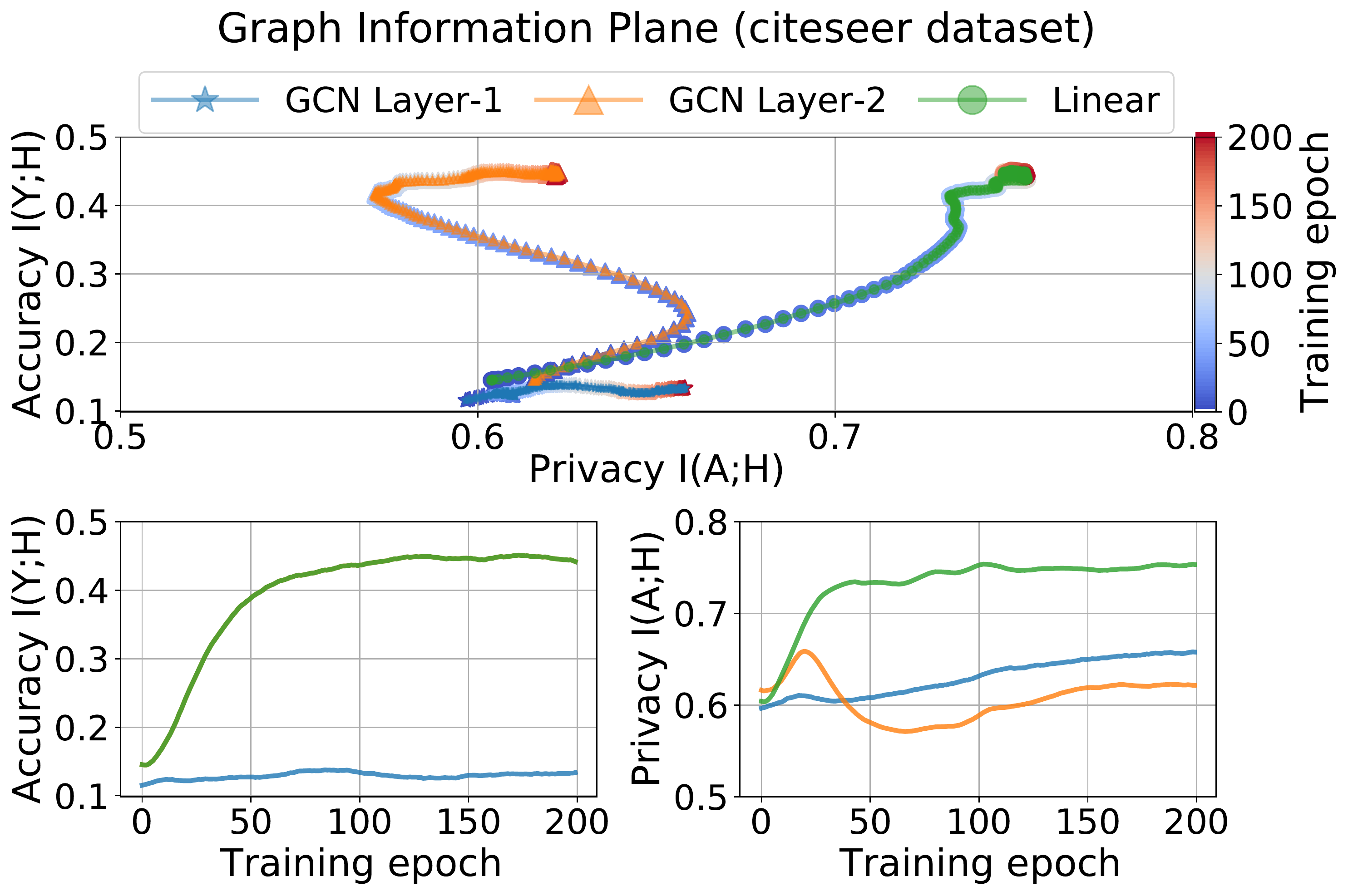}}
		\vspace{-6px}
		\caption{
			Graph information plane on Citeseer dataset.
		}
	\end{figure}
	
	\begin{figure}[H]
		\centering
		\subfigure[GNN with normal training]
		{\includegraphics[width=8.3cm]{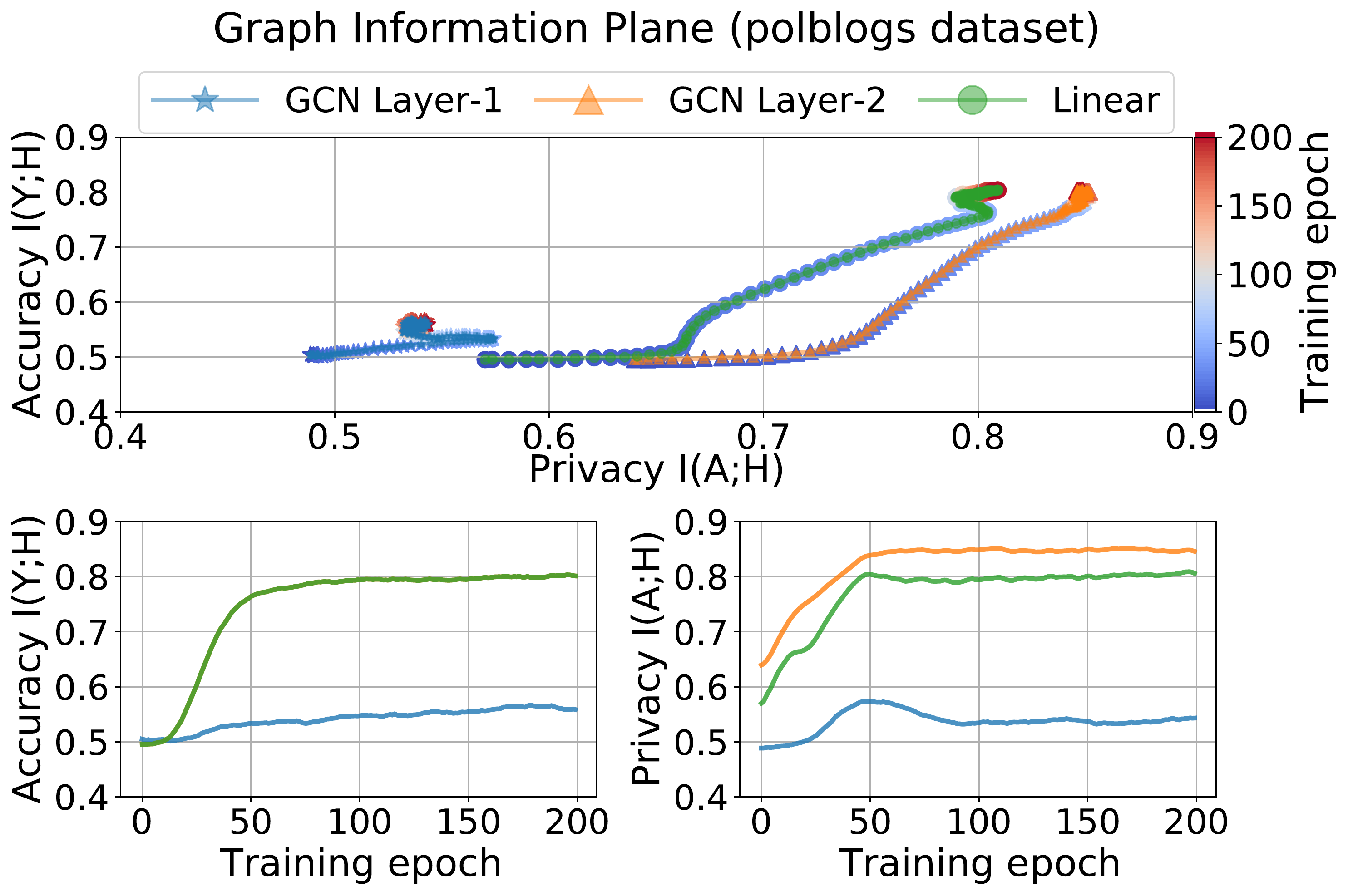}}
		\subfigure[GNN trained with MC-GPB]
		{\includegraphics[width=8.3cm]{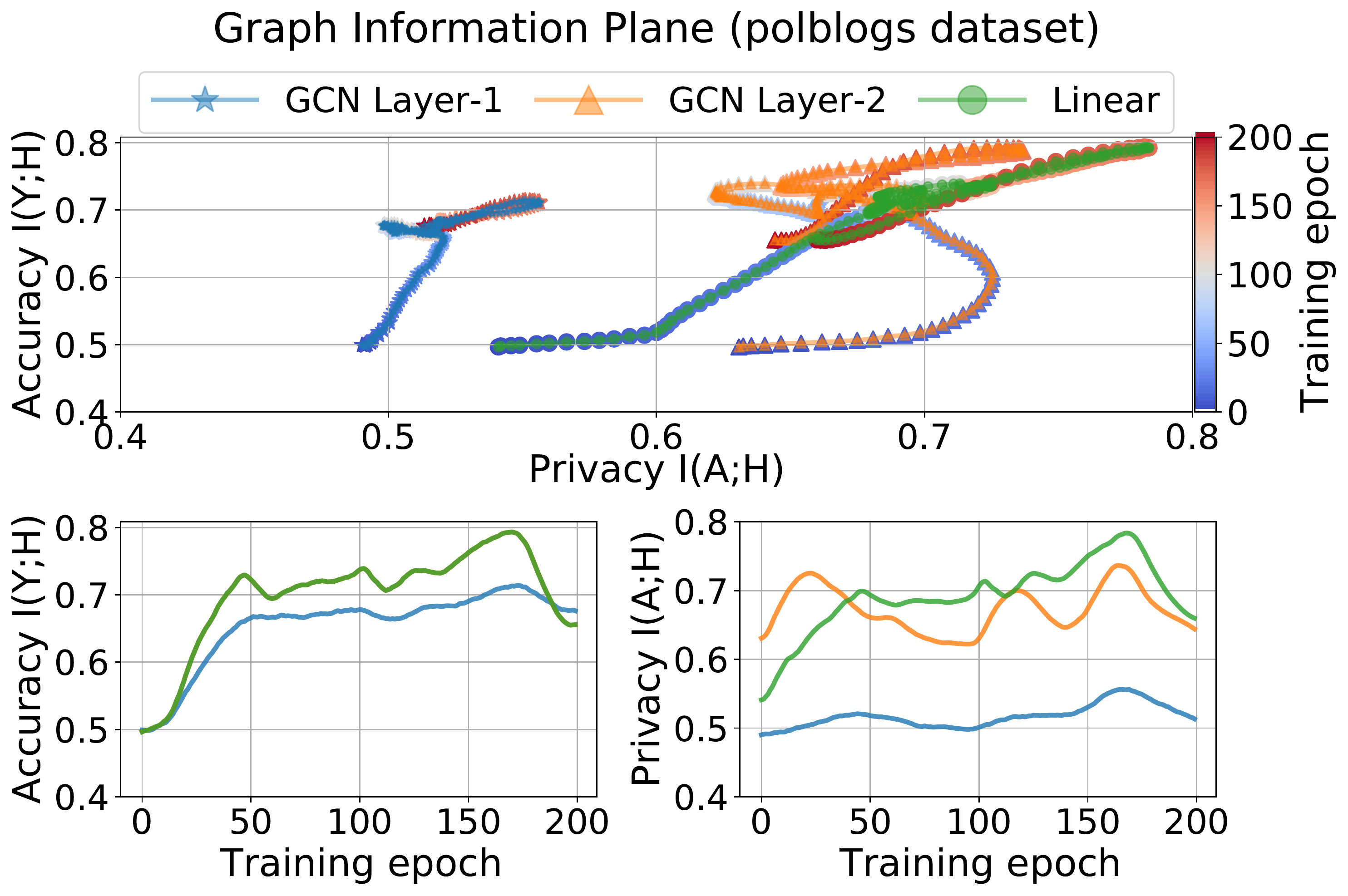}}
		\vspace{-6px}
		\caption{
			Graph information plane on Polblogs dataset.
		}
	\end{figure}
	
	\begin{figure}[H]
		\centering
		\subfigure[GNN with normal training]
		{\includegraphics[width=8.3cm]{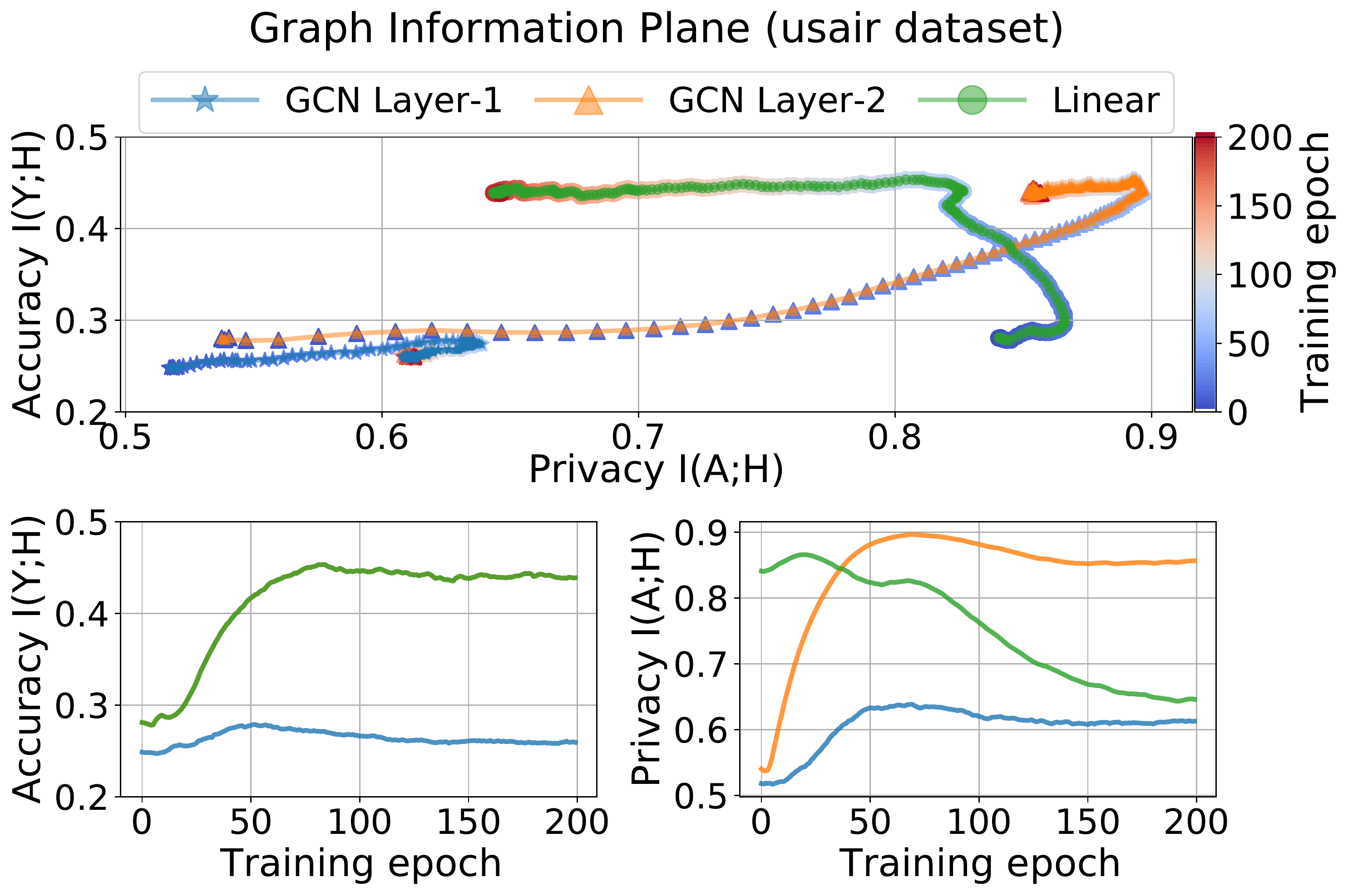}}
		\subfigure[GNN trained with MC-GPB]
		{\includegraphics[width=8.3cm]{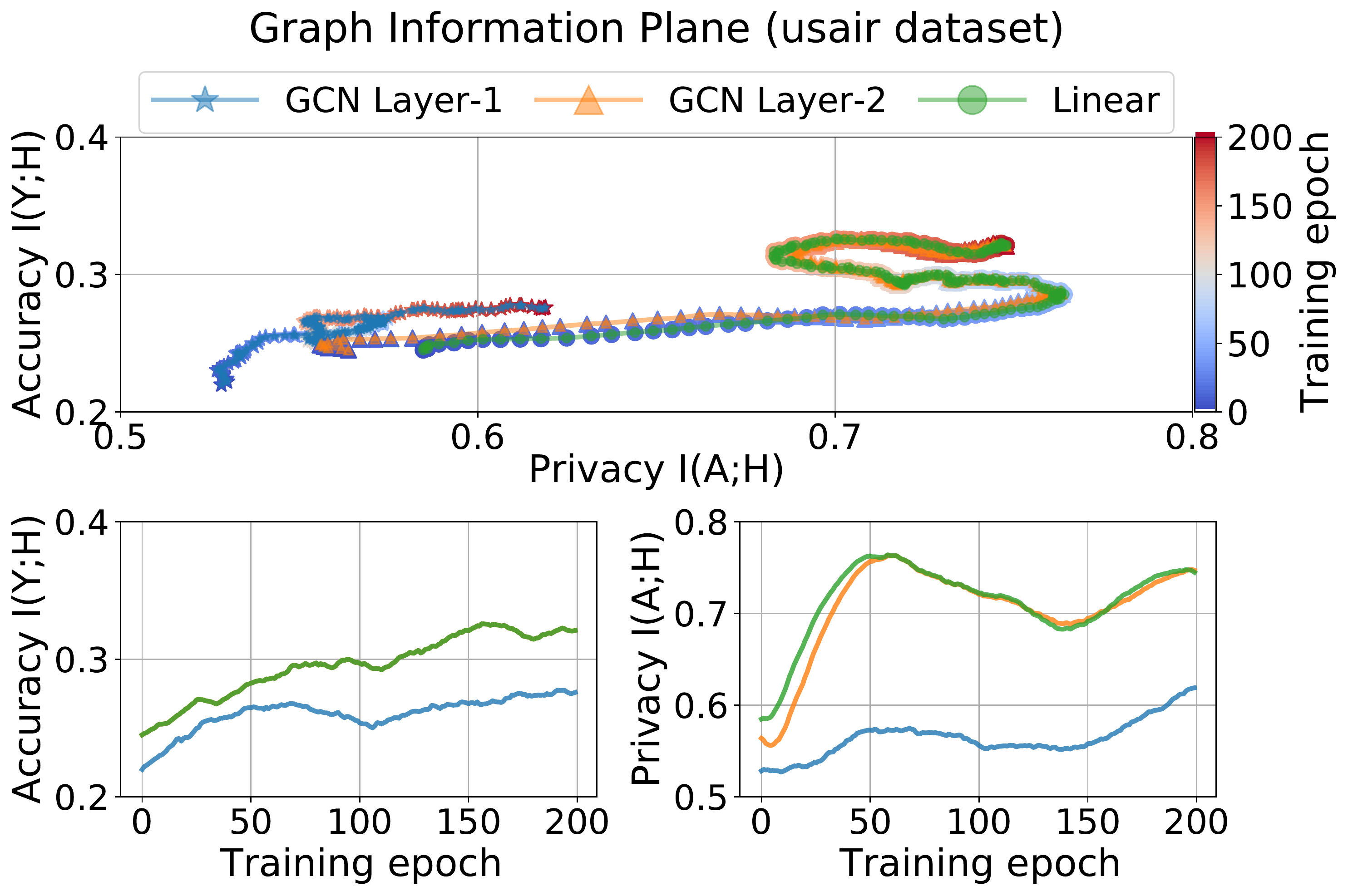}}
		\vspace{-6px}
		\caption{
			Graph information plane on USA dataset.
		}
	\end{figure}
	
	\begin{figure}[H]
		\centering
		\subfigure[GNN with normal training]
		{\includegraphics[width=8.3cm]{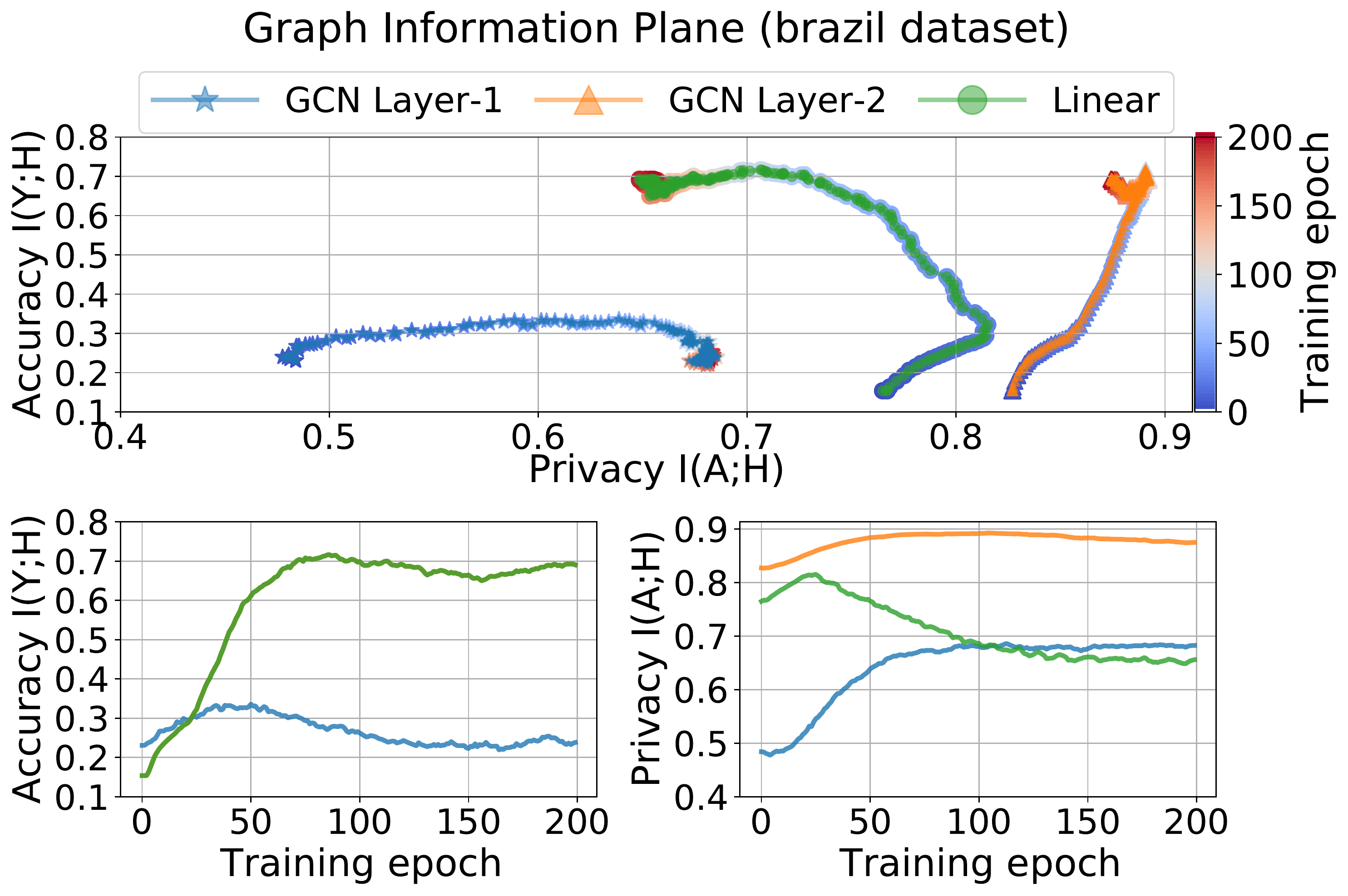}}
		\subfigure[GNN trained with MC-GPB]
		{\includegraphics[width=8.3cm]{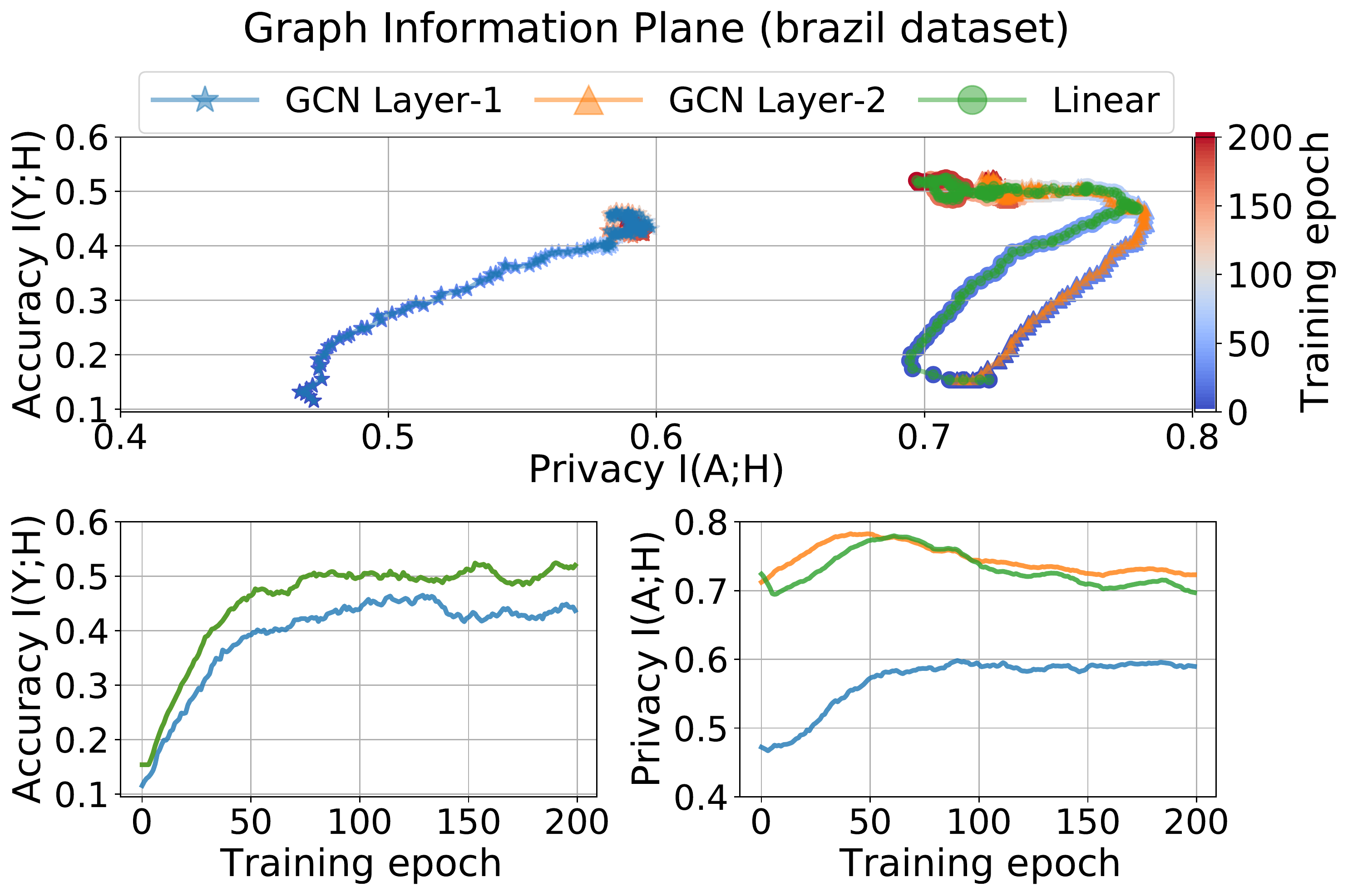}}
		\vspace{-6px}
		\caption{
			Graph information plane on Brazil dataset.
		}
	\end{figure}
	
	\begin{figure}[H]
		\centering
		\subfigure[GNN with normal training]
		{\includegraphics[width=8.3cm]{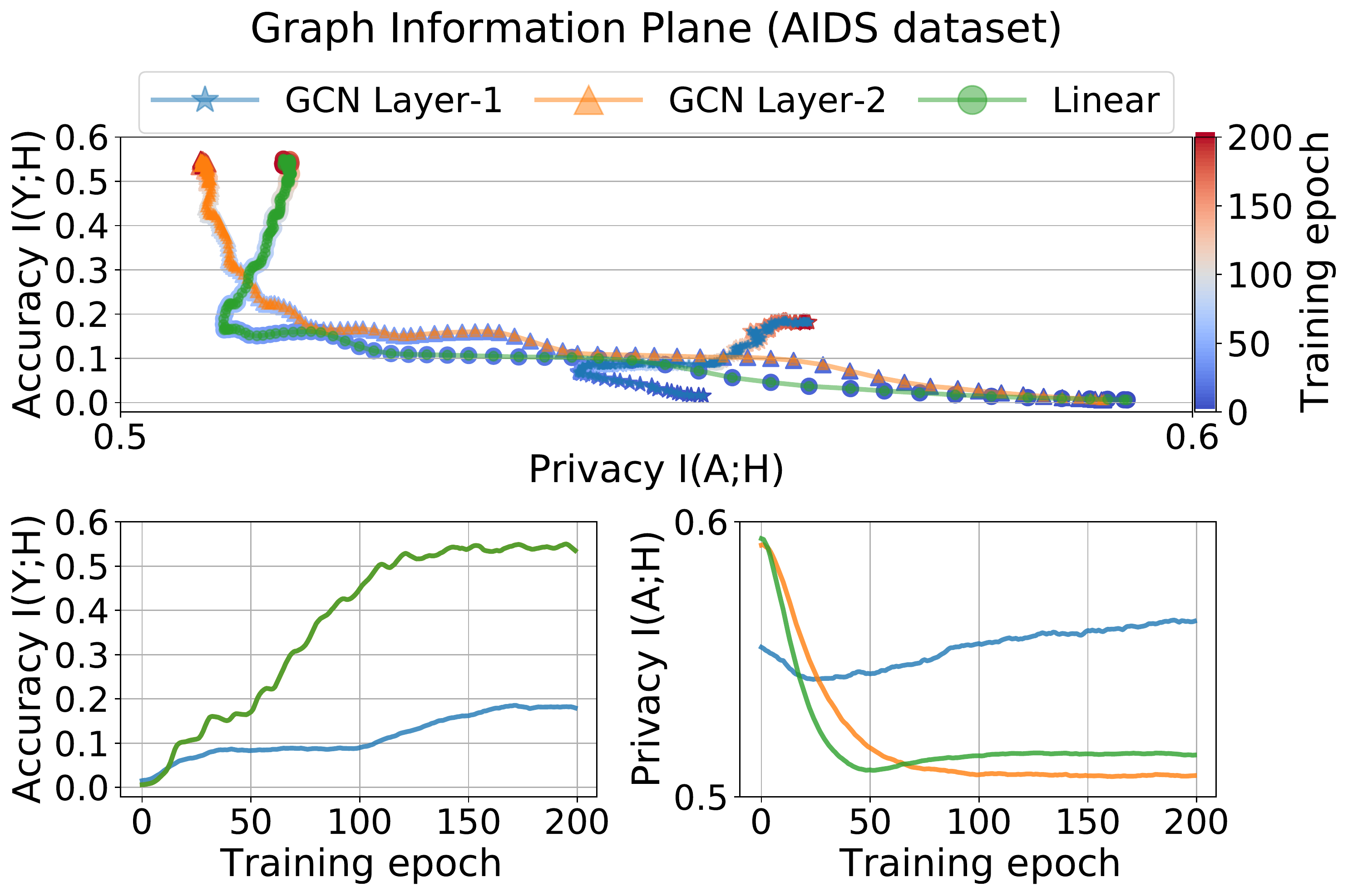}}
		\subfigure[GNN trained with MC-GPB]
		{\includegraphics[width=8.3cm]{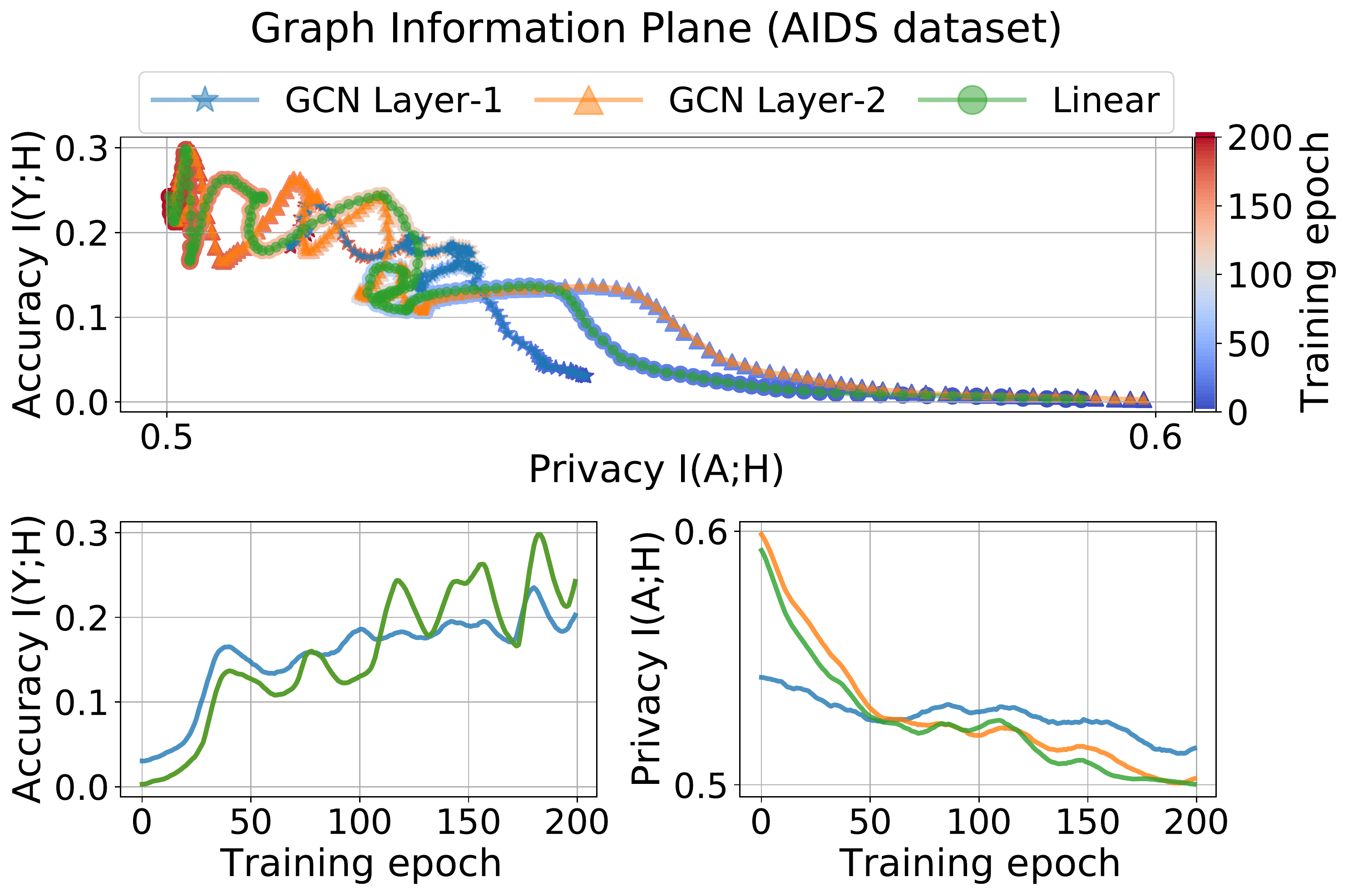}}
		\vspace{-6px}
		\caption{
			Graph information plane on AIDS dataset.
		}
  \label{appd:gip:aids}
	\end{figure}

	\clearpage
	\section{Related work}
	\label{sec: related_work}
	
	
	\subsection{Graph Neural Networks}	
	Predicting node labels 
	requires a parameterized hypothesis 
	$f_{\bm{\theta}}(A, X) \! = \! \bm{\hat{Y}}_{A}$
	with GNN architecture~\citep{kipf2016semi, velickovic2018graph, hamilton2017inductive}
	and message propagation framework~\citep{gilmer2017neural},
	where the architecture can be
	GCN~\citep{kipf2016semi}, GAT~\citep{velickovic2018graph}, or GraphSAGE~\citep{hamilton2017inductive}.
	%
	The forward inference of a $L$-layer GNN
	generates node representations $\bm{H}_{A} \! \in \! \mathbb{R}^{N \times D}$ 
	by a $L$-layer message propagation.

	Formally, let $\ell=1\dots L$ denote the layer index, 
	$h_{i}^{\ell}$ is the representation of the node $i$, 
	$\text{MESS}(\cdot)$ is a learnable mapping function to transform the input feature,
	$\text{AGGREGATE}(\cdot)$ captures the $1$-hop information from 
	neighborhood $\mathcal N(v)$ in the graph,
	and $\text{COMBINE}(\cdot)$ is the final combination between neighbor features and the node itself.
	Then, the $l$-layer operation of GNNs can be formulated as
	$
	\bm m_v^{\ell} \! = \! \text{AGGREGATE}^{\ell}(\{ \text{MESS}(\bm h_u^{\ell-1}, \bm h_v^{\ell-1}, e_{uv}) \! :\!  u \in \mathcal{N}(v)\})
	$,
	where the representation of node $v$ is 
	$
	\bm h_v^{\ell} = \text{COMBINE}^{\ell}(\bm h_v^{\ell-1}, \bm m_v^{\ell})
	$.
	After $L$-layer propagation,
	the final node representations $\bm h_{e}^L$ 
	of each $e \! \in \! V$ are obtained.
	In addition,
	we summarize the detailed architectures of different GNNs
	in the following Table~\ref{tab: GNN-architectures}.

	Then, the follow-up linear layer
	transforms $\bm{H}_{A}$  to classification probabilities
	$\bm{\hat{Y}}_{A} \! \in \! \mathbb{R}^{N \times C}$,
	with $C$ categories in total.
	The training objective is to minimize the classification loss, \textit{e.g.}, 
	the cross-entropy between predictions $\bm{\hat{Y}}_{A}$ and ground truth $Y$.
	
	\begin{table}[H]
		\centering
		\caption{Detailed architectures of different GNNs.}
		\fontsize{10}{16}\selectfont
		\setlength\tabcolsep{6pt}
		\begin{tabular}{c|c|c}
			\toprule
			GNN & $\text{MESS}(\cdot) \; \& \; \text{AGGREGATE}(\cdot)$ & $\text{COMBINE}(\cdot)$ \\ \hline
			GCN & $\bm m_{i}^{l} = \bm W^{l} \sum_{j \in \mathcal{N}(i)} \frac{1}{\sqrt{\hat{d}_i \hat{d}_j }} \bm h_{j}^{l-1}$ 
			& $\bm h_{i}^{l}  = \sigma( \bm m_{i}^{l} + \bm W^{l} \frac{1}{\hat{d}_i} \bm h_{i}^{l-1} )$  \\
			GAT & $\bm m_{i}^{l} =  \sum_{j \in \mathcal{N}(i)} \alpha_{ij} \bm W^{l} \bm h_{j}^{l-1}$ 
			& $\bm h_{i}^{l}  = \sigma( \bm m_{i}^{l} + \bm W^{l}  \alpha_{ii} \bm h_{i}^{l-1} )$  \\
			GraphSAGE & $\bm m_{i}^{l} =  \bm W^{l}  \frac{1}{|\mathcal{N}(i)|} \sum_{j \in \mathcal{N}(i)} \bm h_{j}^{l-1}$ 
			& $\bm h_{i}^{l}  = \sigma( \bm m_{i}^{l} + \bm W^{l}  \bm h_{i}^{l-1} )$  \\
			\bottomrule
		\end{tabular}
		\label{tab: GNN-architectures}
	\end{table}

	\subsection{Privacy Attack on Graphs}	
	Generally, the privacy attack on graphs can be attributed to
	\textit{membership inference attack}, 
	\textit{model extraction attack}, 
	and \textit{model inversion attack}.
	Specifically,
	the membership inference attack~\citep{he2021node}
	aims to indicate whether a data sample is used to train the model.
	Besides,
	model extraction attack~\citep{shen2022model}
	is to extract information about the model parameters and 
	reconstruct a surrogate model that behaves like the black-box model.
	Lastly,
	the model inversion attack aims to extract sensitive features of training data
	with only access to a trained model and non-sensitive features.
	We summarize the literature on the model inversion attack as follows.

	\subsection{Inversion Attack on Graphs}
	
	
	As introduced before, 
	most works of model inversion attack are investigated images and texts domains, 
	leaving its effectiveness in the non-grid domain an open problem, 
	\textit{e.g.}, graph-structured data.
	While recently, several graph neural networks 
	(GNNs)~\citep{kipf2016semi,gilmer2017neural,kipf2016variational, zhang2018link}
	are proposed for graph data
	and boosted many real-world applications, \textit{e.g.},
	recommendation systems~\citep{wu2020graph}
	and drug discovery~\citep{ioannidis2020few}.
	In graph scenarios,
	the target of the model inversion attack is to recover the topology of the training graph,
	\textit{i.e.}, the connectivity properties \textit{w.r.t.} each edge.

	In practice,
	inferring links between nodes leads to a severe privacy threat 
	when the links represent sensitive information,
	\textit{e.g.}, the relationship between users in social networks.
	Besides,
	it may also compromise a model owner’s intellectual property.
	The challenges of applying the model inversion attack 
	to graphs are two folds.
	(1) The discrete nature of graph structure.
	It is hard to optimize in a differentiable way.
	Besides, 
	the nodes and edges in a graph cannot be resized to the same shape.
	(2) Lack of domain knowledge as priors.
	Graph data are less intuitive than images,
	and the domain knowledge can be diverse,
	\textit{e.g.}, molecules, social networks, and citation networks.

	The pioneer works~\citep{duddu2020quantifying, chanpuriya2021deepwalking}
	attempt to reconstruct the target graph from released graph embeddings
	$\bm{H} \in \mathbb{R}^{N \times D}$ of each node,
	that are generated by Deepwalk or GNNs.
	The specific attack method can be
	the Deepwalk Backward~\citep{chanpuriya2021deepwalking},
	or a decoder $f^{dec}(\bm{H}): \mathbb{R}^{N \times D}  \to \{0, 1\}^{N \times N}$ 
	that is trained on auxiliary datasets~\citep{duddu2020quantifying}.
	Graph embedding attack with the auto-encoder framework
	is also an exciting direction as the graph embeddings of each node can usually be accessed in practice.
	\citep{zhang2022inference} systematically investigate the information leakage of graph embedding,
	and justify that the basic graph properties, \textit{e.g.}, number of nodes, number of edges, and graph density, 
	can be accurately extracted.
	Besides, it can determine whether a given subgraph is contained in the target graph or not.
	More importantly, it also shows that the graph topology can be recovered 
	via conducting the MIA with graph embeddings.

	The link stealing attack~\citep{he2021stealing} 
	is the first work to \textit{steal} links from a GNN as the target model.
	It aims to conduct the attack on black-box settings
	with three kinds of prior knowledge, 
	including (1) node features, (2) target dataset’s partial graph, and (3) a shadow dataset.
	This work proposed $8$ different kinds of attacks in total
	to be adaptive to the $2^3 = 8$ scenarios.
	Each proposed method for the attack was verified on
	chemical networks and social networks,
	which justified the feasibility of conducting a model inversion attack on graphs.
	However, it requires to be accessible to the partial graph and an auxiliary dataset.
	The partial graph actually contains sensitive information about the adjacency,
	and selecting the auxiliary dataset also requires extra information about the target graph.
	Thus, these methods cannot be directly utilized here.
	Besides, the GraphMI~\citep{zhang2021graphmi}
	is also a learnable attack that also aims to recover the links of the original graph.
	With the white-box access to the target model,
	the optimal adjacency is obtained 
	by maximizing the classification probability 
	\textit{w.r.t.} given node labels $Y$, 
	namely,
	$\bm{\hat{A}}^* 
	\! = \! \arg \max_{\bm{\hat{A}}} \mathbb{P}(f_{\theta}(\bm{\hat{A}}, X), Y)$.
	In practice, the $\bm{\hat{A}}^* $
	can be recursively updated by the projected gradient descent (PGD).
	The proposed attack method In this work
	is partially inspired by GraphMI,
	and it can be degenerated to GraphMI when
	the prior knowledge set is reduced to $\mathcal{K} = \{X, Y_{sub}\}$,
	where $Y_{sub} \subset Y$.

	In addition to attacking GNNs trained for node classification tasks,
	a recent work~\cite{zhang2022inference} also attempted to attack the model for graph classification tasks.
	The shift from node-level tasks to graph-level tasks brings several unique challenges
	as the obtained one-dimensional embeddings 
	$\bm{h}_{\mathcal{G}} \in \mathbb{R}^D$
	are the compressed information of the whole graph $\mathcal{G}$.
	This work reconstructs the graphs with a graph auto-encoder
	that takes the graph embeddings as inputs and then generates the corresponding graphs.
	Note that the adopted graph auto-encoder is trained on an auxiliary dataset
	and then applied to the target dataset.
	It shares a similar spirit of generative attacks on degenerate
	that the generator (\textit{e.g.}, a generative adversarial network) is pre-trained on public datasets.

	\subsection{Model Inversion Attack on Images}

	Pioneer works~\citep{fredrikson2014privacy, fredrikson2015model, hidano2017model}
	introduce the model inversion attack with 
	comparably simple models,
	\textit{e.g.}, linear regression, decision trees, and shallow networks.
	These early works justified the feasibility of model inversion attacks
	and succeeded in recovering the monochrome images.
	However, the reconstructed images are also usually of low fidelity~\citep{szegedy2013intriguing},
	and they fail in attacking DNNs for image classification tasks.
	
	\textit{So, how can we recover the polychromatic and realistic images used for training?}
	Generative Model Inversion (GMI) \citep{zhang2020secret}
	is the first to conduct model inversion attacks on deep models, 
	\textit{i.e.}, the convolution neural networks.
	Instead of directly reconstructing the private data from scratch,
	its inversion process is guided by a distributional prior 
	through the generative adversarial networks (GAN).
	Specifically, the used GAN is pretrained on public datasets
	for obtaining the generic prior knowledge of human faces
	via minimizing the canonical WassersteinGAN training loss, namely,
	\begin{equation}
		\mathcal{L}_{GAN}(G,D) = \mathbb{E}_x[D(x)] - \mathbb{E}_z[G(z)].
	\end{equation}
	that optimizes the generator $G$ and the discriminator $D$
	in an adversarial training manner.
	Then, GMI introduces a diversity loss to encourage
	the more diverse generated images
	to recover more training patterns in $X^{tra}$.
	To be specific, with two sampled latent vectors $z_1, z_2$,
	the diversity loss is calculated as
	\begin{equation}
		\mathcal{L}_{div}(G) = \mathbb{E}_{z_1, z_2}\big[\frac{||f_\theta^{feat}(G(z_1)) - f_\theta^{feat}(G(z_2))||}{||z_1 - z_2||}\big].
	\end{equation}
	where $f_\theta^{feat}$ is the feature extractor of the target network.
	With these two loss functions,
	the full objective of GAN is as follows.
	\begin{equation}
		\min_{G}\max_{D}\mathcal{L}_{GAN}(G,D) - \lambda \mathcal{L}_{div}(G).
	\end{equation}
	
	
	After training the GAN, 
	GMI aims to find the latent vector $z$ that achieves the highest likelihood 
	under the target network while being constrained to the data manifold learned by $G$,
	\textit{i.e.},
	\begin{equation}
		z^* = \min_z -D(G(z)) - \lambda \log[f_\theta(G(z))].
	\end{equation}
	where a lower prior loss $-D(G(z))$ require the more realistic images,
	while a lower identity loss $\log[f_\theta(G(z))]$ 
	encourages the generated images to have a higher likelihood \textit{w.r.t.} the targeted network.
	In summary,
	GMI conducts the model inversion attack in an end-to-end manner based on GANs
	that can reveal private training data of the target model with high fidelity,
	which make up for the deficiency of the early works.
	Besides, it also reveals the non-convex nature of model inversion attacks on deep models,
	where a more powerful target model can exhibit a higher privacy risk.

	However,
	the top-one identification accuracy of face images 
	inverted from the classifier is not that high.
	Is it because CNNs do not \textit{memorize} much about private data 
	or is it due to the \textit{imperfect} attack algorithm?
	To answer,
	the follow-up work, 
	Knowledge-Enriched Distributional Model Inversion (KED-MI)~\citep{chen2021knowledge},
	shows that the target network maybe not be fully utilized.
	KED-MI further distills the useful knowledge from the target model with two designs.
	On the one hand,
	instead of generating and discriminating real or fake samples,
	DMI utilizes the target model to generate soft labels for supervising the GAN,
	\textit{i.e.}, to minimize the loss $\mathcal{L}_{GAN} = \mathcal{L}(D) + \mathcal{L}(G)$, specifically,
	\begin{align}
		\begin{split}
			\mathcal{L}(D) =& \mathcal{L}_{sup}(D) + \mathcal{L}_{unsup}(D) \\
			\mathcal{L}_{sup}(D) =& -\mathbb{E}_{x \sim p_{data}(x) \sum_{k=1}^{K}}
			f_\theta(x) \log p_{disc}(y=k|x) \\
			\mathcal{L}_{unsup}(D) =& \!- \!\mathbb{E}_{x \! \sim \! p_{data}}\log D(x) \! + \! \mathbb{E}_{z \! \sim \! noise}\log(1 \! - \! D(G(z))) \\
			\mathcal{L}(G) =& || \mathbb{E}_{x \sim p_{data}}f_\theta(x) - \mathbb{E}_{z \sim noise} f_\theta(G(z)) ||^{2}_{2}
			+ \lambda \mathcal{L}_{ent} \\
		\end{split}
	\end{align}
	where the entropy regularization term $\mathcal{L}_{ent}$ is taken from previous work~\citep{grandvalet2004semi}.
	
	On the other hand,
	no longer recovering a sample given a label in a one-to-one manner,
	DMI explicitly parameterizes the distribution of private data
	and proceed with the model inversion attack in a new many-to-one way.
	Technically,
	the latent vectors of the generator are
	sampled from a learnable distribution to capture the class-wise information,
	while the discriminator acts as a (K+1)-classifier
	to differentiate the K classes of private data.
	Here, the latent variable $z$ is parameterized by 
	$z = \sigma \epsilon + \mu$ by the reparameterization trick
	that the corresponding distribution $p_{gen}$ samples the optimal $z^*$ as follows.
	\begin{equation}
		z^* = \min_z -\mathbb{E}_{z \sim p_{gen}} \log D(G(z)) 
		- \lambda \mathbb{E}_{z \sim p_{gen}} \log[f_\theta(G(z))].
	\end{equation}

	Basically,
	a successful attack should generate realistic and diverse samples.
	\textit{So, how can we generate more diverse samples?}
	The Variational Model Inversion (VMI)~\citep{wang2021variational}
	further formulates the model inversion attack as the variational inference.
	VMI generally can bring a higher attack accuracy and diversity
	for its equipped powerful generator StyleGAN to optimize its designed variational objective.
	Specifically,
	for the target class $y$, VMI approximates the target posterior
	with a variational distribution $q(x) \in Q_x$ from the variational family $Q_x$.
	The target model $f_{\theta}(x)$ is then denoted as $p_{\text{TAR}}(x|y)$.
	The variational objective is derived as follows.
	\begin{align}
		\begin{split}
			q^*(x) =& \min_{q \in Q_x} D_{KL}(q(x) || p_{\text{TAR}}(x|y)) \\
			=& \min_{q \in Q_x} \mathbb{E}_{q(x)}\big[ -\log p_{\text{TAR}}(x|y) 
			+ D_{KL}(q(x) || p_{\text{TAR}}(x) ) \big]
		\end{split}
	\end{align}

	In addition to recovering images,
	\textit{can model inversion attack be applied to other extensions?}
	The Contrastive Model Inversion (CMI)~\citep{fang2021contrastive}
	aims for data-free knowledge distillation.
	It recovers the training data from the target model via model inversion attacks,
	based on which it trains a student model.
	To overcome the mode collapse problem
	that recovered images are highly similar to each other,
	CMI proposes the contrastive learning objective
	upon the generated data
	to promote diversity while remaining considerable fidelity.
	With the similarity measurement
	$sim(x_1, x_2, h) = cos(h(x_1), h(x_2) = \nicefrac{<h(x_1), h(x_2)>}{||h(x_1)|| \cdot ||h(x_2)||}$,
	where the $h(\cdot)$ projects $x_i$ to the embedding space,
	the contrastive loss is formulated as 
	\begin{equation}
		\mathcal{L}_{con}(X, h) = -\mathbb{E}_{x_i \in X}\Bigg[ \log 
		\frac{exp(sim(x_i, x_i^{+}, h)/\tau)}{\sum_{j}exp(sim(x_i, x_j^{-}, h)/\tau)} \Bigg]
	\end{equation}

	Besides, XAI-aware model inversion attack~\citep{zhao2021exploiting} shows that
	the additional knowledge 
	collected from the model inference procedure
	can promote the model inversion attack performances.
	In detail, if the model explanations
	\textit{e.g.}, saliency maps of Gradients or Grad-CAM,
	are attainable in practice,
	it might do harm to privacy since 
	these explanations can help recover private data.

	In addition,
	\citep{kahla2022label} conducts the first label-only model inversion attack
	only accessing the model's predicted labels without the confidence scores.
	As a machine learning model is often packed into a black-box
	that only generates the hard label (\textit{i.e.}, the label of the class with the highest probability),
	such an attack scenario is more practical 
	but also much more challenging to perform.
	Despite requiring less knowledge about the target model,
	this work justifies that such a black-box attack is also feasible and effective.
	Specially,
	it attempts to generate the most likelihood images for the target class,
	and observes that a region of high likelihood
	shall be located in the center of the class.
	Based on this observation, 
	this work proposes to iteratively move the generated image
	away from the decision boundary and closer to the center.
	
	Recent advance~\citep{struppek2022ppa}
	significantly decreases the cost of conducting a model inversion attack
	through relaxing the dependency between the target model and the image prior.
	This work enables the use of a single GAN to attack a wide range of targets,
	requiring only minor adjustments to the attack.
	Moreover, this work shows that the model inversion attack is possible even 
	with publicly available pre-trained GANs and under strong distributional shifts.

	\subsection{Model Inversion Attack on Texts} 
	In addition to 
	recovering training images with visual models introduced before,
	model inversion attack on text data with language models 
	is also attacking more and more interests.
	In this domain,
	the input $X$ is changed from image to text (\textit{i.e.}, sentences),
	and the architecture of model $f_\theta$ 
	is also shifted from the convolutional neural network to the transformer-based language model.

	A pioneer work~\cite{carlini2021extracting}
	demonstrates that large language models (\textit{i.e.}, the GPT-2) memorize and leak individual training examples,
	with only black-box query access.
	The private information of an individual person can be accurately recovered.
	More importantly,
	this work reveals that
	some worst-case training examples are indeed memorized,
	although training examples do not have noticeably lower losses than test examples on average.
	Such a phenomenon is correlated with the 
	memorization effect of DNNs and deserved further investigations.

	Besides,
	another work, Text Revealer~\citep{zhang2022text}, 
	firstly proposed to apply the model inversion attack 
	to text classification with the transformer-based pretrained language models.
	Its attack consists of two stages:
	(1) collect texts from the same domain as the public dataset
	and extract high-frequency phrases from the public dataset as templates;
	(2) train a language model as the text generator on the public dataset.
	By minimizing the text classification loss, \textit{i.e.}, cross-entropy,
	the generated text distribution becomes closer to the private dataset.

	\subsection{On defending Model Inversion Attack}
	
	As for defending against the model inversion attack,
	a natural solution can be differential privacy (DP).
	Although effective in defending the membership attack~\citep{abadi2016deep},
	the techniques of DP are proved to be ineffective with model inversion attacks~\citep{fredrikson2014privacy, zhang2020secret}.
	
	The mutual-information-based defense (MID)~\citep{yang2019adversarial}
	and Bilateral Dependency Optimization (BiDO)~\citep{peng2022bilateral} are two representative defense methods that are specially designed for the model inversion attack.
	They follow a similar principle, that is, to control the mutual information among inputs $X$, hidden representations $Z$, 
	prediction outputs $\hat{Y}$, and labels $Y$.
	Specifically,
	MID directly decreases the mutual information between $I(X; \hat{Y})$.
	BiDO forces the model to learn
	the robust representations by minimizing $I(X; Z)$
	to limit redundant information that is
	transferred from the inputs to the latent representations,
	while maximizing $I(Y; Z)$ to keep the representations informative.
	
	These two robust methods are effective in defending against model inversion attacks.
	The recovered images are neither correct nor realistic.
	However, 
	such defense methods can do harm to the performance of the target model,
	as the informative signals from the input side
	can be overlooked under the balance of mutual information.
	Thus, a better trade-off between 
	the model inversion-robustness and prediction performance is expected.
	In general,
	such an area of defending against model inversion attacks is still under-explored.

	%
	In general,
	the principle of conducting the model inversion attack is to
	utilize prior knowledge as much as possible,
	to extract more information from the target model,
	in order to generate more realistic and diverse samples.
	While defending against the model inversion attack,
	one promising solution
	is to store less information about input data in the weights of the model.
	In this way, the attacker is unable to recover the private data
	via querying the target model.
	However,
	it usually forms a trade-off between privacy and accuracy
	that such privacy-safe solutions can harm the accuracy,
	Thus, a better defensive approach is needed.
	The model inversion defense in practice
	where several trained models are expected to be protected without 
	further modifications are much more challenging and essential.

	\section{A Further Discussion on Graph Reconstruction Attack}
	\label{sec: further discussion on GRA}
	\subsection{Background of the research problem}
	
	In this part, 
	we would further clarify the background and settings
	of the research problem in our work, \textit{i.e.}, Graph Reconstruction Attack.
	To be specific, we provide rigorous answers to the three following research questions.
	
	\begin{itemize}[leftmargin=*]
		\setlength\itemsep{0.1em}
		\item 
		\textit{Q1. About the black-box or white-box attack settings regarding accesses to the target model.}
		Here, the black-box setting indicates that the attackers 
		can only query the target model and receive the classification results,
		which is the most difficult setting for the adversary.
		When a company employs GNN tools from another company 
		that could be considered an adversary who possesses black-box access to the GNN model.
		On the contrary,
		the white-box setting means that the entire parameters of the target model can be obtained.
		Here, white-box attacks have created an increasingly serious threat to privacy 
		due to the rise in the number of internet venues and online platforms where users can download full models.
		
		\item 
		\textit{Q2. About accesses to the prior knowledge.}
		The GNN model prediction results $ \bm{\hat{Y}}_{A}$ 
		shared by many departments within the same corporation may be accessed by an adversary. 
		For instance, to train a GNN model for fraudulent account detection, 
		a social network service provider uses the technology of another business. 
		In this situation, the supplier frequently must send the business 
		the nodes' prediction results in order to debug or improve them.
		Similar circumstances apply to representations $\bm{H}_{A}$ (\textit{i.e.}, node embeddings), which are typically released. 
		Furthermore, GNNs use the message passing framework, as is common knowledge, 
		to produce representations of each node that are used in downstream tasks like node classification and link prediction.
		Additionally, a prior study~\citep{he2021stealing} also took into account the availability of an auxiliary dataset 
		and a partial graph $A_{sub} \subseteq A$. 
		The additional prior knowledge does successfully improve the GRA, even though it necessitates more access. 
		As a result, 
		the GRA can be carried out if the attacker can access the trained GNN model from malicious clients 
		and has some leaked prior knowledge that is connected to the model.
		
		\item 
		\textit{Q3. About the attack target \textit{i.e.}, the adjacency rather than the node feature.}
		Intuitively, the adjacency $A$ and the node feature $X$
		are all located on the input sides of the forward Markov chain,
		which means both of them can be the inversion target.
		The key motivation to attack the adjacency is two folds,
		\textit{i.e.}, its practical risks and understandability to human beings.
		For instance, social network data is gathered with user consent 
		in order to train GNNs for better service, such as friend classification or ad recommendation. 
		It should be noted that user friendship data is sensitive and relational and that it should be kept secret.
		If the user's friendship is recovered in this scenario, the user's privacy will be compromised. 
		A bigger security risk is presented by the fact that attackers can grasp such privacy, which is more critical. 
		As a result, the privacy of graph adjacency data should receive greater attention and protection 
		because it is more sensitive and intelligible than the node feature.
		
	\end{itemize}
	
	\subsection{Graph Reconstruction Attack in practice}
	Here, we provide a detailed explanation 
	for the existence of the threat model in practice
	with several real-world examples.
	
	\textit{Q1. Why can adjacency be attacked, and should be protected in practice? }
	In practice, inferring links between nodes leads to a severe privacy threat when the links represent sensitive information, \textit{e.g.}, the relationship between users in social networks. Taking social networks as an example, which require gathering interaction among individuals, GNNs can only have satisfied performance on downstream tasks like community detection or ad recommendation once the network structure is accurately characterized. However, this connection among users should be private since it is gathered with user consent and shall be kept between the service provider and users. 
        
	For instance, to train a GNN model for fraudulent account detection, a social network service provider uses the technology of another business. In this case, the supplier will frequently send the business the nodes' prediction results to debug or improve them. Similar circumstances apply to node representations, which are typically released. Thus, the model's outputs shared by many departments within the same corporation can be accessed by an adversary. The same to node features and node labels. Note that the graph reconstruction attack can be conducted with only a subset of the above informative variables, as we have empirically justified its feasibility in Section~\ref{sec: experiment}. The attack target here, \textit{i.e.}, links, can reflect the model owner's sensitive relationship information or intellectual properties, which brings considerable safety risk that is orthogonal to the well-known and widely-studied adversarial attacks~\citep{dai2018adversarial, chen2022understanding, zhu2023combating}.
    The inversion of adjacency is a severe privacy threat in several real-world scenarios of GNNs, which have been widely used in recommendation systems, social networks, citation networks, and drug discovery. 
	
	Thus, the users' privacy should be paid attention to and protected, especially for personal relationships and sensitive information. The private connection among individuals shall be protected since it might become a powerful weapon for fraud syndicates to generate fake identities or threaten people with concerns about their secret relationships. Moreover, in the drug discovery scenario, the company might train its graph generation model for searching for new medicine and share its model with other companies for further development. However, the dataset for training the model contains private drug structures on the market, which is worth stealing as it involves tremendous efforts to discover and test the drug's safety before appearing on the market. 
	
	\textit{Q2. Why investigating the GRA is meaningful and practical?}
	One cannot ignore the importance of privacy leaking of existing GNNs, regardless of the possibility of the attack. The model inversion attack also receives noticeable attention in the visual domain~\citep{fredrikson2014privacy, hidano2017model, chen2021knowledge, wang2021variational, zhao2021exploiting, kahla2022label} 
    and the natural language processing domain~\citep{carlini2021extracting, zhang2022text}. For the graph domain, previous works~\citep{he2021stealing, zhang2021graphmi} have justified the practical feasibility and possible impact on real life.
	
	The GRA, or the general MIA, is a well-defined problem that attacks growing interests. We have justified that the adjacency matrix contains rich private information and is prone to attack. The purpose of this work is to illustrate the flaw of the current GNNs training process and provide a direction to protect the model.
        One cannot wait until a mistake is made to fix it.

	\subsection{Issues of existing attack or defense methods}
	\label{ssec: GRA analysis}
	
	
	In this part, 
	we further elaborate
	the challenges of the studied GRA problems 
	and issues of existing methods (introduced in Appendix.~\ref{sec: related_work}),
	which are summarized in the following three folds.
	
	\begin{itemize}[leftmargin=*]
		\setlength\itemsep{0.1em}
		\item 
		Directly applying existing methods~\citep{fredrikson2015model}
		to graphs can be easily sub-optimal.
		The attribute is that they are originally designed for grid data like images,
		overlooking the inherently topological and semantical properties of graphs.
		Besides, another modality gap
		is the lacking of a distributional prior (\textit{e.g.}, a public face dataset)
		stored in a pre-trained generative adversarial network
		that is used to guide the inversion process of graphs.
		Thus, it hinders several generative attack methods~\citep{zhang2020secret, struppek2022ppa}
		to be applicable to graphs.
		
		\item 
		Considering the inductive nature
		that graphs can be collected from diverse domains,
		the fetched prior knowledge set $\mathcal{K}$ is of vital importance.
		However,
		$|\mathcal{K}|$ can be $2$, $3$, or $4$ (with $7$ combinations in total),
		while trivially treating each case with a specific method is quite non-general.
		Thus, how the adaptively utilized each $\mathcal{K}_i \! \in \! \mathcal{K}$
		in the form of \textit{one} generic objective of combination optimization,
		is the main challenge to solving the non-convex problem here.
		
		\item 
		As for the defense,
		the differential privacy techniques are proven to be of little help
		to defend against MIA~\citep{zhang2020secret, zhang2021graphmi},
		although it can be helpful to defend against the membership attack.
		On the other hand, the improvement of privacy guarantee 
		might come at the cost of degenerating the empirical performance~\cite{bietti2022personalization}.
		Thus, an effective defense method customized for GNNs 
		that nicely balancing of accuracy and privacy is expected.
	\end{itemize}

        \textbf{Technical contribution.}
	To better clarify the technical contribution of our work, we provide a brief summary with regard to existing works as follows.
        \begin{itemize}[leftmargin=*]
		\setlength\itemsep{0.1em}
		\item 
            For the attack, we propose the Markov Chain-based Graph Reconstruction Attack (MC-GRA) that boosts the attack fidelity with parameterization techniques and injected stochasticity. Unlike existing GRA methods designed for ad-hoc scenarios, our proposed MC-GRA aims to locate, present, and utilize the interplaying variables of GNN forward in a generic manner. It can adaptively support the white-box GRA that leverages the target model and any prior knowledge. That is, \textit{to recover better, you must extract more.}

            \item 
            For the defense, existing works have justified that differential privacy (DP) is ineffective with GRA (or general MIA). And currently, there is a lack of an effective way to defend GRA. In this work, we propose the Markov Chain-based Graph Privacy Bottleneck (MC-GPB), an information theory-guided principle that significantly degenerates GRA with only a slight accuracy loss. The MC-GPB requires the GNN to forget the privacy information in the training process, \textit{i.e.}, to make the learned representations contain less information about adjacency. That is, \textit{to learn safer, you must forget more.}

            \item 
            To the best of our knowledge, we are the first to conduct a systematic study of GRA from both sides of attack and defense. By taking GNNs as a Markov chain and attacking GNNs via a flexible chain approximation, we systematically explore the underlying principles of GRA and reveal several essential phenomena. In addition, we also provide a rigorous analysis from information-theoretical perspectives to disclose several valuable insights on how to strengthen and defend GRA.
            
        \end{itemize}
        
	\subsection{The information-theoretic principles of GRA}
	\label{ssec: principle of GRA}
	
	Basically, the objective of the attack is to recover the adjacency, as was stated earlier. 
	On the other hand, the defense consists of acquiring a dependable and thoroughly trained model that is resistant to assault. 
	In order to launch an attack, one must first collect and combine all of the relevant prior knowledge and then do a backward recovery concerning the adjacency; 
	As a sort of defense, rather than that, it is necessary to mandate that the GNN forget all of the information on the adjacency during its training phase.
	
	Specifically, 
	the correlation between each $\mathcal{K}_i \in \mathcal{K}$ and 
	its counterpart in the forward process with the recovered adjacency $\hat{\bm{A}}$
	should be encouraged to enhance the GRA.
	For example, $\hat{\bm{Y}}_{A} \in \mathcal{K}$ should be approximated by $\hat{\bm{Y}}_{\hat{\bm{A}}}$,
	\textit{i.e.}, a higher $I(\hat{\bm{Y}}_{A} ; \hat{\bm{Y}}_{\hat{\bm{A}}})$,
	that is essential to obtain a high fidelity $I(A; \hat{\bm{A}})$.
	On the contrary,
	constraining the correlation between intermediate variables
	and the original adjacency, \textit{e.g.}, a lower $I(A; \bm{H}_{A}^{i})$ for $i = 1,2,\cdots, L$,
	is a natural solution to defend the GRA.
	As such,
	even these variables are leaked, 
	the attacker is scarcely possible to recover $A$.

	\subsection{Deriving the MC-GRA objective}
	\label{ssec: deriving MC-GRA}
	
	For approximating $A$ by $\bm{\hat{A}}$,
	the basic objective of attacking and its relaxed form
	to optimize $\bm{\hat{A}}$ are derived as follows.
	
	\textbf{The basic attack objective.}
	%
	Intuitively,
	given a prior knowledge set
	$\mathcal{K} \! \subseteq \! \{X,  Y, \bm{H}_{A}, \bm{\hat{Y}}_{A} \}$,
	the optimal recovered adjacency $\bm{\hat{A}}^*$ can be obtained by
	directly maximizing its correlation with each term $\mathcal{K}_i \! \in \! \mathcal{K}$,
	\textit{i.e.}, solving the Basic-GRA,
	\begin{equation}	
		\bm{\hat{A}}^* = \arg \max_{\bm{\hat{A}}} I(\bm{\hat{A}}; \mathcal{K}) 
		\triangleq \sum_{\mathcal{K}_i  \in \mathcal{K}}  \alpha_{i} I(\bm{\hat{A}} ; \mathcal{K}_i).
		\label{eqn: Basic-GRA}
	\end{equation}
	where the hyper-parameters $\{\alpha_{i}\}_{i=1}^{|\mathcal{K}|}$ 
	balance the MI terms $\{I(\bm{\hat{A}} ; \mathcal{K}_i)\}_{i=1}^{|\mathcal{K}|}$.
	Intuitively,
	maximizing $I(\bm{\hat{A}} ; \mathcal{K})$ 
	enables to extract information in $\mathcal{K}$ and store it in $\bm{\hat{A}}$
	to approximate $A$, namely,
	\begin{equation}	
		\max_{\bm{\hat{A}}} H(\bm{\hat{A}}) \! \approx \!  H( \mathcal{K})
		\; \Rightarrow \;
		\max_{\bm{\hat{A}}} I(A ; \bm{\hat{A}}) \! \approx \! I(A ; \mathcal{K}).
	\end{equation}
	The Basic-GRA can be applied to \textit{black-box} settings,
	however,
	it can also be sub-optimal as locations of $\bm{\hat{A}}$ and $\mathcal{K}_i$ are distant:
	$\bm{\hat{A}}$ is in the front-end while $\mathcal{K}_i$ is in the back-end.
	Which means,
	the information unrelated to adjacency
	induced by the \texttt{ORI-chain}, 
	\textit{i.e.}, $H(\mathcal{K}_i | A)$,
	will be also stored in $H(\bm{\hat{A}})$.
	That is,
	$\max_{\bm{\hat{A}}} I(A ; \bm{\hat{A}}) \! \approx \! I(A ; \mathcal{K})$
	might come at the cost of
	$H(\bm{\hat{A}}) \! \approx \!  H( \mathcal{K})$.
	Besides,
	the knowledge stored in the target model $f_{\bm{\theta}^*}(\cdot)$ is entirely not utilized.
	Thus, the recovered $\bm{\hat{A}}^*$ is not good enough,
	and a refined objective is required.

	\textbf{The relaxed attack objective.}
	For extracting the target model and relaxing the optimization simultaneously,
	we replace $\bm{\hat{A}}$ in Eq.~\eqref{eqn: Basic-GRA}
	by the latent variable $\bm{Z}_{\bm{\hat{A}}}^j$
	generated by \texttt{GRA-chain} $f_{\bm{\theta}^*}(\bm{\hat{A}}, X)$.
	We promote $I(\bm{Z}_A^j; \bm{Z}_{\bm{\hat{A}}}^j)$ instead of $I(\bm{Z}_A^j; \bm{\hat{A}})$,
	as it provides supervision signals that can be tractably approximated.
	Here,
	$\bm{Z}_{\bm{\hat{A}}}^j$ shares the same location as 
	$\mathcal{K}_i$ in the chain
	($\bm{Z}_A^j \! = \! \mathcal{K}_i$).
	The derived new objective is
	\begin{equation}	
		\begin{split}
			\text{MC-GRA: \;}
			\bm{\hat{A}}^* = \arg \max_{\bm{\hat{A}}}
			\underbrace{\alpha_{p} I(\bm{H}_A ; \bm{H}^{i}_{\bm{\hat{A}}})}_{\text{propagation approximation}}
			+ \underbrace{\alpha_{o} I(\bm{Y}_A ; \bm{Y}_{\bm{\hat{A}}})
				+ \alpha_{s} I(Y ; \bm{Y}_{\bm{\hat{A}}})}_{\text{outputs approximation}}
			- \underbrace{\alpha_{c} H(\bm{\hat{A}})}_{\text{complexity}}.
		\end{split}
	\end{equation}
	Note that MC-GRA is a maximin game:
	it maximizes the approximation of encoding and decoding processes
	of the two Markov chains,
	while minimizing the complexity to avoid trivial solutions
	by constraining the graph density.
	Compared with Basic-GRA (Eq.~\eqref{eqn: Basic-GRA}),
	it approximates latent variables $\mathcal{S}_{A}$ in \texttt{ORI-chain}
	by $\mathcal{S}_{\bm{\hat{A}}}$ in \texttt{GRA-chain}, \textit{i.e.},
	\vspace{-4pt}
	\begin{equation}	
		\begin{split}
			\max_{\bm{\hat{A}}} I(\mathcal{S}_{A}; \mathcal{S}_{\bm{\hat{A}}})
			\; \Rightarrow \;
			\max_{\bm{\hat{A}}} I(A ; \bm{\hat{A}}), 
			\; \; s.t. \; \;
			\mathcal{S}_{A} \! = \! \{\bm{Z}_{A}^{i}: \bm{Z}_{A}^{i} \! \in \! \mathcal{K} \}, \;
			\mathcal{S}_{\bm{\hat{A}}} \! = \! \{\bm{Z}_{\bm{\hat{A}}}^{i}: \bm{Z}_{A}^{i} \! \in \! \mathcal{S}_{A}\}.
		\end{split}
		\label{eqn: MC-GRA-MI}
	\end{equation}

	\section{Implementation Details}
	

	In this section,
	we provide a detailed introduction 
	to the technical designs and implementation details 
	of the proposed methods in our work.
	Specifically,
	Appendix.~\ref{ssec: quantifying privacy leakage} and Appendix.~\ref{ssec: graph information plane}
	are technical details of the empirical study in Sec.~\ref{sec: overview},
	Appendix.~\ref{ssec: differentiable-MI} and Appendix.~\ref{ssec: full algorithm}
	are elaboration for the methodology in Sec.~\ref{sec: GRA attack} and Sec.~\ref{sec: GRA defense}.
	
	\subsection{Quantifying privacy leakage}
	\label{ssec: quantifying privacy leakage}
	
	As introduced in Sec.~\ref{sec: overview},
	studying the direct correlation between ground-truth adjacency $A$
	and single variable $\bm{Z} \! \in \! \{ X, Y, \bm{H}_{A}, \bm{\hat{Y}}_{A} \}$
	can provide insight into the studied GRA problem.
	%
	The informative concept of mutual information (MI)
	that $I(X; Y) = D_{KL} [ p(x,y) || p(x)p(y)] = H(X) - H(X|Y)$
	is a measure of the symmetric correlation between the two variables,
	which suits our needs perfectly.
	To avoid the cumbersome calculation of MI,
	a surrogate estimation \textit{w.r.t.} the existence of edges in $A$, \textit{i.e.},
	the AUC (area under the curve) metric is utilized~\citep{zhang2018link, zhu2021neural} to quantify $I(A; \bm{Z})$
	regarding edges in $A$ and $\hat{A}_{\bm{Z}} \! = \! \sigma(\bm{Z}\bm{Z}^\top)$
	can be an efficient solution here.
	
	Specifically,
	suppose $A \! \in \! \{0,1\}^{N \times N}$ 
	and $\bm{Z} \! \in \! \mathbb{R}^{N \times D}$,
	where $N$ is the number of nodes and $D$ is the hidden size.
	We define $I(A;\bm{Z}) \triangleq I(A; \hat{A}_{\bm{Z}})$,
	where $\hat{A}_{\bm{Z}} \! = \! \sigma(\bm{Z} \bm{Z}^\top) \! \in \! \mathbb{R}^{N \times N}$
	indicates the predictive existence of each edge.
	Note that the dot product $\bm{Z} \bm{Z}^\top$
	followed by the activate function $\sigma(\cdot)$
	is in direct proportion to the cosine similarity,
	which is commonly utilized in investigating the distribution of representations.
	Here, a higher MI $I(A;\bm{Z})$
	indicates a lower expectation of distance $\mathbb{E}_{(i,j) \sim A} d(z_i, z_j)$,
	where measurement $d(\cdot, \cdot)$ can be cosine, euclidean, etc.
	Alternatively,
	the activation function $\sigma(\cdot)$ can be \texttt{ReLU}, \texttt{Sigmoid}, etc.
	For quantifying $I(A; \hat{A}_{\bm{Z}})$,
	the AUC (area under the curve) is utilized as the metric~\citep{zhang2018link, zhu2021neural}
	regarding edges 
	in factual adjacency $A$ and recovered adjacency $\hat{A}_{\bm{Z}}$.

	\subsection{Graph information plane}
	\label{ssec: graph information plane}
	
	Recall that in Sec.~\ref{ssec: tracking by graph information plane},
	we track the aforementioned MI terms in the training process for further study. 
	Here, the training dynamics of representations H in each layer are projected to 
	the two dimensional $\big( I(A; \bm{H}), I(Y ; \bm{H}) \big)$ plane. 
	The $I(X; \bm{H})$ is not considered as node features do not always exist, 
	while the characteristic adjacency $A$ and labels $Y$ are indispensable for supervised graph learning. 
	Thus, we derive the Graph Information Plane
        as defined in the following Def.~\ref{def:Graph Information Plane}.
	
	\begin{definition}[Graph Information Plane]	
		For any node representations $\bm{H}$ of a graph,
		it can be seen as encoded from the adjacency $A$
		and decoded into the prediction objective $Y$.
		The sample complexity of the graph learning model
		is determined by the encoder MI $I(A; \bm{H}_{A})$,
		while the generalization error is indicated by the decoder MI $I(Y; \bm{H})$.
		And technically, the AUC is utilized for computing $I(A; \bm{H})$ \textit{w.r.t.} edges in $A$,
		while the Accuracy is used to measure $I(Y; \bm{H})$.
		Here, the node representations $\bm{H}$ 
		can be uniquely mapped to the plane with coordinates $\big(I(A; \bm{H}),  I(Y; \bm{H})\big)$.
		\label{def:Graph Information Plane}
	\end{definition}

	%
	%
	%
	%

	\subsection{Differentiable similarity estimations}
	\label{ssec: differentiable-MI}
	
	Solving Eq.~\eqref{eqn: MC-GRA} and Eq.~\eqref{eqn: tighter-defense-MI}
	requires to derive tractable objectives.
	%
	Given two variables,
	$X \! \in \! \mathbb{R}^{N \times D_x}$ and $Y \! \in \! \mathbb{R}^{N \times D_y}$,
	we estimate their similarity $s(X, Y)$
	with six following differentiable measurements~\citep{kornblith2019similarity}.
	\vspace{-8pt}
	\begin{itemize}[leftmargin=*]
		\setlength\itemsep{0.1em}
		\item
		\textit{Dot Product-Based Similarity (DP).}
		That is, dot products measure the similarity between samples' features as
		$\texttt{DP}(X, Y) \! = \! \text{tr}(X X^\top Y Y^\top) \! = \!  ||Y^\top X||^{2}_{F}$.
		Note the $ij$-th element of $X X^\top$ 
		are dot products of feature $x_i$ and $x_j$.
		
		\item
		\textit{Hilbert-Schmidt Independence Criterion (HSIC).}
		HSIC first takes a nonlinear feature transformation of each variable, 
		and then measures the norm of cross-covariance between these features.
		The empirical estimation~\citep{gretton2005measuring} is
		$\texttt{HSIC}(K, L) \! = \! \frac{1}{(n-1)^2} \text{tr}(KHLH)$.
		Specifically, $K_{ij} \! = \! k(x_i, x_j)$ and $L_{ij} \! = \! l(y_i, y_j)$,
		$k$ and $l$ are kernels,
		and $H$ is the centering matrix of size $N$.
		
		\item
		\textit{Centered Kernel Alignment (CKA).}
		CKA~\citep{cortes2012algorithms} is devised based on HSIC.
		It cooperates with HSIC with normalization
		to become invariant to isotropic scaling,
		\textit{i.e.}, $\forall \alpha, \beta \! \in \! \mathbb{R}^{+}, s(X, Y) \! = \! s(\alpha X, \beta Y)$.
		In calculation,
		$\texttt{CKA}(K, L) = \nicefrac{\texttt{HSIC}(K, L)}{\sqrt{\texttt{HSIC}(K, K) \texttt{HSIC}(L, L)}}$.
		
		\item
		\textit{Kernel Density Estimation (KDE).}
		KDE estimates the margin and joint PDF (Probability Density Function) of the data as
		$K_\mathbf{H}(\mathbf{x})={(2\pi)^{-d/2}}\mathbf{|H|}^{-1/2}e^{-\frac{1}{2}\mathbf{x^T}\mathbf{H^{-1}}\mathbf{x}}$.
		%
		KDE generates the $p_X$ and $p_Y$ along with kernels $k_X$ and $k_Y$, 
		which are then used to compute the joint PDF $p_{XY}$ by simply taking the dot product,
		based on which the MI $I(X;Y)$ is then computed.
		
		\item
		\textit{The Kullback-Leibler Divergence (KL)}.
		Given two probability distributions $X$ and $Y$,
		the KL divergence is computed as 
		$\texttt{KL}(X, Y) = \sum_{z\in \mathcal{X}} X(z)\log(\nicefrac{X(z)}{Y(z)})$.
		
		\item
		\textit{Mean Squared Error (MSE)}. MSE calculates the point-wise distance of two distributions 
		to indicate their similarity
		as $\texttt{MSE}(X, Y) = \frac{1}{n} \sum_{i=1}^{n} (X_i-Y_i)^2$.
	\end{itemize}

	\subsection{Full algorithm}
	\label{ssec: full algorithm}
	
	Recall that the two proposed methods,
	MC-GRA and MC-GPB,
	that are illustrated in Fig.~\ref{fig: markov-attack} and Fig.~\ref{fig: markov-defense}.
	Here,
	regarding these two methods, we respectively elaborate the full algorithms 
	in Alg.~\ref{alg: GRA} and Alg.~\ref{alg: GPB} as follows.
	
	\begin{algorithm}[H]
		\caption{Markov Chain-based Graph Reconstruction Attack.}
		\begin{algorithmic}[1]
			\REQUIRE Target model $f_{\bm{\theta}^*}$, prior knowledge set $\mathcal{K}$, similarity measurement $s(\cdot)$
			\STATE initialize the parameterized distribution $\mathbb{P}_{\bm{\phi}}(\bm{\hat{A}})$ with parameters $\bm{\phi}$
			\STATE collect $\mathcal{S}_{A} \! = \! \{\bm{Z}_{A}^{i}: \bm{Z}_{A}^{i} \! \in \! \mathcal{K} \}$
			\FOR{$i=1\dots n$}
			\STATE sample an  adjacency from the distribution $\bm{\hat{A}} \sim \mathbb{P}_{\bm{\phi}}(\bm{\hat{A}})$
			\STATE inject stochasticity as $\tilde{X} \! = \! X \oplus X_{\epsilon}, \bm{\tilde{A}} \! = \! \bm{\tilde{A}} \oplus A_{\epsilon}$
			\STATE obtain $\mathcal{S}_{\bm{\hat{A}}} \! = \! \{\bm{Z}_{\bm{\hat{A}}}^{i}: \bm{Z}_{A}^{i} \! \in \! \mathcal{S}_{A}\}$ by the forward of \texttt{GRA-chain} as $f_{\bm{\theta}^*}(\bm{\tilde{A}}, \tilde{X})$ 
			\STATE update $\bm{\phi}$ by maximizing the MC-GRA objective in Eq.~\eqref{eqn: MC-GRA} with $\mathcal{S}_{\bm{\hat{A}}}$ and $\mathcal{S}_{A}$
			\ENDFOR
			\STATE \textbf{return } The optimal recovered adjacency $\bm{\hat{A}}^*$ that $\bm{\hat{A}}^* \sim \mathbb{P}_{\bm{\phi}^*}(\bm{\hat{A}})$
		\end{algorithmic}
		\label{alg: GRA}
	\end{algorithm}

	\begin{algorithm}[H]
		\caption{Markov Chain-based Defensive Training Against Graph Reconstruction Attack.}
		\begin{algorithmic}[1]
			\REQUIRE Graph data $A, X, Y$ and similarity measurement $s(\cdot)$
			\STATE initialize parameters $\bm{\theta}$ of the GNN $f_{\bm{\theta}}$
			\FOR{$i=1\dots n$}
			\STATE inject stochasticity as $\tilde{A} \! = \! A \oplus A_{\epsilon}$ by randomly dropping edges
			\STATE obtain the hidden representations in each layer and outputs by forwarding $f_{\bm{\theta}}(\tilde{A}, X)$
			\STATE update $\bm{\theta}$ by minimizing the MC-GPB objective in Eq.~\eqref{eqn: tighter-defense-MI}
			\ENDFOR
			\STATE \textbf{return } The trained model $f_{\bm{\theta}^*}$
		\end{algorithmic}
		\label{alg: GPB}
	\end{algorithm}

	\subsection{Reproduction}
	
	The source code is publicly available at:
	\url{https://github.com/tmlr-group/MC-GRA}.
	In addition, we summarize the search space of hyperparameters and optimal cases as follows.
	The optimal hyperparameters are obtained by random search or grid search~\cite{lavalle2004relationship}.
	
	
	\begin{table}[H]
		\centering
		\setlength\tabcolsep{10pt}
		\scriptsize
		\caption{The search space of hyper-parameters in MC-GPB.}
		\vspace{-10pt}
		\begin{tabular}{c|c|c|c}
			\toprule
			component & name & type & range \\
			\midrule
			\multirow{3}{*}{Privacy constraint} & $\beta^1_p$ (GNN layer-1) & float & $(0,10)$ \\
			& $\beta^2_p$ (GNN layer-2) & float & $(0,10)$ \\
			& $\beta^3_p$ (linear layer) & float & $(0,10)$ \\
			\midrule
			\multirow{2}{*}{Complexity constraint} 
			& $\beta^1_c$ (GNN layer-1) & float & $(0,10)$ \\
			& $\beta^2_c$ (GNN layer-2) & float & $(0,10)$ \\
			\midrule
			Similarity measurement & metric $s(\cdot, \cdot)$ & category & DP, HSIC, CKA, KDE, KL, MSE \\
			\bottomrule
		\end{tabular}
	\end{table}

	\begin{table}[H]
		\centering
		\scriptsize
		\setlength\tabcolsep{10pt}
		\caption{The optimal hyper-parameters for MC-GPB.}
		\vspace{-10pt}
		\begin{tabular}{c|c|cccccc}
			\toprule
			component & name & {    Cora    }  & Citeseer &  Polblogs   & USA   &   Brazil &  AIDS   \\
			\midrule
			\multirow{3}{*}{Privacy constraint} & $\beta^1_p$ (GNN layer-1) & 1.3 		& 0.09		& 3.00	  & 6.60	 	&	1.90	& 2.40		\\
			& $\beta^2_p$ (GNN layer-2) & 1.3 		& 0.006		& 2.00	  & 1.00	 	& 2.50		& 3.90	\\
			& $\beta^3_p$ (linear layer) & 1.7 		& 0.01	&	2.00  & 0.50 	& 1.00	& 1.30		\\
			\midrule
			\multirow{2}{*}{Complexity constraint} 
			& $\beta^1_c$ (GNN layer-1) &  1.4		& 5e-10		& 1.00	  & 1.30	 	& 1.20		& 1.30		\\
			& $\beta^2_c$ (GNN layer-2) & 1.5 		& 1e-10		& 1.00	  & 3.80	 	& 1.20	& 1.30	\\
			\midrule
			Similarity measurement & metric $s(\cdot, \cdot)$ 		&	KDE	& KDE	  &	KDE 	&	DP	& KDE	&	KDE	\\
			\bottomrule
		\end{tabular}
	\end{table}

	\begin{table}[H]
		\centering
		\setlength\tabcolsep{10pt}
		\scriptsize
		\caption{The search space of hyper-parameters in MC-GRA.}
		\vspace{-10pt}
		\begin{tabular}{c|c|c|c}
			\toprule
			component & name & type & range \\
			\midrule
			Complexity constraint
			& $\alpha_{c}$ (information entropy of $\hat{\bm{A}}$) & float & $(10^{-4},10^{4})$ \\
			\midrule
			Propagation approximation 
			& $\alpha_{p}$ (between $\bm{H}_{\hat{\bm{A}}}$ and $\bm{H}_A$) & float & $(10^{-4},10^{4})$ \\
			\midrule
			Output approximation &$\alpha_{o}$ (between $\hat{\bm{Y}}_{\hat{\bm{A}}}$ and $\hat{\bm{Y}}_{A}$) & float & $(10^{-4},10^{4})$ \\
			\midrule
			Label approximation 
			& $\alpha_{s}$ (between $\hat{\bm{Y}}_{\hat{\bm{A}}}$ and $Y$) & float & $(10^{-4},10^{4})$ \\
			\midrule
			Similarity measurement & metric $s(\cdot, \cdot)$ & category & DP, HSIC, CKA, KDE, KL, MSE \\
			\bottomrule
		\end{tabular}
	\end{table}

	\begin{table}[!htp]
		\centering
		\caption{The optimal hyper-parameters for MC-GRA.}
		\vspace{-10pt}
		\scriptsize
		\setlength\tabcolsep{10pt}
		\begin{tabular}{cc|ccccc}\toprule
			dataset &prior $\mathcal{K}$ &$\alpha_c$ &$\alpha_p$ &$\alpha_o$ &$\alpha_s$ &$s(\cdot, \cdot)$ \\
			\midrule
			\multirow{7}{*}{cora}
                &$\{X, \bm{H}_{A}\}$ &10000 &1000 &0 &0 &MSE \\
			&$\{X, \bm{Y}_{A}\}$ &100 &0 &0.1 &0 &KL \\
			&$\{X, Y\}$ &0.01 &0 &0 &1 &CKA \\
			&$\{X, \bm{H}_{A}, \bm{Y}_{A}\}$ &0.1 &100 &100 &0 &MSE \\
			&$\{X, \bm{H}_{A}, Y\}$ &0.0001 &10 &0 &1 &MSE \\
			&$\{X, \bm{Y}_{A}, Y\}$ &10 &0 &0.01 &1 &MSE \\
			&$\{X, \bm{H}_{A}, \bm{Y}_{A}, Y\}$ &10 &10 &1000 &1 &MSE \\
			\midrule
			\multirow{7}{*}{citeseer}
                &$\{X, \bm{H}_{A}\}$ &0.01 &10 &0 &0 &KL \\
			&$\{X, \bm{Y}_{A}\}$ &0 &0 &1 &0 &KL \\
			&$\{X, Y\}$ &100 &0 &0 &1 &MSE \\
			&$\{X, \bm{H}_{A}, \bm{Y}_{A}\}$ &10 &100 &0.001 &0 &MSE \\
			&$\{X, \bm{H}_{A}, Y\}$ &0.0001 &100 &0 &1 &KL \\
			&$\{X, \bm{Y}_{A}, Y\}$ &1 &0 &10000 &1 &KL \\
			&$\{X, \bm{H}_{A}, \bm{Y}_{A}, Y\}$ &0.001 &1000 &0.001 &1 &KL \\
			\midrule
			\multirow{7}{*}{polblogs}
                &$\{X, \bm{H}_{A}\}$ &0 &1000 &0 &0 &KL \\
			&$\{X, \bm{Y}_{A}\}$ &100 &0 &1 &0 &DP \\
			&$\{X, Y\}$ &1000 &0 &0 &1 &MSE \\
			&$\{X, \bm{H}_{A}, \bm{Y}_{A}\}$ &0.1 &1000 &0 &0 &MSE \\
			&$\{X, \bm{H}_{A}, Y\}$ &100 &0.001 &0 &1 &HSIC \\
			&$\{X, \bm{Y}_{A}, Y\}$ &100 &0 &0 &1 &CKA \\
			&$\{X, \bm{H}_{A}, \bm{Y}_{A}, Y\}$ &10000 &0.001 &1000 &1 &HSIC \\
			\midrule
			\multirow{7}{*}{usair}
                &$\{X, \bm{H}_{A}\}$ &100 &100 &0 &0 &MSE \\
			&$\{X, \bm{Y}_{A}\}$ &0.01 &0 &0.001 &0 &MSE \\
			&$\{X, Y\}$ &0.0001 &0 &0 &1 &DP \\
			&$\{X, \bm{H}_{A}, \bm{Y}_{A}\}$ &0.0001 &0.01 &0.0001 &0 &HSIC \\
			&$\{X, \bm{H}_{A}, Y\}$ &1000 &0.001 &0 &1 &HSIC \\
			&$\{X, \bm{Y}_{A}, Y\}$ &1000 &0 &0.01 &1 &CKA \\
			&$\{X, \bm{H}_{A}, \bm{Y}_{A}, Y\}$ &10 &0.01 &0.01 &1 &DP \\
			\midrule
			\multirow{7}{*}{brazil}&$\{X, \bm{H}_{A}\}$ &100 &0.01 &0 &0 &KL \\
			&$\{X, \bm{Y}_{A}\}$ &1 &0 &100 &0 &MSE \\
			&$\{X, Y\}$ &0.001 &0 &0 &1 &KDE \\
			&$\{X, \bm{H}_{A}, \bm{Y}_{A}\}$ &0.0001 &100 &1000 &0 &MSE \\
			&$\{X, \bm{H}_{A}, Y\}$ &0.0001 &100 &0 &1 &HSIC \\
			&$\{X, \bm{Y}_{A}, Y\}$ &100 &0 &0.1 &1 &DP \\
			&$\{X, \bm{H}_{A}, \bm{Y}_{A}, Y\}$ &0.1 &0.1 &100 &1 &KL \\
			\midrule
			\multirow{7}{*}{AIDS}
                &$\{X, \bm{H}_{A}\}$ &10000  &1000 &0 &0 &KL \\
			&$\{X, \bm{Y}_{A}\}$ &1 &0 &0.01 &0 &MSE \\
			&$\{X, Y\}$ &0.0001 &0 &0 &1 &CKA \\
			&$\{X, \bm{H}_{A}, \bm{Y}_{A}\}$ &0.001 &1 &100 &0 &MSE \\
			&$\{X, \bm{H}_{A}, Y\}$ &0.0001 &0.1 &0 &1 &MSE \\
			&$\{X, \bm{Y}_{A}, Y\}$ &0 &0 &0 &1 &MSE \\
			&$\{X, \bm{H}_{A}, \bm{Y}_{A}, Y\}$ &0 &1000 &0.0001 &1 &KL\\
			\bottomrule
		\end{tabular}
	\end{table}

\end{document}